\documentclass[twoside,11pt,dvipsnames]{article}
\usepackage{subcaption}
\usepackage[abbrvbib,preprint]{jmlr2e}

\usepackage{microtype}
\usepackage{ragged2e}
\usepackage{tabularx}

\usepackage{tikz}
\usetikzlibrary{automata}
\usetikzlibrary{arrows}
\usetikzlibrary{positioning}
\usetikzlibrary{calc}
\usetikzlibrary{shapes.geometric}
\usetikzlibrary{arrows.meta}

\usepackage{amsmath}
\usepackage{amssymb}
\usepackage[capitalize,noabbrev]{cleveref}
\usepackage{thmtools}
\usepackage{thm-restate}
\usepackage{tabularx}
\usepackage{booktabs}
\usepackage{enumitem}
\usepackage{xcolor}

\makeatletter
\def\cleartheorem#1{\expandafter\let\csname#1\endcsname\relax
    \expandafter\let\csname c@#1\endcsname\relax
}
\makeatother
\cleartheorem{theorem}
\cleartheorem{lemma}
\cleartheorem{proposition}
\cleartheorem{definition}
\cleartheorem{corollary}
\cleartheorem{axiom}
\cleartheorem{remark}
\cleartheorem{conjecture}

\newtheorem{theorem}{Theorem}
\newtheorem{proposition}[theorem]{Proposition}
\newtheorem{lemma}[theorem]{Lemma}
\newtheorem{assumption}[theorem]{Assumption}
\newtheorem{definition}[theorem]{Definition}

\newtheorem{axiom}{Axiom}
\newtheorem{remark}[theorem]{Remark}

\crefname{axiom}{Axiom}{Axioms}
\creflabelformat{equation}{#2#1#3}

\DeclareMathOperator*{\argmin}{arg\,min}
\DeclareMathOperator*{\argmax}{arg\,max}

\newcommand{\cvar}{\mathrm{CVaR}}
\newcommand{\ocvar}{\mathrm{OCVaR}}
\newcommand{\states}{\mathcal{S}}
\newcommand{\actions}{\mathcal{A}}
\newcommand{\histories}{\mathcal{H}}
\newcommand{\reals}{\mathbb{R}}
\newcommand{\naturals}{\mathbb{N}}
\newcommand{\pibar}{\overline{\pi}}
\newcommand{\rhobar}{\overline{\rho}}
\newcommand{\pitilde}{\tilde{\pi}}
\newcommand{\rvs}{\Delta(\reals)}
\newcommand{\crvs}{\mathcal{D}}
\newcommand{\wass}{\mathrm{w}}
\newcommand{\E}{\mathbb{E}}
\newcommand{\Prob}{\mathbb{P}}
\newcommand{\Ind}{\mathbb{I}}
\newcommand{\qf}{\mathrm{QF}}
\newcommand{\Uf}{U_f}
\newcommand{\distequiv}{\overset{\mathcal{D}}{=}}
\newcommand{\df}{\mathrm{df}}
\newcommand{\utilde}{\widetilde{u}}
\newcommand{\rtilde}{\widetilde{r}}
\newcommand{\Rtilde}{\widetilde{R}}
\newcommand{\Vtilde}{\widetilde{V}}
\newcommand{\vtilde}{\widetilde{v}}

\newcommand{\stationary}{\Pi}
\newcommand{\markov}{\Pi_{\mathrm{M}}}
\newcommand{\historybased}{\Pi_{\mathrm{H}}}
\newcommand{\xC}{C}
\newcommand{\xc}{c}
\newcommand{\cmin}{\xc_{\mathrm{min}}}
\newcommand{\cmax}{\xc_{\mathrm{max}}}
\newcommand{\cs}{\mathcal{\xC}}
\newcommand{\lookaheadG}{G_{\mathrm{lookahead}}}
\newcommand{\classicBellman}{T}
\newcommand{\sinit}{s_{\mathrm{init}}}
\newcommand{\sterm}{s_{\mathrm{term}}}

\newcommand{\stock}{stock}
\newcommand{\Stock}{Stock}
\newcommand{\STOCK}{Stock}
\newcommand{\problem}{return distribution optimization}
\newcommand{\Problem}{Return distribution optimization}
\newcommand{\PROBLEM}{Return Distribution Optimization}

\usepackage{lastpage}

\ShortHeadings{Optimizing Return Distributions with Distributional Dynamic Programming}{\'Avila Pires, Rowland, Borsa, Guo, Khetarpal, Barreto, Abel, Munos and Dabney}
\firstpageno{1}

\begin{document}

\title{Optimizing Return Distributions with \\ Distributional Dynamic Programming}

\author{\name Bernardo \'Avila Pires \email bavilapires@google.com \\
       \addr Google DeepMind, London, UK
       \AND
       \name Mark Rowland \\
       \addr Google DeepMind
       \AND
       \name Diana Borsa \\
       \addr Google DeepMind
       \AND
       \name Zhaohan Daniel Guo \\
       \addr Google DeepMind
       \AND
       \name Khimya Khetarpal \\
       \addr Google DeepMind
       \AND
       \name Andr\'e Barreto \\
       \addr Google DeepMind
       \AND
       \name David Abel \\
       \addr Google DeepMind
       \AND
       \name R\'emi Munos \\
       \addr FAIR, Meta; work done at Google DeepMind
       \AND
       \name Will Dabney \\
       \addr Google DeepMind
}

\editor{Martha White}

\maketitle

\begin{abstract}
We introduce distributional dynamic programming (DP) methods for optimizing statistical functionals of the return distribution, with standard reinforcement learning as a special case.
Previous distributional DP methods could optimize the same class of expected utilities as classic DP.
To go beyond, we combine distributional DP with \emph{\stock{} augmentation}, a technique previously introduced for classic DP in the context of risk-sensitive RL, where the MDP state is augmented with a statistic of the rewards obtained since the first time step.
We find that a number of recently studied problems can be formulated as \stock{}-augmented \problem{}, and we show that we can use distributional DP to solve them.
We analyze distributional value and policy iteration, with bounds and a study of what objectives these distributional DP methods can or cannot optimize.
We describe a number of applications
outlining how to use distributional DP to solve different \stock{}-augmented \problem{} problems, for example maximizing conditional value-at-risk, and homeostatic regulation.
To highlight the practical potential of \stock{}-augmented \problem{} and distributional DP, we introduce an agent that combines DQN and the core ideas of distributional DP, and empirically evaluate it for solving instances of the applications discussed.
\end{abstract}

\begin{keywords}
reinforcement learning, 
distributional reinforcement learning, 
risk-sensitive reinforcement learning,
dynamic programming,
\stock{}-augmented Markov decision process
\end{keywords}

\section{Introduction}
\label{sec:introduction}

Reinforcement learning \citep[RL;][]{sutton2018reinforcement,szepesvari2022algorithms} is a powerful framework for building intelligent agents, and it has been successfully applied to solve many practical problems~\citep{mnih2015human,silver2018general,bellemare2020autonomous,degrave2022magnetic,fawzi2022discovering}.
In the standard formulation of the RL problem, the objective is to find a policy (a decision rule for selecting actions) that maximizes the expected (discounted) return in a Markov decision process~\citep[MDP;][]{puterman2014markov}.
A similar, related problem is what we refer to as \emph{\problem{}}, where the objective is to optimize a functional of the return distribution~\citep{marthe2024beyond}, which may not be the expectation. 
For example, we could maximize an \emph{expected utility}~\citep{von2007theory,bauerle2014more,marthe2024beyond}, that is, the expectation of the return ``distorted'' by some function.

By varying the choice of statistical functional being optimized (be it an expected utility or more general),
we can model various RL-like problems as \problem{}, including problems in the field of risk-sensitive RL~\citep{chung1987discounted,chow2014algorithms,noorani2022risk}, homeostatic regulation~\citep{keramati2011reinforcement} and satisficing~\citep{simon1956rational,goodrich2004satisficing}.

The fact that \problem{} captures many problems of interest makes it appealing to develop solution methods for the general problem.
At first glance, the apparent benefits of solving the general problem are offset by the fact that, for many instances, optimal stationary Markov policies do not exist~\citep[see, for example,][]{marthe2024beyond}.
This can be problematic, because it rules out dynamic programming \citep[DP; value iteration and policy iteration;][]{bertsekas1996neuro,sutton2018reinforcement,szepesvari2022algorithms} and various other RL methods that are designed to output stationary Markov policies.
Defaulting to solution methods that produce history-based policies is an alternative we would like to avoid, under the premise that learning history-based policies can be intractable~\citep{papadimitriou1987complexity,madani1999undecidability}. 

We show that we can reclaim optimality of stationary Markov policies for many instances of \problem{} by augmenting the state of the MDP with a simple statistic we call \emph{\stock{}}.
\Stock{} is a backward looking quantity related to the agent's accumulated past rewards, including an initial \stock{} (the precise definition is given in \cref{sec:preliminaries}).
It was introduced by \citet{bauerle2011markov}\footnote{\citet[Example b, p.~269;][]{kreps1977decision} outlined a similar statistic in the undiscounted setting.} for maximizing conditional value-at-risk~\citep{rockafellar2000optimization}.
The MDP state and \stock{} together provide enough information for stationary Markov policies (with respect to the state-\stock{} pair) to optimize various statistical functionals of the distribution of returns offset by the agent's initial \stock{}.

Incorporating \stock{} into \problem{} gives rise to the specific formulation we consider in this paper, where the environment is assumed to be an MDP with states augmented by \stock{}, and the return is offset by an initial \stock{}.
We refer to this formulation as \emph{\stock{}-augmented \problem{}}.

The optimality guarantee for stationary \stock{}-augmented Markov policies in \problem{} suggests that we may be able to develop DP solution methods for the instances where the guarantee applies.
Classic value/policy iteration cannot cope with return distributions, but this limitation can be overcome using distributional RL~\citep{chung1987discounted,morimura2010nonparametric,bellemare2017distributional,bellemare2023distributional}.
That is, we may resort to \emph{distributional dynamic programming} to tackle \problem{}.

In the standard MDP setting, without \stock{}, 
distributional DP methods already exist for policy evaluation~\citep[Chapter 5;][]{bellemare2023distributional}, for maximizing expected return (as an obvious adaptation), and for expected utilities~\citep{marthe2024beyond}.
However, these methods can only solve problems that classic DP can also solve~\citep{marthe2024beyond}, namely, the \problem{} problems for which an optimal stationary Markov policy exists (with respect to the MDP states alone).
Notably, by incorporating \stock{} into distributional DP, we can optimize statistical functionals of the return distribution that we could not otherwise.
Moreover, \stock{}-augmented distributional DP is a single solution method for a variety of \problem{} problems (which so far have been studied and solved in isolation), and also a blueprint for practical methods to solve \problem{}, in much the same way the principles of classic DP and distributional policy evaluation factor into previously proposed, successful RL methods.

\subsection{Paper Summary and Contributions}
This paper is an in-depth study of distributional dynamic programming for solving \stock{}-augmented \problem{}, and we make the following contributions:
\begin{enumerate}
    \item We identify conditions on the statistical functional being optimized under which distributional DP can solve \stock{}-augmented \problem{}, and develop a theory of distributional DP for solving this problem, including:
    \begin{itemize}[label=-]
        \item principled distributional DP methods (distributional value/policy iteration),
        \item performance bounds and asymptotic optimality guarantees (for the cases that distributional DP can solve),
        \item necessary and sufficient conditions for the finite-horizon case, plus mild sufficient conditions to the infinite-horizon discounted case.
    \end{itemize}
    \item We demonstrate multiple applications of distributional value/policy iteration for \stock{}-augmented \problem{}, namely:
    \begin{itemize}[label=-]
        \item Optimizing expected utilities~\citep{von2007theory,bauerle2014more}.
        \item Maximizing conditional value-at-risk, a form of risk-sensitive RL, both the risk-averse conditional value-at-risk~\citep{bauerle2011markov}, and a risk-seeking variant that we introduce.
        \item Homeostatic regulation~\citep{keramati2011reinforcement}, where the agent aims to maintain vector-valued returns near a target.
        \item Satisfying constraints, and trading off minimizing constraint violations with maximizing expected return.
    \end{itemize}
    \item We show how to reduce \stock{}-augmented \problem{} objective to a \stock{}-augmented RL objective (via reward design); and that, in \stock{}-augmented settings, classic DP cannot optimize all the \problem{} objectives that distributional DP can.
    \item We introduce D$\eta$N (pronounced \emph{din}), a deep RL agent that combines QR-DQN~\citep{dabney2018distributional} with the principles of distributional value iteration and \stock{} augmentation to optimize expected utilities. Through experiments, we demonstrate D$\eta$N's ability to learn effectively under objectives in toy gridworld problems and the game of Pong in Atari~\citep{bellemare2013arcade}.
\end{enumerate}

\subsection{Paper Outline}
\Cref{sec:preliminaries} introduces notation and basic definitions.
In \cref{sec:problem-formulation} we formalize the problem of \stock{}-augmented \problem{}, and provide some basic example instances.
\Cref{sec:ddp} introduces distributional value/policy iteration and presents our main theoretical results.
In \cref{sec:applications}, we discuss multiple applications of our results and show concrete examples of how to model different problems using \stock{} augmentation and distributional DP (\cref{sec:generating-returns,sec:cvar,sec:ocvar,sec:homeostatic-regulation,sec:constraint-satisfaction,sec:beyond-expected-utilities}).\footnote{Some of these problems have been previously studied, and distributional DP is a novel solution approach in some cases (see \cref{sec:applications}).}
In \cref{sec:applications} we also explore implications of our results in different contexts: Generalized policy evaluation~(\citealp{barreto2020fast}; \cref{sec:gpe}); reward design and the relationship between \stock{}-augmented RL and \stock{}-augmented \problem{} (\cref{sec:reward-design}).
In \cref{sec:din}, we introduce D$\eta$N and show how distributional DP can inform the design of deep RL agents.
To highlight the practical implications of our contributions, in \cref{sec:gridworld-experiments} we present an empirical study of D$\eta$N in different gridworld instances of some of the applications considered in \cref{sec:applications}.
In \cref{sec:atari} we complement our gridworld results with a demonstration of D$\eta$N controlling returns in a more complex setting: The Atari game of Pong, where we show that a single trained D$\eta$N agent can obtain various specific scores in a range, and where we use \stock{} augmentation to specify the scores we want the agent to achieve.
\Cref{sec:conclusion} concludes our work and presents directions for future work, notably practical questions revealed by our empirical study.
We provide additional theoretical results in \cref{app:additional-theoretical-results}.
\Cref{app:ddp} contains the full analysis of distributional value/policy iteration, and \cref{app:conditions} contains 
the full analysis of the conditions for our main theorems.
\Cref{app:cvar,app:ocvar,app:beyond-expected-utilities,app:reward-design} contain proofs for the results in \cref{sec:applications}.
\Cref{app:implementation-details} contains implementation details for D$\eta$N and our experiments.
\Cref{app:summary-of-guarantees} provides a summary guarantees for classic and distributional DP in the various different settings considered throughout this paper, and is a useful map for readers interested in understanding the kinds of problems that DP can solve.

\section{Preliminaries}
\label{sec:preliminaries}

We write $\naturals \doteq \{1, 2, \ldots\}$ for the natural numbers excluding zero, and $\naturals_0 \doteq \{0, 1, 2, \ldots\}$.
For $n \in \naturals_0$, $\Delta(n)$ denotes the $|n|$-dimensional simplex.
For $m \in \naturals$, $\Delta(\reals^m)$ denotes the set of probability distribution functions of $\reals^m$-valued random variables.

We study the problems where an agent interacts with a Markov decision process \citep[MDP;][]{puterman2014markov} with (possibly infinite) state space $\states$ and finite action space $\actions$.
Rewards can be stochastic and the discount is $\gamma \in (0, 1]$.
We adopt the convention that $R_{t+1}$ is the reward random variable observed jointly with $S_{t+1}$, that is, $R_{t+1}, S_{t+1}$ result from taking action $A_t$ at state $S_t$, according to the MDP's reward and transition kernels.

The reward signal may be a vector-valued pseudo-reward (cumulant) signal~\citep{sutton2011horde} in $\cs \doteq \reals^m$.
The vector-valued case allows us to capture interesting applications that are worth the extra generality.
However, to avoid unnecessary complication, our presentation is intentionally in terms of $\cs$, so that the reader can easily appreciate the results in the scalar case ($\cs = \reals$) if they wish.
We use the terms reward and returns to avoid an excess of \emph{pseudo} prefixes in the text.

Some of our results only apply to finite-horizon MDPs.
We say an MDP \emph{has finite horizon} if there exists a constant $n \in \naturals$ such that $S_n$ is terminal with probability one for any trajectory $S_0, A_0, S_1, A_1, \ldots, S_n$ generated in the MDP.
We call the smallest such $n$ the \emph{horizon} of the MDP.
A state $s$ is terminal if $(S_{t+1}, R_{t+1}) = (s, 0)$ with probability one whenever $S_t = s$ (regardless of $A_t$).
We refer to the case where the MDP has finite horizon as the \emph{finite-horizon case} (complementary to the \emph{infinite-horizon} case), and to the case where $\gamma < 1$ as the \emph{discounted case} (complementary to the \emph{undiscounted} case, where $\gamma = 1$).

We make the following assumption throughout the work, similar to Assumption 2.5 by \citeauthor{bellemare2023distributional} (p.~19; \citeyear{bellemare2023distributional}).

\begin{assumption}[All rewards have uniformly bounded first moment]
\label{ass:bounded-first-moment-rewards}
\[
    \sup_{s, a \in \states \times \actions}\E \left( \left. \| R_{t+1} \|_1 \, \right| S_t = s, A_t = a \right) < \infty
\]
\end{assumption}

Similar to \citet{bauerle2011markov}, we consider an augmented MDP state space $\states \times \cs$. If $s, a, r', s'$ is a transition in the original MDP, then for any $\xc \in \cs$ the augmented MDP transitions as $(s, \xc), a, r', (s', \gamma^{-1}(\xc + r'))$, that is:
\begin{equation}
    \label{eq:b-def}
    \xc_{t+1} = \frac{\xc_t + r_{t+1}}{\gamma}.
\end{equation}

We refer to $\xc_t$ as the agent's \emph{\stock{}}.\footnote{In our formulation the \stock{} and the rewards are $m$-dimensional, whereas \citet{bauerle2011markov} consider $1$-dimensional \stock{}.}
If we unroll the recursion in \Cref{eq:b-def} up to an \emph{initial \stock{} $\xc_0$} (see \cref{rem:initial-balance} below), we can interpret the \stock{}, in a forward view, as a scaled sum of the initial \stock{} $\xc_0$ and the discounted return from time step zero up to time step $t$:
\[
    \xc_t = \underbrace{\gamma^{-t}}_{\substack{\mbox{time-dependent}\\\mbox{scaling}}}\big(\underbrace{\xc_0}_{\substack{\mbox{initial}\\\mbox{\stock{}}}} + \underbrace{\sum_{i=0}^{t-1} \gamma^i r_{i+1}}_{\substack{\mbox{partial discounted}\\\mbox{return}}}\big).
\]
In a backward view, the \stock{} can be seen as a backward reverse-discounted return:
\[
    \xc_t = \gamma^{-1}r_t + \gamma^{-2}r_{t-1} + \cdots + \gamma^{-t}r_1 + \gamma^{-t}\xc_0.
\]
Importantly, the \stock{} allows us to keep track of the discounted return (plus the initial \stock{} $\xc_0$) from time step $0$, since, for all $t \geq 0$,
\[
    \xc_0 + \sum_{i=0}^\infty \gamma^i r_{i+1} = \gamma^t \left(\xc_t + \sum_{i=0}^\infty \gamma^i r_{t+i+1} \right),
\]
When rewards (and \stock{}s) are random, the above holds with probability one, written as
\begin{equation}
    \xC_t + G_t = \gamma^{-t}\left(\xC_0 + G_0\right),
    \label{eq:balance-non-stationary-invariant}
\end{equation}
with $G_t \doteq \sum_{i=0}^\infty \gamma^i R_{t+i+1}$ denoting the respective discounted return from time step $t$.
\Cref{eq:balance-non-stationary-invariant} will be key to optimizing return distributions:
The distribution of $\xC_t + G_t$ will work as an ``anytime proxy'' for the distribution of $\xC_0 + G_0$, and by controlling the former distribution we can also control the latter---provided the objective is such that the $\gamma^{-t}$ factor does not interfere with the optimization (we will later introduce this as an indifference of the objective to the discount $\gamma$).

\begin{remark}[The Initial \STOCK{} $\xc_0$]
\label{rem:initial-balance}
The expansion of \stock{} includes an initial \stock{} $\xc_0$ that is unspecified.
Together with the initial MDP state $s_0$, this \stock{} will form the initial augmented state $(s_0, \xc_0)$.
While the initial $s_0$ is often ``given'', $\xc_0$ can be set (even as a function of $s_0$). 
This will provide extra flexibility to policies, which may display diverse behaviors in response to changes in $\xc_0$, and it will allow us to reduce different problems to \problem{} by plugging in specific choices of $\xc_0$ (as a function of $s_0$).
For example, as shown by \citet{bauerle2011markov}, we can choose $\xc_0$ in such a way that optimizing conditional value-at-risk reduces to an instance of \problem{} with \stock{} augmentation (see \cref{thm:cvar-c-star} in \cref{sec:cvar}).
\end{remark}

\begin{remark}[Dynamics Influenced by \STOCK{}]
Our results do not rely on the transitions and rewards of the augmented MDP depending only on $s$.
In a transition $(s, \xc), a, r', (s', \xc')$, $\xc'$ must be updated according to \Cref{eq:b-def}, but $s', r'$ may depend on $\xc$.
This can be useful, for example, to define termination conditions: The state $s'$ may be terminal when $\xc' = 0$ or when $|\xc'|$ is too large.
\end{remark}

Stationary Markov policies with respect to \stock{} are $\states \times \cs \rightarrow \Delta(\actions)$ functions, and the space of these policies is $\stationary \doteq \Delta(\actions)^{\states \times \cs}$.
A Markov policy $\pi$ is a sequence $\pi = \pi_0, \pi_1, \pi_2, \ldots$ of stationary policies $\pi_n : \states \times \cs \rightarrow \Delta(\actions)$, and the space of these policies is $\markov \doteq \stationary^\naturals$.
For a policy $\pi = \pi_0, \pi_1, \pi_2, \ldots$, returns are written as
$G^\pi(s, \xc) \doteq \sum_{t=0}^\infty \gamma^t R_{t+1}$ where $R_{t+1}$ are the rewards generated by starting at state $(S_0, \xC_0) = (s, \xc)$, then selecting $A_t \sim \pi_t(S_t, \xC_t)$ for $t \geq 0$. 
The return $G^\pi(s, \xc)$ may depend on $\xc$ (even when rewards do not depend on the \stock{}), because $\pi$ may choose actions differently depending on $\xc$, so the trajectories generated depend on $\xc$ as well.
If $\pi$ is stationary, then $A_t \sim \pi(S_t, \xC_t)$ for all $t \geq 0$.

A \emph{history} is the sequence of everything observed preceding action $A_t$, that is,
\[
    H_t \doteq (S_0, \xC_0), A_0, R_1, (S_1, \xC_1), A_1, \ldots, R_t, (S_t, \xC_t),
\]
The history at $t = 0$ is $S_0, \xC_0$.
The set of possible histories of finite length is
\[
    \histories \doteq \bigcup_{n \in \naturals_0} \underbrace{\states \times \cs}_{(s_0, \xc_0)} \times (\underbrace{\actions}_{a_t} \times \underbrace{\cs}_{r_{t+1}} \times \underbrace{\states \times \cs}_{(s_{t+1}, \xc_{t+1})})^n,
\]
and a \emph{history-based policy} is a function $\histories \rightarrow \Delta(\actions)$.
That is, a history-based policy makes decisions based on everything observed so far.
For $\pi$ history-based and $t \geq 0$ we have $A_t \sim \pi(H_t)$, and the set of all history-based policies is $\historybased \doteq \Delta(\actions)^\histories$. 

We let $\rvs$ be the set of distributions of $\reals$-valued random variables.
With $X \sim \nu$, we write $\df X = \nu$.
For two $\cs$-valued random variables $X, X'$ we say $X \distequiv X'$ if $\df X = \df X'$.
For $\nu \in \rvs$, we let $\qf_\nu$ be the quantile function of $\nu$:
\[
    \qf_\nu(\tau) \doteq \inf\{t \in \reals : \Prob(X \leq t) \geq \tau\}.  
    \tag{$X \sim \nu$}
\]
For $\xc \in \cs$, we denote by $\delta_{\xc}$ the Dirac measure on $\xc$, that is, the distribution such that if $\Prob(G = \xc) = 1$ when $G \sim \delta_{\xc}$.
The Dirac on zero is $\delta_{0}$ (where in the vector-valued case it is understood that $0$ refers to the all-zeros vector).
 
We define $\crvs \doteq \Delta(\cs)$ as the set of distributions of $\cs$-valued random variables.
The \emph{$1$-Wasserstein distance} for $\nu, \nu' \in \crvs$ is defined as \citep[Definition 6.1, p.~105;][]{villani2009optimal}
\[
    \wass(\nu, \nu') \doteq \inf \left\{ \E\| X - X' \|_1 : \df(X) = \nu, \df(X') = \nu'  \right\},
\]
where $X$ and $X'$ may be jointly distributed.
In the scalar case ($\cs = \reals$), we have
\[
    \wass(\nu, \nu') = \| \qf_\nu - \qf_{\nu'} \|_{\ell_1} = \E_{\tau \sim u_{(0, 1)}} |\qf_\nu(\tau) - \qf_{\nu'}(\tau)|,
\]
where $u_{(0, 1)}$ denotes the uniform distribution in $(0, 1)$.
Sometimes we will say the sequence $\nu_1, \nu_2, \ldots$ converges to $\nu_\infty$; when we say this, we mean convergence in $1$-Wasserstein distance: $\lim_{n \rightarrow \infty} \wass(\nu_n, \nu_\infty) = 0$.
The \emph{supremum $1$-Wasserstein distance} is defined for $\eta, \eta' \in \crvs^{\states \times \cs}$ as
\begin{equation}
    \overline{\wass}(\eta, \eta') \doteq \sup_{s \in \states, \xc \in \cs} \wass(\eta(s, \xc), \eta'(s, \xc)).
    \label{eq:sup-wass}
\end{equation}
With a slight abuse of notation, we let $\wass(\nu) \doteq \wass(\nu, \delta_0)$ and $\overline{\wass}(\eta) \doteq \sup_{s \in \states, \xc \in \cs}\overline{\wass}(\eta(s, \xc), \delta_0)$.

Given a policy $\pi \in \historybased$, we define its \emph{return distribution function} $\eta^\pi : \states \times \cs \rightarrow \crvs$ by $\eta^\pi(s, \xc) \doteq \df(G^\pi(s, \xc))$ (for $(s, \xc) \in \states \times \cs$).

We will make ample use of Banach's fixed point theorem \citep[Theorem 1, p.~77,][]{szepesvari2022algorithms} and the following spaces:
\begin{align*}
    (\crvs, \wass) &\doteq \{ \nu \in \crvs : \wass(\nu) < \infty \}, \\
    (\crvs^{\states \times \cs}, \overline{\wass}) &\doteq \{ \eta \in \crvs^{\states \times \cs} : \overline{\wass}(
    \eta) < \infty \}.
\end{align*}
These spaces are complete as shown in \cref{lem:complete-spaces}, \cref{app:additional-theoretical-results}.
\Cref{ass:bounded-first-moment-rewards} combined with $\gamma < 1$ or a finite-horizon MDP ensure that the return distributions of all policies are uniformly bounded, that is,
$\sup_{\pi \in \historybased}\overline{\wass}(\eta^\pi) < \infty$.

Given a stationary policy $\pi \in \stationary$, we define the \emph{\stock{}-augmented distributional Bellman operator} $T_\pi : (\crvs^{\states \times \cs}, \overline{\wass}) \rightarrow (\crvs^{\states \times \cs}, \overline{\wass})$
for $\eta \in (\crvs^{\states \times \cs}, \overline{\wass})$
as follows:
$(T_\pi \eta)(s, \xc)$ is the distribution of $R_{t+1} + \gamma G(S_{t+1}, \xC_{t+1})$ when $(S_t, \xC_t) = (s, \xc)$, $A_t \sim \pi(S_t, \xC_t)$, and $G(s, \xc) \sim \eta(s, \xc)$.
We require that if $s$ is terminal then $(T_{\pi}\eta)(s, \xc) = \delta_0$ for all $\eta \in (\crvs^{\states \times \cs}, \overline{\wass})$ and $\xc \in \cs$.

On occasion, we will refer back to classic RL operators for comparison against the distributional case.
We will denote the space of possible (state-) value functions by $(\reals^\states, \| \cdot \|_{\infty}) \doteq \left\{ V \in \reals^{\states} : \sup_{s \in \states} | V(s) | < \infty \right\}$.
To avoid introducing further notation, we will also denote the classic Bellman operator by $\classicBellman_\pi$.
Whether the Bellman operator is classic or distributional will be clear from whether its argument is a return distribution function or a value function.

We let
$x_+ \doteq \max\{ x, 0 \}$, $x_- \doteq \min\{ x, 0 \}$, and $\Ind(\cdot)$ be the indicator function.

\section{\STOCK{}-Augmented \PROBLEM{}}

\subsection{Problem Formulation}
\label{sec:problem-formulation}

We are concerned with building intelligent agents that can do various things.
When the agent can be expressed in terms of its behavior (a policy)
and the outcome of the agent acting can be modeled as the \stock{}-augmented discounted return generated by that policy, 
we can frame the problem of building intelligent agents as an optimization problem.
A person looking to build an intelligent agent in this framework (we will call them \emph{the designer}) is thus tasked with expressing what they want of agents as an objective to be optimized---where the better the agent, the higher the objective value of its policy.\footnote{In practice, the designer is also tasked with modeling the environment as an MDP. In standard RL, this means designing the states, actions and rewards. \Stock{}-augmented MDPs additionally require designing the \stock{} and the pseudo-rewards.}

We propose to control the distribution of the quantity $\xc_0 + G^\pi(s_0, \xc_0)$,\footnote{
In terms of a problem/solution separation, incorporating \stock{} is part of the solution (distributional DP).
However, because the scope of our work is DP, it is convenient for our presentation to incorporate \stock{} augmentation and the offset by $\xc_0$ as part of the problem (\problem{}).
The simpler formulation without \stock{} augmentation or the offset is limiting for distributional DP:
\citet{marthe2024beyond} studied \problem{} without \stock{} augmentation in the finite-horizon undiscounted setting, and concluded that only exponential utilities could be optimized through distributional DP---the same class that classic DP can optimize.
On the other hand, as our analysis will show, the distributional DP with \stock{} can optimize a broader class of objectives than without, and, surprisingly, than classic DP with \stock{} augmentation.
}
which is the return generated by $\pi$ from the initial augmented state $(s_0, \xc_0) \in \states \times \cs$, offset by the initial \stock{} $\xc_0$.
We want an objective that quantifies how preferred $\df(\xc_0 + G^\pi(s_0, \xc_0))$ is for each policy $\pi$, so that we can phrase the problem of finding the most preferred policy.
We can accomplish this with a statistical functional $K : (\crvs, \wass) \rightarrow \reals$ that assigns a real number to each possible distribution of $\xc_0 + G^\pi(s_0, \xc_0)$, to phrase the optimization problem as:
\begin{equation}
    \sup_{\pi \in \historybased} K\df\left(\xc_0 + G^\pi(s_0, \xc_0) \right).
    \label{eq:distributional-functional}
\end{equation}
As an example, 
the standard RL problem can be expressed in \Cref{eq:distributional-functional}
by taking $K$ to be the expectation:
\[
    \sup_{\pi \in \historybased} \E(\xc_0 + G^\pi(s_0, \xc_0)) = \xc_0 + \sup_{\pi \in \historybased} \E(G^\pi(s_0, \xc_0)).
\]

The optimization, for the moment, is over the (most general) class of history-based policies $\historybased$.
In standard RL,
this problem formulation \citep[adopted, for example, by][]{altman1999constrained} differs from the more frequent optimization over stationary Markov policies \citep[adopted, for example, by][]{sutton2018reinforcement,szepesvari2022algorithms}, 
but the two formulations are equivalent in MDPs because of the existence of optimal stationary Markov policies~\citep{puterman2014markov}.
For \stock{}-augmented \problem{}, we have elected to introduce the problem in terms of history-based policies, and to address the existence of optimal stationary Markov policies on the solution side of the results (in connection to DP; see~\cref{app:history-based-policies}).

Because the supremum in \Cref{eq:distributional-functional} is over all history-based policies, it makes sense to talk about optimizing
$K\df\left(\xc_0 + G^\pi(s_0, \xc_0) \right)$ simultaneously for all $(s_0, \xc_0) \in \states \times \cs$.
We can state this problem concisely, using an objective functional applied to the return distribution function $\eta^\pi$:
\begin{definition}[\STOCK{}-Augmented \PROBLEM{}]
\label{def:FK}
Given \\
$K : (\crvs, \wass) \rightarrow \reals$,
define the \emph{\stock{}-augmented objective functional} $F_K : (\crvs^{\states \times \cs}, \overline{\wass}) \rightarrow \reals^{\states \times \cs}$ as
\[
    (F_K \eta)(s, \xc) \doteq K\df(\xc + G(s, \xc)). 
    \tag{$G(s, \xc) \sim \eta(s, \xc)$}
\]

The \emph{\stock{}-augmented \problem{} problem} is 
\begin{equation}
    \sup_{\pi \in \historybased} F_K \eta^\pi. \label{eq:distributional-functional-concise}
\end{equation}
\end{definition}
We will often drop the subscript and refer to a \stock{}-augmented objective as $F$, in which case a corresponding $K$ is implied.
We will also drop $\df$ and write $K(G) = K\df(G)$.

To recap \Cref{eq:distributional-functional-concise}: The \stock{}-augmented \problem{} problem consists of optimizing, over all policies $\pi \in \historybased$, a preference specified by a statistical functional $K : (\crvs, \wass) \rightarrow \reals$, over the distribution of the policy's discounted return offset by the \stock{} ($\xc_0 + G^\pi(s_0, \xc_0)$).
The optimization is considered simultaneously for all $(s_0, \xc_0)$, as allowed by history-based policies.

\subsection{Example: Expected Utilities}
\label{sec:expected-utilities}

\Cref{eq:distributional-functional-concise} provides a flexible problem formulation for controlling $\df(\xc_0 + G^\pi(s_0, \xc_0))$, based on a choice of $K : (\crvs, \wass) \rightarrow \reals$ provided by a designer to capture what they want an agent to achieve.
We have already shown that the RL problem can be recovered by taking $K$ to be the expectation ($K\nu = \E G$, $G \sim \nu$), so what else can we do?
We can obtain an interesting family of objective functionals by considering the expected value of transformations of the return specified by a function $f : \cs \rightarrow \reals$: $K\nu = \E f(G)$ ($G \sim \nu$).
These are the \emph{expected utilities}, which have been widely studied in decision-making theory \citep{von2007theory},
and also used for sequential decision-making in RL~\citep{bauerle2014more,bowling2023settling}.
\begin{definition}
\label{def:expected-utility}
A \stock{}-augmented objective functional $F_K$ is an \emph{expected utility} if there exists $f : \cs \rightarrow \reals$ such that
\[
    K\nu = \E f(G).
    \tag{$G \sim \nu$}
\]
In this case, we write $F_K = \Uf$, which can be written as
\[
    (\Uf \eta) (s, \xc) \doteq \E f(\xc + G(s, \xc)).
    \tag{$G(s, \xc) \sim \eta(s, \xc)$}
\]
\end{definition}

\Cref{tab:utility-examples} gives examples of \problem{} problems resulting from different choices of $f$ in the scalar case\footnote{Expected utilities are not restricted to the scalar case, as implied by \cref{def:expected-utility}, since the domain of $f$ is $\cs$.
We provide some concrete examples in \cref{sec:applications} of expected utilities for the vector-valued case.}
($\cs = \reals$), with some notable risk-sensitive examples: Maximizing conditional value-at-risk \citep{bauerle2011markov,chow2014algorithms,lim2022distributional} and maximizing the probability of the discounted return being above a threshold.
Recall that the choice of initial \stock{} $\xc_0$ is ``up to the user'' and can be made as a function of the starting state $s_0$.
\begin{table}[tb]
    \centering
    \begin{tabularx}{\textwidth}{Xcc}
        \toprule
        Problem & $f(x)$ & Formulation\\
        \midrule
        Standard RL & $x$ & \begin{tabular}{@{}c@{}}$\sup_{\pi \in \historybased} \E(\xc_0 + G^\pi(s_0, \xc_0))$ \\ $\equiv \sup_{\pi \in \historybased} \E G^\pi(s_0, \cdot)$\end{tabular} \\
        Minimize the expected absolute distance to a target $\xc_0$ (\cref{sec:generating-returns}) \ & $-|x|$ & $\inf_{\pi \in \historybased} \E |G^\pi(s_0, -\xc_0) - \xc_0| $ \\
        Optimizing $\tau$-CVaR (conditional value-at-risk, \cref{sec:cvar}) & $x_-$ & $\inf_{\pi \in \historybased, \xc_0} \frac{1}{\tau}\int_0^\tau \qf_{\eta^\pi(s_0, \xc_0)}(t)\mathrm{d}t$ \\
        \midrule
        Maximize the probability of the return above a threshold $\xc_0$ & $\Ind(x > 0)$ & $\sup_{\pi \in \historybased} \Prob(G^\pi(s_0, -\xc_0) > \xc_0)$ \\
        Minimize the expected square distance to a target $\xc_0$ & $-x^2$ & $\inf_{\pi \in \historybased} \E\left((G^\pi(s_0, -\xc_0) - \xc_0)^2\right)$\\
        \midrule
        Maximize the probability of the return above a threshold plus a margin $\xc_0 + \xc$ & $\Ind(x > \xc)$ & $\sup_{\pi \in \historybased} \Prob\left(G^\pi(s_0, -\xc_0) > \xc_0 + \xc\right)$ \\
        \bottomrule
    \end{tabularx}
    \caption{Example problems that can be formulated as optimizing an expected utility, with the respective choices of $f$ and the formulation.
    }
    \label{tab:utility-examples}
\end{table}

We will later show that the examples in the first part of the table can be optimized by distributional DP both in the finite-horizon and discounted cases, the ones in the second part of the table can be optimized in the finite-horizon case, and the example in the third part can only be optimized in the finite-horizon undiscounted case (see \cref{thm:value-iteration,thm:policy-iteration,sec:ddp-conditions,app:lipschitzness}).

We will also establish that distributional DP can, in fact, optimize any expected utility in the finite-horizon undiscounted case (see \cref{lem:expected-utility-conditions}).
Going beyond expected utilities, we will see that is an open question whether it is possible for distributional DP to optimize any non-expected utility in the infinite-horizon discounted case, but we provide examples that can be optimized in the finite-horizon case (see \cref{sec:beyond-expected-utilities}).

\section{Distributional Dynamic Programming}
\label{sec:ddp}

Dynamic programming~\citep{bertsekas1996neuro,sutton2018reinforcement} is at the heart of RL theory and many RL algorithms.\footnote{As pointed out by \citet{szepesvari2022algorithms} many RL algorithms can be thought of as dynamic programming methods modified to cope with scale and complexity of practical problems.}
For this reason, we have chosen to establish the basic theory of solving \stock{}-augmented \problem{} by studying how we can solve these problems using DP.
We refer to the solution methods we introduce as \emph{distributional dynamic programming}.
As in the case of distributional DP for policy evaluation~\citep[Chapter 5;][]{bellemare2023distributional}, return distribution functions (in $(\crvs^{\states \times \cs}, \overline{\wass})$) are the main object of distributional value/policy iteration, whereas, in contrast, classic DP, namely value/policy iteration, work directly with value functions~\citep[see, for example,][]{szepesvari2022algorithms}.

\subsection{Distributional Value Iteration}
\label{sec:ddp:value-iteration}

Classic value iteration computes the iterates $V_1, V_2, \ldots$ satisfying, for $n \geq 0$,
\begin{equation}
    V_{n+1} = \sup_{\pi \in \stationary} \classicBellman_\pi V_n,
    \label{eq:classic-value-iterates}
\end{equation}
and the procedure enjoys the following optimality guarantees.
In finite-horizon MDPs, $V_n$ is optimal if $n$ is at least the horizon of the MDP and in the discounted case \citep[Section 2.4;][]{szepesvari2022algorithms}:
\begin{equation}
    V^* - V_n \leq \gamma^n \| V^* - V_0 \|_\infty 
    \label{eq:classic-value-iteration-iterate-bound}
\end{equation}
pointwise for all $s \in \states$,
where $V^* \doteq \sup_{\pi \in \historybased} V^\pi$ and $V^\pi$ denotes the value function of a policy $\pi$.

Note how the bounds are distinct for the finite-horizon case and the discounted case.
This distinction recurs in results for both classic and distributional value/policy iteration, and it will merit further discussion in the case of distributional DP.

In classic value iteration, the iterates correspond to the values of the objective functional being optimized, and the iteration in \Cref{eq:classic-value-iterates} makes a one-step decision that maximizes that objective functional.
We typically use the value iterates to obtain policies via a greedy selection, and leverage a near-optimality guarantee for these greedy policies.
We say $\pitilde_n$ is a greedy policy with respect to $V_n$ if it satisfies the following:
\[
    T_{\pitilde_n}V_n = \sup_{\pi \in \stationary} \classicBellman_\pi V_n.
\]
Classic value iteration results give us the following optimality guarantees for the greedy policies:
In finite-horizon MDPs, $\pitilde_n$ is optimal when $n$ is at least the horizon of the MDP, and in the discounted case 
(Section 2.4, \citealp{szepesvari2022algorithms}; \citealp{singh1994upper}):
\begin{equation}
    V^* - V^{\pitilde_n} \leq \frac{2\gamma^n}{1 - \gamma} \| V^* - V_0 \|_\infty.
    \label{eq:classic-value-iteration-policy-bound}
\end{equation}

Distributional value iteration, while similar to value iteration, maintains distributional iterates $\eta_1, \eta_2, \ldots \in (\crvs^{\states \times \cs}, \overline{\wass})$, which means the iterates no longer correspond to values of the objective functional.
The distributional analogue of \Cref{eq:classic-value-iterates} makes a one-step decision that maximizes $F_K$, and this iteration of locally optimal one-step decisions gives guarantees similar to the classic case.
\Cref{thm:value-iteration} formalizes this claim:\footnote{To simplify the presentation, we have chosen to present the distributional DP results upfront, and discuss the conditions on the objective functional $F_K$ in \cref{sec:ddp-conditions}.}
\begin{restatable}[Distributional Value Iteration]{theorem}{distributionalvalueiterationtheorem}
\label{thm:value-iteration}
If $K : (\crvs, \wass) \rightarrow \reals$ is indifferent to mixtures and indifferent to $\gamma$,
then for every $\eta_0 \in (\crvs^{\states \times \cs}, \overline{\wass})$,
if the iterates $\eta_1, \eta_2, \ldots$ satisfy (for $n \geq 0$)
\[
    F_K\eta_{n+1} = \sup_{\pi \in \stationary} F_K T_\pi\eta_n,
    \tag{Distributional Value Iterates}
\]
and the policies $\pibar_0, \ldots, \pibar_n$ satisfy (for $n \geq 0$),
\[
    F_KT_{\pibar_n}\eta_n = \sup_{\pi \in \stationary} F_K T_\pi\eta_n,
    \tag{Greedy Policies}
\]
then the following hold.

\emph{Finite-horizon case:} for all $n$ greater or equal to the horizon of the MDP,
\begin{equation}
    F_K \eta_n = \sup_{\pi \in \historybased} F_K \eta^\pi,
    \label{eq:finite-horizon-value-iterates}
\end{equation}
and
\begin{equation}
    F_K \eta^{\pibar_n} = \sup_{\pi \in \historybased} F_K \eta^\pi.
    \label{eq:finite-horizon-value-greedy}
\end{equation}

\emph{Discounted case ($\gamma < 1$):}
If $K$ is $L$-Lipschitz, then for all $n \geq 0$
\begin{equation}
    \sup_{\pi \in \historybased} F_K \eta^\pi - F_K \eta_n \leq L\gamma^n \cdot \sup_{\pi \in \markov}\overline{\wass}(\eta_0, \eta^\pi),
    \label{eq:discounted-value-iterates}
\end{equation}
and
\begin{equation}
    \sup_{\pi \in \historybased} F_K \eta^\pi - F_K \eta^{\pibar_n} \leq 
    L\gamma^n \cdot \left(\frac{1}{1 - \gamma}\sup_{\pi \in \stationary}\overline{\wass}(T_{\pi}\eta_0, \eta_0) + \sup_{\pi \in \markov}\overline{\wass}(\eta_0, \eta^\pi)\right).
    \label{eq:discounted-value-greedy}
\end{equation}
\end{restatable}

Next, we discuss a number of aspects of our value iteration result.

\emph{Iterates may not converge.}
The guarantees in \cref{thm:value-iteration} only apply to values of the objective functional $F_K \eta_n$, and
iterate convergence cannot be guaranteed because multiple iterates may be tied at the optimum.
Iterate non-convergence has been identified before in distributional RL, as multiple return distributions can be optimal~\citep[Example 7.11, p.~210, ][]{bellemare2023distributional}.

\emph{Comparison to classic DP bounds in the finite-horizon case.}
The guarantees for finite-horizon MDPs are essentially the same for distributional and classic value iteration: Namely, optimality after iterating at least as many times as the MDP horizon.

\emph{Comparison to classic DP bounds in the discounted case.}
In the discounted case, the bounds for distributional value iteration (\cref{eq:discounted-value-iterates,eq:discounted-value-greedy}) are similar to the classic value iteration bounds (\cref{eq:classic-value-iteration-iterate-bound,eq:classic-value-iteration-policy-bound}) with three notable differences:
\begin{enumerate}[label=\roman*)]
    \item The bounding terms are $1$-Wasserstein distances, rather than $\infty$-norms. 
    This is inherent to the fact that our iterates are distributional.
    \item The Lipschitz constant of $K$ is present. 
    This constant is $1$ when $F_K$ is the standard RL objective functional.
    \item The classic value iteration bounds are given in terms of $V^*$, but the distributional value iteration bounds are not. This is because it is still an open question whether an optimal return distribution $\eta^*$ exists in the discounted case in general. However, if we assume $\eta^*$ exists, we can replace the bounding term in \Cref{eq:discounted-value-iterates} with $L\gamma^n \cdot \overline{\wass}(\eta_0, \eta^*)$, which is comparable to the classic DP bounds.
\end{enumerate}
The considerations above apply similarly to the greedy policy bounds for distributional and classic DP.

When an optimal return distribution $\eta^*$ exists, we can also show an optimality guarantee for policies that are greedy with respect to $\eta^*$, similar to the classic case:
\begin{restatable}[Greedy Optimality]{theorem}{optimalitytheorem}
\label{thm:optimality}
If $K : (\crvs, \wass) \rightarrow \reals$ is indifferent to mixtures and indifferent to $\gamma$,
and if: i) the MDP has finite horizon; or ii) $\gamma < 1$ and $K$ is Lipschitz,
then the following hold.

There exists an optimal return distribution $\eta^* \in \crvs^{\states \times \cs}$ satisfying
\[
    F_K \eta^* = \sup_{\pi \in \historybased} F_K \eta^\pi,
\]
iff the supremum on the right-hand side is attained (that is, an optimal policy exists).

If such $\eta^*$ exists, then any greedy policy with respect to $\eta^*$ is optimal (and thus attains the supremum above).
\end{restatable}

\subsection{Distributional Policy Iteration}
\label{sec:ddp:policy-iteration}

Classic policy iteration computes the iterates $\pi_1, \pi_2, \ldots$ satisfying, for $n \geq 0$,
\begin{equation*}
    \classicBellman_{\pi_{n+1}}V^{\pi_n} = \sup_{\pi \in \stationary} \classicBellman_\pi V^{\pi_n},
\end{equation*}
that is, each iterate $\pi_{n+1}$ is greedy with respect to the value of the previous iterate $\pi_n$.
In finite-horizon MDPs, $V^{\pi_n}$ is optimal if $n$ is at least the horizon of the MDP.
In the discounted case, we have~\citep[Proposition 2.8, p.~45;][]{bertsekas1996neuro}:
\begin{equation*}
    V^* - V^{\pi_n} \leq \gamma^n \| V^* - V^{\pi_0} \|_\infty.
\end{equation*}

Distributional policy iteration is similar to its classic counterpart (he main difference being that the objective functional $F_K$ determines the greedy policy selection) and also enjoys similar guarantees, as formalized by
\cref{thm:policy-iteration}:
\begin{restatable}[Distributional Policy Iteration]{theorem}{distributionalpolicyiterationtheorem}
\label{thm:policy-iteration}
If $K : (\crvs, \wass) \rightarrow \reals$ is indifferent to mixtures and indifferent to $\gamma$,
for every stationary policy $\pi_0 \in \stationary$ if the iterates $\pi_1, \pi_2, \ldots$ satisfy (for $n \geq 0$)
\[
    F_K T_{\pi_{n+1}}\eta^{\pi_n} = \sup_{\pi \in \stationary} F_K T_\pi\eta^{\pi_n} 
    \tag{Distributional Policy Iterates}
\]
then the following hold.

\emph{Finite-horizon case:} 
For all $n$ greater or equal to the horizon of the MDP,
\begin{equation}
    F_K \eta^{\pi_n} = \sup_{\pi \in \historybased} F_K \eta^\pi.
    \label{eq:finite-horizon-policy-iterates}
\end{equation}

\emph{Discounted case ($\gamma < 1$):} 
If $K$ is $L$-Lipschitz, then for all $n \geq 0$
\begin{equation}
    \sup_{\pi \in \historybased} F_K \eta^\pi - F_K \eta^{\pi_n} \leq L\gamma^n \cdot \sup_{\pi \in \markov}\overline{\wass}(\eta^{\pi_0}, \eta^\pi),
    \label{eq:discounted-policy-iterates}
\end{equation}
\end{restatable}

\vspace{1em}
 
\emph{Comparison to classic policy iteration bounds.}
Essentially the same considerations apply here as in \cref{sec:ddp:value-iteration}, for comparing the respective value iteration bounds.
This is because we obtain the policy iteration bounds from the value iteration bounds, using the fact that $V^{\pi_n} \geq V_n$ for classic DP, and $F_K \eta^{\pi_n} \geq F_K \eta_n$ for distributional DP (see the proof of \cref{thm:policy-iteration} in \cref{app:ddp-main-results-proofs}).

\subsection{Conditions Overview}
\label{sec:ddp-conditions}

\Cref{thm:value-iteration,thm:policy-iteration} only apply to objective functionals that satisfy certain properties: Indifference to mixtures and indifference to $\gamma$ in the finite-horizon case, plus Lipschitz continuity in the infinite-horizon discounted case.
In this section we give an overview of these conditions and test them: How restrictive are these conditions? Can they be weakened?
The proofs for the results in this section can be found in \cref{app:conditions}.
Recall that we are abusing notation and writing $K(G) = K\df(G)$.
\begin{definition}[Indifference to Mixtures (of Initial Augmented States)]
\label{def:indifference-to-mixtures}
We say $K : (\crvs, \wass) \rightarrow \reals$ is \emph{indifferent to mixtures} (of initial augmented states) if for every 
$\eta, \eta' \in (\crvs^{\states \times \cs}, \overline{\wass})$
such that
\[
    K\eta(s, \xc) \geq K\eta'(s, \xc),
\]
for all $(s, \xc) \in \states \times \cs$, then for all random variables $(S, \xC)$ taking values in $\states \times \cs$
we also have
\[
    K(G(S, \xC)) \geq K(G'(S, \xC)).
    \tag{$G(s, \xc) \sim \eta(s, \xc)$, $G'(s, \xc) \sim \eta'(s, \xc)$}
\]
\end{definition}
\begin{definition}[Indifference to $\gamma$]
\label{def:indifference-to-gamma}
We say $K : (\crvs, \wass) \rightarrow \reals$ is \emph{indifferent to $\gamma$} if, for every $\nu, \nu' \in (\crvs, \wass)$
\[
    K\nu \geq K\nu' \Rightarrow K(\gamma G) \geq K(\gamma G').
    \tag{$G \sim \nu$, $G' \sim \nu'$}
\]
\end{definition}
\begin{definition}[Lipschitz Continuity]
\label{def:lipschitz}
We say $K : (\crvs, \wass) \rightarrow \reals$ is \emph{$L$-Lipschitz} (or \emph{Lipschitz}, for simplicity) if there exists $L \in \reals$ such that
\[
     \sup_{\substack{\nu, \nu':\\ \wass(\nu) < \infty \\ \wass(\nu') < \infty \\ \wass(\nu, \nu') > 0}} \frac{|K\nu - K\nu'|}{\wass(\nu, \nu')} \leq L.
\]
$L$ is the \emph{Lipschitz constant} of $K$.
\end{definition}

We believe that in general these conditions are fairly easy to verify for different choices of $K$.
As an example, \cref{lem:expected-utility-conditions} does part of the verification for expected utilities.
\begin{restatable}[Conditions for Expected Utilities]{lemma}{expectedutilityconditions}
\label{lem:expected-utility-conditions}
Let $\Uf$ be an expected utility, which is an objective functional $F_K$ with $K\nu = \E f(G)$ ($G \sim \nu$).
Then the following hold:
\begin{enumerate}
    \item $K$ is indifferent to mixtures. \label{lem:expected-utility-conditions:indifferent-to-mixtures}
    \item $K$ is indifferent to $\gamma$ iff there exists $\alpha \in (0, 1]$ such that $\gamma < 1 \Rightarrow \alpha < 1$ and, for all $\xc \in \cs$, \label{lem:expected-utility-conditions:indifferent-to-gamma}
        \begin{equation}
            f(\gamma \xc) = \alpha f(\xc) + (1 - \alpha) f(0).
            \label{eq:expected-utility-conditions:indifferent-to-gamma}
        \end{equation}
    \item $K$ is $L$-Lipschitz iff $f$ is $L$-Lipschitz. \label{lem:expected-utility-conditions:lipschitz}
\end{enumerate}
\end{restatable}

The condition for indifference to $\gamma$ is interesting because it means $\xc \mapsto f(\xc) - f(0)$ is positively homogeneous with degree $\log_\gamma \alpha$.

If we refer back to \cref{tab:utility-examples}, we see that the choices of $f$ in the first part of the table satisfy all three conditions, so distributional DP can optimize the corresponding $\Uf$ both in the finite-horizon and discounted cases.
The choices of $f$ in the second part of the table are not Lipschitz, so we know that DP can optimize the corresponding $\Uf$ in the finite-horizon setting.
The choice of $U_f$ in the third part of the table is neither Lipschitz nor indifferent to $\gamma < 1$, so distributional DP is only guaranteed to optimize the corresponding $\Uf$ in the finite-horizon \emph{undiscounted} setting.
A consequence of \cref{lem:expected-utility-conditions}, since indifference to $\gamma = 1$ is trivially true, is that distributional DP can optimize any expected utility in the finite-horizon undiscounted case. 

We have investigated the three conditions (\cref{def:indifference-to-gamma,def:indifference-to-mixtures,def:lipschitz}) to determine how restrictive they are.
We have found that indifference to mixtures and indifference to $\gamma$ are necessary and sufficient, so they are minimal.
In the absence of either, even a basic greedy optimality guarantee (\cref{thm:optimality}) fails:
\begin{restatable}{proposition}{indifferencesnecessary}
\label{prop:indifferences-necessary}
If $K : (\crvs, \wass) \rightarrow \reals$ is not indifferent to mixtures or not indifferent to $\gamma$, then there exists an MDP, an $\eta^* \in  (\crvs^{\states \times \cs}, \overline{\wass})$ and a $\pibar \in \stationary$ such that $\pibar$ is greedy with respect to $\eta^*$ and
\[
    F_K \eta^* = \sup_{\pi \in \historybased} F_K\eta^\pi,
\]
however, for some $(s, \xc) \in \states \times \cs$
\[
    F_K \eta^{\pibar}(s, \xc) < \sup_{\pi \in \historybased} F_K\eta^\pi(s, \xc).
\]
\end{restatable}

We have found that the relationship between Lipschitz continuity and the infinite-horizon discounted case is less clear, and it is still an open question whether this property is necessary.
However, we can show that indifference to mixtures and indifference to $\gamma$ are not sufficient for the infinite-horizon discounted case, so there is a real distinction between the finite-horizon and infinite-horizon discounted cases, in line with our results for distributional DP (\cref{thm:value-iteration,thm:policy-iteration}).

In \cref{app:lipschitzness}, we show an instance where distributional value/policy iteration fail for the expected utility $\Uf$ with $f(x) = \Ind(x > 0)$, even though the starting iterate is optimal.
The intuition for this is simple and we outline it here (the key is to exploit the fact that $f$ is not continuous).
Consider an MDP with $\states = \{s_0, s_1\}$, $\actions = \{ a_0, a_1 \}$, $r(\cdot, a_i) = i$ and $\gamma < 1$.
The initial state is $s_0$ and $s_1$ is terminal, and taking $a_i$ in $s_0$ transitions to $s_i$.
A stationary policy $\pi \in \stationary$ satisfying $\pi(a_1 | s_0, \cdot)$ is optimal, so let us denote it by $\pi^*$ and its return distribution function by $\eta^*$.
Thanks to $\eta^*$, we have
\[
    \Uf T_\pi \eta^* = \Uf \eta^*
\]
for all $\pi \in \stationary$, including a policy that always selects $a_0$, and, in fact, by induction, any non-stationary policy that selects $a_0$ finitely many times is also optimal, even though selecting $a_0$ \emph{always} is suboptimal.
In the case of distributional value iteration with $\eta_0 = \eta^*$, if we take $\pibar_n$ to be the policy that always selects $a_0$ we will have $\Uf \eta_n = \Uf \eta^*$ for all $n$, however, $\Uf \eta^{\pibar_n} < \Uf \eta^*$ also for all $n$, which means distributional value iteration has failed.
Distributional policy iteration fails too, except that when starting from $\pi^*$ every other iterate may be suboptimal depending on how ties are broken.

The assumption on Lipschitz continuity of $f$ for the infinite-horizon discounted case prevents failures like the example above (which we attributed to the fact that $f$ is not continuous).
In \cref{app:lipschitzness} we also show that the lack of Lipschitz continuity affects our ability to evaluate policies, in the sense that if we take $f(x) = x^2$ (which is continuous but not Lipschitz) we can construct an MDP and a policy $\pi \in \stationary$ such that $T^n_{\pi}\eta$ converges to $\eta^\pi$ as $n \rightarrow \infty$, but $\Uf T^n_{\pi}\eta$ does not converge uniformly to $\Uf\eta^\pi$ (though it converges pointwise).

It is unclear whether the lack of uniform convergence for non-Lipschitz $f$ can be translated to a failure of distributional value/policy iteration, however we have a failure case example of a discontinuous $f$, so it suggests that some property related to continuity of $f$ (and $K$ more generally) is necessary.

\subsection{Analysis Overview}

The valuable insight in this work is that we can use distributional DP to optimize different objective functionals $F_K$ of the (\stock{}-augmented) return distribution (and a broader class than without).
Once we identify the right conditions and the core components for distributional value/policy iteration to work, the remaining work is relatively straightforward: We retrace the steps of classic DP and ensure technical correctness.
Most of the challenge is, in fact, ensuring technical correctness with a generic objective functional---for example, we need to be careful to make correct statements about convergence; we cannot rely on the existence of an optimal return distribution $\eta^*$ or on the convergence of distributional value iterates.

In this section, we give an outline of our analysis with the most interesting points and a focus on how we can obtain asymptotic optimality guarantees.
This will allow us to understand how the different conditions factor into our proofs, and how they work in essence.
We defer the technical proofs to \cref{app:ddp}, including details about performance bounds.

A fundamental component for DP is monotonicity.
In classic RL \citep[see Lemma 2.1, p.~21,][]{bertsekas1996neuro}, it states that if we have $V \geq V'$, then following a policy $\pi$ for one step and having a value of $V$ afterward is always better than following the same policy but obtaining a value of $V'$ afterward, regardless of the policy $\pi$.
That is, we have
\[
    V \geq V' \Rightarrow \classicBellman_{\pi}V \geq \classicBellman_{\pi}V'
\]
for all $\pi \in \stationary$.
In distributional DP, it translates to the following:
\begin{restatable}[Monotonicity]{lemma}{monotonicitylemma}
\label{lem:monotonicity}
If $K : (\crvs, \wass) \rightarrow \reals$ is indifferent to mixtures and indifferent to $\gamma$, then,
for every $\pi \in \stationary$, the distributional Bellman operator $T_\pi$ is \emph{monotone} (or \emph{order-preserving}) with respect to the preference induced by $F_K$ on $(\crvs^{\states \times \cs}, \overline{\wass})$.
That is, for every stationary policy $\pi \in \stationary$ and $\eta, \eta' \in (\crvs^{\states \times \cs}, \overline{\wass})$,
we have
\[
    F_K\eta \geq F_K\eta' \Rightarrow F_K T_\pi \eta \geq F_K T_\pi \eta'.
\]
\end{restatable}

Monotonicity is a powerful result that underpins value iteration, policy iteration and also \emph{policy improvement}.\footnote{To underscore the importance of monotonicity, we note that the result in \cref{prop:indifferences-necessary} holds essentially because monotonicity is equivalent to $K$ being indifferent to mixtures and indifferent to $\gamma$, and it is the absence of monotonicity that causes greedy optimality (\cref{thm:optimality}) to fail.}
Classic policy improvement \citep[see Proposition 2.4, p.~30,][]{bertsekas1996neuro} states that if a policy $\pitilde$ is greedy with respect to $V^\pi$, then $\pitilde$ is better than $\pi$ ($V^{\pitilde} \geq V^\pi$).
We have a similar result for distributional DP, given as \cref{lem:policy-improvement}.
This result is of particular interest here because its proof gives a good sense of how to provide asymptotic guarantees for distributional DP, and how the different conditions factor in, in particular how departing from the standard RL case in classic DP demands special attention to convergence guarantees.
\begin{lemma}[Distributional Policy Improvement]
\label{lem:policy-improvement}
If $K : (\crvs, \wass) \rightarrow \reals$ is indifferent to mixtures and indifferent to $\gamma$,
and if: i) the MDP has finite horizon; or ii) $\gamma < 1$ and $K$ is Lipschitz,
then for $\eta \in (\crvs^{\states \times \cs}, \overline{\wass})$ and any stationary policy $\pibar \in \stationary$ if
\begin{equation}
    F_K T_{\pibar}\eta \geq F_K \eta,
    \label{eq:one-step-improvement}
\end{equation}
then
\[
   F_K \eta^{\pibar} \geq F_K \eta.
\]
In particular, for any stationary policy $\pi \in \stationary$, if $\pibar$ satisfies
\[
    F_KT_{\pibar}\eta^{\pi} = \sup_{\pi' \in \stationary} F_K T_{\pi'}\eta^{\pi},
    \tag{Greedy Policy}
\]
then \Cref{eq:one-step-improvement} is satisfied with $\eta = \eta^{\pi}$ and we have
\[
   F_K \eta^{\pibar} \geq F_K \eta^{\pi}.
   \tag{Greedy Policy Improvement}
\]
\end{lemma}

\begin{proof}
We write $F = F_K$ for simplicity, and fix $\eta \in (\crvs^{\states \times \cs}, \overline{\wass})$ arbitrary.
Indifference to mixtures and indifference to $\gamma$ give us monotonicity.
By induction, for all $n \geq 1$, if we assume that \Cref{eq:one-step-improvement} holds and that $F T^n_{\pibar} \eta \geq F \eta$, then
\begin{align*}
    F T^{n+1}_{\pibar} \eta 
    &= F T_{\pibar} T^n_{\pibar} \eta \\
    &\geq F T_{\pibar} \eta
    \tag{Monotonicity, induction assumption} \\
    &\geq F \eta
    \tag{\cref{eq:one-step-improvement}}.
\end{align*}
Thus, if \Cref{eq:one-step-improvement} holds, then, for all $n \geq 1$,
\begin{equation}
    F T^n_{\pibar} \eta \geq F \eta.
    \label{eq:policy-improvement-n-steps}
\end{equation}

In the finite-horizon case, we can take $n$ to be the horizon of the MDP and the result follows, since $T^n_{\pibar}  \eta = \eta^{\pibar}$.

In the infinite-horizon discounted case,
the induction argument is not enough to show that $F \eta^{\pibar}  \geq F \eta$, since we need \Cref{eq:policy-improvement-n-steps} to hold in the limit. 
In this case,
we have $\gamma < 1$, $T_{\pibar}$ is a contraction
(see \cref{lem:distributional-policy-evaluation} and \citealp[Proposition 4.15, p.~88, ][]{bellemare2023distributional})
and $\overline{\wass}(\eta) < \infty$, so $T^n_{\pibar} \eta$ converges to $\eta^{\pibar}$.
$K$ Lipschitz implies $F$ Lipschitz by \cref{prop:K-Lipschitz-F-Lipschitz}, and because $F$ is Lipschitz, the convergence of $T^n_{\pibar} \eta$ to $\eta^{\pibar} $ implies the convergence of $F T^n_{\pibar} \eta$ to $F \eta^{\pibar} $ (see \cref{prop:F-convergence}).
Thus, \Cref{eq:policy-improvement-n-steps} holds in the limit of $n \rightarrow \infty$, which gives the result:
\[
     F \eta^{\pibar} = \lim_{n \rightarrow \infty} F T^n_{\pibar} \eta \geq F \eta.
\]

For the greedy policy improvement result for stationary $\pi$, it suffices to use the fact that $T_{\pi}\eta^{\pi} = \eta^{\pi}$, so the choice of greedy policy gives
\[
F T_{\pibar}\eta^{\pi} = \sup_{\pi' \in \stationary} F T_{\pi'}\eta^{\pi} \geq F T_{\pi}\eta^{\pi} = F \eta^{\pi}.
\]
which gives us \Cref{eq:one-step-improvement}.
\end{proof}

As we can see in the proof of \cref{lem:policy-improvement}, indifference to mixtures and indifference to $\gamma$ are connected to monotonicity, whereas Lipschitz continuity is used to ensure that $F T^n_{\pibar} \eta^\pi$ converges to $F \eta^{\pibar}$ as $n \rightarrow \infty$.
In terms of asymptotic convergence, the main additional technical challenge in the proofs of \cref{thm:value-iteration,thm:policy-iteration} comes from the fact that iterates do not necessarily converge.
However, it is still possible to show that the value of the objective functional converges uniformly for all starting augmented states.
Then the induction argument for chaining improvements (\cref{eq:policy-improvement-n-steps}), and the use of monotonicity and Lipschitz continuity are essentially the same as in the proof of \cref{lem:policy-improvement}.

The condition in \Cref{eq:one-step-improvement} in
\cref{lem:policy-improvement} corresponds to the assumption that $\pibar$ is a one-step improvement on $\eta$.
We can always improve on return distributions of stationary policies with a greedy policy (as the second part of \cref{lem:policy-improvement} shows), however improvement is not always possible for return distributions of non-stationary policies.
To see this, consider a finite-horizon binary-tree MDP and a non-stationary policy $\pi = \pi_1, \pi_2, \ldots$ where each $\pi_t$ has optimal performance on the $t$-th level of the tree, but poor performance in all other states. The policy $(\pibar, \pi_1, \pi_2, \ldots)$ would suffer from the poor performance of all $\pi_t$ because of the time-shift introduced by first following $\pibar$ and then $\pi$.
Importantly, however, when $\eta$ is optimal, even over non-stationary policies, we can satisfy \Cref{eq:one-step-improvement}.
This is used in the proof of the distributional value iteration result (\cref{thm:value-iteration}) for finite-horizon MDPs: In an MDP with horizon $n$, the iterates $\eta_n$ and $\eta_{n+1} = T_{\pibar_{n}}\eta_n$ are optimal (where, recall, $\pibar_n$ is greedy with respect to $\eta_n$), so we can use \cref{lem:policy-improvement} to show that $F\eta^{\pibar_n} \geq F\eta_n$ and therefore $\pibar_n$ is optimal.

\subsection{Previous Distributional Dynamic Programming Results}
\label{sec:ddp:existing-results}

From the vantage point provided by the results in this section, we can better appreciate the landscape of distributional DP in the literature:
The core elements of distributional DP for \stock{}-augmented \problem{} have been studied before, albeit separately, and with different analysis techniques for the standard case and the \stock{}-augmented case.
Our results expand the \stock{}-augmented problems that can be demonstrably solved by distributional DP beyond what was previously known and beyond what can be achieved without \stock{} augmentation, and our analysis adapts the commonly used tools for the standard case~\citep[see, for example,][]{bertsekas1996neuro} to the \stock{}-augmented case.
Moreover, previous work only considered the scalar case ($\cs = \reals$), and we are the first to provide the extension to the vector-valued case ($\cs = \reals^m$).

In the standard case, the theory of distributional DP for policy evaluation has been known prior to this work, as well as distributional value and policy iteration for the standard RL objective~\citep{bellemare2023distributional}.
\citet{marthe2024beyond} posed the \problem{} without \stock{} augmentation and, having demonstrated that only expected utilities could be optimized, introduced distributional value iteration for optimizing expected utilities.
As they show, only affine utilities ($\Uf$ with $f(x) = a x + b$ for $a,b \in \reals$) and exponential utilities ($\Uf$ with $f(x) = a e^{\lambda x} + b$ for $a, b, \lambda \in \reals$) can be optimized without \stock{} augmentation~\citep[in the finite-horizon undiscounted setting; ][]{marthe2024beyond}.

In \stock{}-augmented problems, classic and distributional DP have been considered primarily in the context of optimizing risk measures.
\citet{bauerle2011markov} introduced a value iteration procedure that maintains $\Uf\eta_n$ (with $f(x) = x_-$) as iterates, so it is not distributional.
\citet{bauerle2014more,bauerle2021minimizing} employed the methodology with an augmentation other than \stock{}, for optimizing expected utilities $\Uf$ with continuous and increasing $f$ in the former work, and increasing and convex $f$ in the latter.\footnote{Rather than the \stock{} $C_t$, their augmentation is the pair $\left(\sum_{i=0}^{t-1} \gamma^i R_{i+1}, \gamma^t \right)$. While this setting is the same for the finite-horizon undiscounted case, the settings are different in the infinite-horizon discounted case. Notably, the approach of \citet{bauerle2021minimizing} can optimize increasing and convex objectives, which need not be Lipschitz. We hypothesize that the different augmentation allows removing the requirement for indifference to $\gamma$.}
Parallel to the development of this work, \citet{moghimi2025beyond}
introduced a related policy iteration method that can optimize expected utilities where $f$ has the form $f(x) = \E(x - Z)_-$ and $Z$ satisfies certain conditions~\citep[see Equation 6 in][]{moghimi2025beyond}.
While they built their analysis on the work introduced by \citet{bauerle2011markov}, the iterates used by their method are return distributions, so it is fair to say that their method is \stock{}-augmented distributional policy iteration.

The distributional Q-learning method introduced by 
\citet{lim2022distributional} for optimizing expected utilities $\Uf$ with $f(x) = x_-$
can be associated with a partially \stock{}-augmented DP.
The method tracks the \stock{} throughout each episode and uses it during action selection, however it does not employ \stock{}-augmented states for the return distribution functions.
In other words,
their method adopts a hybrid greedy selection that we can write as $\sup_{\pi \in \stationary} T_{\pi}\eta$, but with $\eta : \states \rightarrow \crvs$ rather than $\eta : \states \times \cs \rightarrow \crvs$.

In terms of analysis, ours is distinct from~\citet{bauerle2014more}.
Instead, we use results and proofs from classic-DP theory~\citep{bertsekas1996neuro,szepesvari2022algorithms} as a roadmap, incorporate techniques from distributional policy evaluation~\citep{bellemare2017distributional} to cope with return distributions, and employ novel results required to cope, additionally, with \stock{} augmentation and statistical functionals of the return distribution.

\section{Applications}
\label{sec:applications}

\subsection{Generating Desired Returns}
\label{sec:generating-returns}

In many cases, we want to instruct agents to perform tasks in highly controllable environments, but not necessarily the tasks with a ``do something as much as possible'' nature that are a clear fit for RL.
For example, we may want to specify the task of collecting a given number of objects in a room, or obtaining a score equal to two in the game of Pong in the Atari Benchmark \citep{bellemare2013arcade}.
The standard RL framework can be unwieldy for this type of task, but this type of task can be easily modeled as a \stock{}-augmented problem.

If we were to model this an RL problem without \stock{} augmentation, we would likely have to use a non-Markov reward that tracks how many apples have been collected, give a reward of $1$ to the agent when the third apple is collected, and zero otherwise.
Moreover, we would have one reward function for each number of apples to be collected, which might require training one agent per reward function (which seems wasteful).

With \stock{} augmentation, on the other hand, this type of task can be tackled effectively.
We can frame it as a \stock{}-augmented \problem{} problem with an expected utility $\Uf$ and $f(x) = -|x|$, where the \stock{} is the number of apples collected so far by the agent.
Moreover, we can get \emph{a single \stock{}-augmented agent} to perform various instances of the same task---for example, collect one apple, or collect three apples---simply by changing the agent's initial \stock{}:
Without discounting and with a reward of $1$ for each apple, 
a \stock{} of $-3$ will cause an optimal \stock{}-augmented agent to collect $3$ apples, a \stock{} of $-2$ will cause the agent to collect $2$ apples, and so forth.

\subsection{Maximizing the Conditional Value-at-Risk of Returns}
\label{sec:cvar}

The problem of maximizing \emph{conditional value-at-risk}~\citep[CVaR;][]{rockafellar2000optimization}, also known as \emph{average value-at-risk} or \emph{expected shortfall}, has received attention both in the context of risk-sensitive RL~\citep{bauerle2011markov,chow2014algorithms,chow2015risk,bauerle2021minimizing,greenberg2022efficient} and in non-sequential decision-making~\citep{rockafellar2000optimization}.

It was for this problem that \stock{}-augmented methods were originally developed and studied \citep[see \cref{sec:ddp:existing-results} and][]{bauerle2011markov,bauerle2014more,bauerle2021minimizing,lim2022distributional,moghimi2025beyond}.
Other works have also proposed methods for optimizing the CVaR and other risk measures, in approaches that can be seen as alternatives to \stock{} augmentation \citep{chow2014algorithms,chow2015risk,tamar2015optimizing,greenberg2022efficient}.

The \emph{$\tau$-CVaR} of returns with distribution $\nu \in (\Delta(\reals), \wass)$ is defined as
\[
    \cvar(\nu, \tau) \doteq \frac{1}{\tau}\int_0^{\tau} \qf_\nu (t)\mathrm{d}t.
\]
We can see the $\tau$-CVaR as an ``expected return in the worst-case'', since it corresponds to the expected return of $X \sim \nu$ in the lower-tail of the return distribution (where the tail has mass $\tau$).

For any starting augmented state $(s_0, \xc_0)$, a history-based policy $\pi \in \historybased$ generates returns distributed according to $\eta^\pi(s_0, \xc_0)$, and we want to find a policy $\pi$ and a $\xc_0$ to maximize the $\tau$-CVaR of these returns:
\[
    \sup_{\pi \in \historybased, \xc_0 \in \cs} \cvar(\eta^\pi(s_0, \xc_0), \tau).
\]
It is easy to see that this problem does not admit an optimal stationary Markov policy on states alone, however \citet{bauerle2011markov} showed that we can solve it as follows (see~\cref{app:cvar} for the proof):
\begin{restatable}[Adapted from \citealp{bauerle2011markov}]{theorem}{cvartheorem}
\label{thm:cvar-c-star}
For every $\tau \in (0, 1)$ and $s_0 \in \states$,
\[
    \sup_{\pi \in \historybased, \xc_0 \in \cs} \cvar(\eta^\pi(s_0, \xc_0), \tau) = -\xc^*_0 + \frac{1}{\tau}\sup_{\pi \in \historybased}\E(\xc^*_0 + G^\pi(s_0, \xc^*_0))_-,
\]
where $\xc^*_0$ is the solution of 
\begin{equation}
    \max_{\xc_0} \left( -\xc_0 + \frac{1}{\tau}\sup_{\pi \in \historybased}\E(\xc_0 + G^\pi(s_0, \xc_0))_- \right).
    \label{eq:b0-selection}
\end{equation}
\end{restatable}

The main algorithmic difference between our work and that of \citet{bauerle2011markov} is how to obtain $\pi^*$.\footnote{The differences in analysis are discussed in \cref{sec:ddp:existing-results}.}
While we propose to use distributional DP with $F_K = \Uf$ and $f(x) = x_-$,
\citet{bauerle2011markov} used a modified classic value iteration, but required the iterates to satisfy specific conditions \citep[see $\mathbb{M}$, p.~45, ][]{bauerle2011markov}.
With distributional DP, on the other hand, it is possible to establish approximate guarantees for $\tau$-CVaR optimization, for both distributional value/policy iteration, with minimal conditions on the starting iterates (return distribution iterates must have uniformly bounded first moment).
This is what the following result shows, if we combine distributional DP with a grid search procedure to approximately solve the optimization in \Cref{eq:b0-selection}:
\begin{restatable}{theorem}{approximatecvartheorem}
\label{thm:cvar-approximate}
For every $\tau \in (0, 1)$, $s_0 \in \states$ and $\varepsilon > 0$, there exists a stationary policy $\pibar \in \stationary$ (obtainable through distributional DP) and a $\overline{\xc}^*_0$ (obtainable through grid search) such that
\[
    \sup_{\pi \in \historybased, \xc_0 \in \cs} \cvar(\eta^\pi(s_0, \xc_0), \tau) - \cvar(\eta^{\pibar}(s_0, \overline{\xc}^*_0), \tau) \leq 4\varepsilon.
\]
In particular, $\pibar$ satisfies (for $f(x) = x_-$)
\[
    \sup_{\pi \in \historybased}\Uf\eta^\pi - \Uf\eta^{\pibar} \leq \varepsilon,
\]
and
\begin{equation}
    \overline{\xc}^*_0 = \argmax_{\xc_0 \in \overline{\cs}} \left( -\xc_0 + \frac{1}{\tau}\E(\xc_0 + G^{\pibar}(s_0, \xc_0))_- \right),
    \label{eq:grid-search}
\end{equation}
where $\overline{\cs} \doteq \{ \cmin + i\varepsilon : i \in \naturals_0, \cmin + i\varepsilon \leq \cmax \}$ and $\cmin$ and $\cmax$ are chosen so that
\begin{align*}
	&\max_{\xc_0} \left( -\xc_0 + \frac{1}{\tau}\E(\xc_0 + G^{\pibar}(s_0, \xc_0))_- \right) \\
	&= \max_{\cmin \leq \xc_0 \leq \cmax} \left( -\xc_0 + \frac{1}{\tau}\E(\xc_0 + G^{\pibar}(s_0, \xc_0))_- \right).
\end{align*}
\end{restatable}

The key insight in \cref{thm:cvar-approximate} is that the objective functional being maximized over $\xc_0$ in \Cref{eq:b0-selection} is $1$-Lipschitz,
so we can approximate it through a grid search with an approximately optimal return distribution (\cref{eq:grid-search}).
A remaining limitation of the approach is how to choose $\cmin,\cmax$ in practice.
We know from \cref{thm:cvar-c-star,thm:cvar-approximate} that we can choose $\cmin$ small enough and $\cmax$ large enough to satisfy the requirement, but how large/small they need to be is left to a case-by-case basis.

\subsection{Maximizing the Optimistic Conditional Value-at-Risk of Returns}
\label{sec:ocvar}

The $\tau$-CVaR is the expectation of the return over the lower tail of the distribution (with tail mass $\tau$), and maximizing it is a risk-averse approach.
With $\tau = 0$, the $\tau$-CVaR is the risk-neutral expected return, and as $\tau$ decreases the amount of risk-aversion increases.

We can also consider the problem of maximizing the upper tail of the return distribution, which we call the \emph{optimistic $\tau$-CVaR}, defined for returns with distribution $\nu \in (\Delta(\reals), \wass)$ as
\[
    \ocvar(\nu, \tau) \doteq \frac{1}{\tau}\int_{1 - \tau}^{1} \qf_\nu (t)\mathrm{d}t.
\]

This application is interesting to analyze because it is similar to the optimism used by \citet{fawzi2022discovering} in AlphaTensor.
More generally, risk-seeking behavior can be useful for ``scientific discovery'' problems like discovering matrix multiplication algorithms, where it is more helpful to attain exceptional outcomes some of the time, even at the expense of performance in most cases, than to perform well on average.
This is because in this type of problem the RL agent is being used to generate solutions to a search-like problem where exceptional solutions are very valuable, but low-quality solutions are harmless, as they can simply be discarded.

We can show that analogues of \cref{thm:cvar-c-star,thm:cvar-approximate} hold for optimizing the optimistic $\tau$-CVaR.
\begin{restatable}{theorem}{ocvartheorem}
\label{thm:ocvar-c-star}
For every $\tau \in (0, 1)$ and $s_0 \in \states$,
\[
    \sup_{\pi \in \historybased, \xc_0 \in \cs} \ocvar(\eta^\pi(s_0, \xc_0), \tau) = -\xc^*_0 + \frac{1}{\tau}\sup_{\pi \in \historybased}\E(\xc^*_0 + G^\pi(s_0, \xc^*_0))_+,
\]
where $\xc^*_0$ is the solution of 
\[
    \min_{\xc_0} \left( -\xc_0 + \frac{1}{\tau}\sup_{\pi \in \historybased}\E(\xc_0 + G^\pi(s_0, \xc_0))_+ \right).
\]
\end{restatable}

The proof of \cref{thm:ocvar-c-star} is more subtle than the proof of its risk-averse counterpart.
In \cref{thm:cvar-c-star}, we can exploit the equivalence
\[
    \sup_{\pi \in \historybased, \xc_0 \in \cs} \cvar(\eta^\pi(s_0, \xc_0), \tau) = \sup_{\pi \in \historybased, \xc_0 \in \cs} \left( -\xc_0 + \frac{1}{\tau}\E(\xc_0 + G^\pi(s_0, \xc_0))_+ \right). 
\]
The similar step in the case of the optimistic $\tau$-CVaR gives
\[
    \sup_{\pi \in \historybased, \xc_0 \in \cs} \cvar(\eta^\pi(s_0, \xc_0), \tau) = \sup_{\pi \in \historybased}\inf_{\xc_0 \in \cs} \left( -\xc_0 + \frac{1}{\tau}\E(\xc_0 + G^\pi(s_0, \xc_0))_+ \right).
\]
Thanks to distributional DP, we can optimize $\Uf$ with $f(x) = x_+$ uniformly for all $(s_0, \xc_0)$, and we use this to swap the supremum and the infimum above, which gives \cref{thm:ocvar-c-star}.

The approximate version of \cref{thm:ocvar-approximate} then follows analogously to \cref{thm:cvar-approximate}.
\begin{restatable}{theorem}{approximateocvartheorem}
\label{thm:ocvar-approximate}
For every $\tau \in (0, 1)$, $s_0 \in \states$ and $\varepsilon > 0$, there exists a stationary policy $\pibar \in \stationary$ (obtainable through distributional DP) and a $\overline{\xc}^*_0$ (obtainable through grid search) such that
\[
    \sup_{\pi \in \historybased, \xc_0 \in \cs} \ocvar(\eta^\pi(s_0, \xc_0), \tau) - \ocvar(\eta^{\pibar}(s_0, \overline{\xc}^*_0), \tau) \leq 4\varepsilon.
\]
In particular, $\pibar$ satisfies (for $f(x) = x_+$)
\[
    \sup_{\pi \in \historybased}\Uf\eta^\pi - \Uf\eta^{\pibar} \leq \varepsilon,
\]
and
\[
    \overline{\xc}^*_0 = \argmin_{\xc_0 \in \overline{\cs}} \left( -\xc_0 + \frac{1}{\tau}\E(\xc_0 + G^{\pibar}(s_0, \xc_0))_+ \right),
\]
where $\overline{\cs} \doteq \{ \cmin + i\varepsilon : i \in \naturals_0, \cmin + i\varepsilon \leq \cmax \}$ and $\cmin$ and $\cmax$ are chosen so that
\begin{align*}
	&\min_{\xc_0} \left( -\xc_0 + \frac{1}{\tau}\E(\xc_0 + G^{\pibar}(s_0, \xc_0))_+ \right) \\
	&= \min_{\cmin \leq \xc_0 \leq \cmax} \left( -\xc_0 + \frac{1}{\tau}\E(\xc_0 + G^{\pibar}(s_0, \xc_0))_+ \right).
\end{align*}
\end{restatable}

\subsection{Homeostatic Regulation}
\label{sec:homeostatic-regulation}

Homeostatic regulation is a computational model for the behavior of natural agents \citep{keramati2011reinforcement}
whereby they aim to reduce \emph{drive} \citep{hull1943principles}, the mismatch between their current internal state and a stable state.
Drive reduction aims to explain empirical observations about the behavior of natural agents \citep{hull1943principles}---a simplistic instance being the hypothesis that an animal feeds to reduce its hunger.

We can formalize the homeostatic regulation problem considered by \citet{keramati2011reinforcement}
as:
\[
    \sup_{\pi \in \historybased}-\E\| \xc_0 + G^\pi(s_0, \xc_0) \|_p^q,
\]
where $p, q \geq 1$,
$\cs = \reals^m$,
$-\xc_0$ is the ``ideal'' setpoint for the agent's internal state,
and the agent's \stock{} $\xC_t$ represents its drive (the deviation from the desired state to be reduced).

``Minimizing drive in norm'' above corresponds to the expected utility $\Uf$ with $f(x) = -\| x \|_p^q$.
This choice of $f$ is positively homogeneous (since $f(\gamma x) = \gamma^{\frac{q}{p}} f(x)$), but Lipschitz only when $q = 1$, 
so by \cref{lem:expected-utility-conditions,thm:value-iteration,thm:policy-iteration} distributional DP can solve this variant of homeostatic regulation in the finite-horizon case (regardless of $q$) and in the infinite-horizon discounted case if $q = 1$ and if we consider the variant where the agent's drive increases over time due to the reverse-discounting, as $\xC_{t+1} = \gamma^{-1}(\xC_t + R_{t+1})$.

The formulation where $f$ is a norm presumes that there is an ideal setpoint (namely, $-\xc_0$), and that the agent wants to keep its \stock{} as close to that as possible, that is, the agent wants its drive (positive or negative) to be as close to zero as possible.
This is different from minimizing positive drive---intuitively, a sated agent would not actively drive itself back to the threshold of being hungry.

To accommodate for minimizing only positive drive, we can consider a homeostatic regulation problem with an expected utility, but a different choice of $f$:
\[
    f(x) = \sum_{i = 1}^m \alpha_i \cdot (x_i)_-,
\]
where $\alpha_1, \dots, \alpha_m \in \reals$ are fixed weights.
Once again, this choice of $f$ is positively homogeneous (since $f(\gamma x) = \gamma f(x)$) and Lipschitz (since $f(x) \leq \max_i |\alpha_i| \cdot \| x \|_1$),
so by \cref{lem:expected-utility-conditions,thm:value-iteration,thm:policy-iteration} distributional DP can also solve this variant of homeostatic regulation both in the finite-horizon case and in the infinite-horizon discounted case.

These two reductions are examples of how we can use the framework of \stock{}-augmented \problem{} to provide simple solution methods for a problem that has been otherwise complicated to solve with RL.
Previously, solving homeostatic regulation with RL methods required the design of an appropriate reward signal (as done by \citealp{keramati2011reinforcement}).
Considering that \citet{keramati2011reinforcement} aimed to reconcile the differences between the drive reduction model and the RL-based computational model proposed by \citet{schultz1997neural},
perhaps the framework of \stock{}-augmented \problem{} will help bring the two models closer together.

The reward signal designed by \citet{keramati2011reinforcement} to reduce homeostatic regulation to RL corresponds precisely to the reward signal that we have identified as the way to reduce \stock{}-augmented \problem{} to \stock{}-augmented RL (see \cref{thm:reward-design}).

\subsection{Constraint Satisfaction}
\label{sec:constraint-satisfaction}

In this application, we want an agent to generate returns that satisfy various constraints, with probability one if they are feasible.
Our proposal is to model constraint satisfaction as minimizing constraint violations in expectation, which is a variation of minimizing only positive drive discussed in \cref{sec:homeostatic-regulation} and generating exact returns from \cref{sec:generating-returns}.
Constraint satisfaction is related to satisficing problems~\citep{simon1956rational,goodrich2004satisficing}, though satisficing proposes to use constraint satisfaction as a means to find acceptable suboptimal policies when finding optimal policies is inviable.

If we want a policy with return above a threshold $g$, we can implement the constraint satisfaction as a \stock{}-augmented \problem{} problem with $\Uf$, $f(x) = x_-$ and set $\xc_0 = -g$.
This choice of $f$ satisfies \Cref{eq:expected-utility-conditions:indifferent-to-gamma} (the condition for $U_f$ to be indifferent to $\gamma$), so distributional DP can optimize $\Uf$.
Maximizing the expected utility will correspond to minimizing the expected violation: 
\[
    \E(\xc_0 + G^\pi(s_0, \xc_0))_- = -\E(g - G^\pi(s_0, -g))_+.
\]
For any $\pi$, we have $G^\pi(s_0, -g) \geq g$ with probability one iff $\E(g - G^\pi(s_0, -g))_+ = 0$.
So if the constraint can be satisfied, optimizing $\Uf$ will suffice.
If we want a policy with return below a threshold $g$, we optimize $\Uf$ with $f(x) = -(x_+)$ and set $\xc_0 = g$, and 
for any $\pi$, we have $G^\pi(s_0, -g) \leq g$ with probability one iff $\E(G^\pi(s_0, -g) - g)_+$ is zero.
For an equality constraint, we can use $f(x) = -|x|$ as in \cref{sec:generating-returns}.

Distributional DP can also optimize any weighted combination of the constraints above, with a different \stock{} and reward vector coordinate per constraint, since the weighted combination will also satisfy \Cref{eq:expected-utility-conditions:indifferent-to-gamma}.
For example, to generate a return in the interval $[g_1, g_2]$, assume the return is replicated, so that $G_1 = G_2$, set $\xc_0 = (-g_1, -g_2)$ and optimize $\Uf$ with
\[
    f(x) = (x_1)_- - (x_2)_+.
\]
Then for any $\pi$, we have $G^\pi(s_0, (-g_1, -g_2)) \in [g_1, g_2]$ with probability one iff
\[
    \E\left(G^\pi(s_0, (-g_1, -g_2))_1 - g_1\right)_- - \E\left(G^\pi(s_0, (-g_1, -g_2))_2 - g_2\right)_+ = 0.
\]

Finally, we can also trade off minimizing constraint violations and minimizing or maximizing expected return.
An example of this kind of problem is when we want an agent achieve a certain goal ``as fast as possible'' \citep[Section 3.2, ][]{sutton2018reinforcement}.
Traditionally, this kind of goal is normally implemented in episodic settings by terminating the episode when the goal is achieved, with a constant negative reward at each step, or in discounted settings with a reward of $1$ when the goal is achieved, and zero otherwise.
This is manageable when the goal is achieved instantaneously,\footnote{Admittedly neither a sparse reward nor a constant reward of $-1$ may be easy for deep RL agents to optimize in practical settings.}
but otherwise specifying a reward can be tricky.
\Problem{} with vector-valued rewards allows for an alternative formulation of this problem with $\Uf$ and
\[
    f(x) = -x_1 + \sum_{i = 2}^m \alpha_i \cdot (x_i)_-,
\]
where the first coordinate of the reward vector is always $-1$ (representing the time penalty), and the remaining $\alpha_i \cdot (x_i)_-$ regularize the agent's behavior to achieve the multiple goals.
It is easy to see that this choice of $f$ is Lipschitz and satisfies \Cref{eq:expected-utility-conditions:indifferent-to-gamma},
so by \cref{lem:expected-utility-conditions,thm:value-iteration,thm:policy-iteration} distributional DP can solve this problem both in the finite-horizon case and in the infinite-horizon discounted case.
We will explore this application in an empirical setting in \cref{sec:gridworld:vector-rewards}.

\subsection{Generalized Policy Evaluation}
\label{sec:gpe}

One interesting aspect of \stock{}-augmented \problem{} is that policy evaluation is not bound to any particular objective functional:
If we know the return distribution for a policy $\pi$, we can evaluate it under various different choices of $F_K$, which means the setting is amenable to Generalized Policy Evaluation~\citep[GPE;][]{barreto2020fast}.
In the standard RL setting, GPE is ``the computation of the value function of a policy $\pi$ on a set of tasks''~\citep{barreto2020fast}.
Its natural adaptation to our setting can be stated as the evaluation of a policy under multiple objective functionals $F_{K_1}, \ldots, F_{K_n}$, each corresponding to a different task.
This adaptation can be used without \stock{}, with the caveat that removing \stock{} augmentation limits the objectives that distributional DP can optimize (cf.~\cref{sec:ddp:existing-results,app:summary-of-guarantees}).

We can also adapt Generalized Policy Improvement~\citep[GPI;][]{barreto2020fast} in a similar way: Given policies $\pi_1, \ldots, \pi_n$ and an objective functional $F_K$, the following is an improved policy using GPI:
\[
    \pibar(s, \xc) \doteq \argmax_{\pi \in \{ \pi_1, \ldots, \pi_{n'}\}} (F_K\eta^\pi)(s, \xc).
\]
The individual policies $\pi_1, \ldots, \pi_n$ may have been obtained by optimizing different objective functionals $F_{K_1}, \ldots, F_{K_n}$, and they can be combined into a policy $\pibar$ for a new objective functional $F_K$.
Thanks to distributional policy improvement (\cref{lem:policy-improvement}), we know that $\pibar$ is, fact, at least as good for $F_K$ as any of the individual policies $\pi_1, \ldots, \pi_n$.

\subsection{Reward Design}
\label{sec:reward-design}

In deploying RL algorithms on real-world sequential decision-making problems, it is often required to explicitly design a reward signal to codify the intended outcomes.
As the reward hypothesis states \citep[Section 3.2, ][]{sutton2018reinforcement}: ``All of what we mean by goals and purposes can be well thought of as the maximization of the expected value of the cumulative sum of a received scalar signal (called reward).''
This hypothesis has been explored and disproved for some interpretations of what constitutes a ``goal'' \citep{pitis2019rethinking,abel2021reward,shakerinava2022utility,bowling2023settling}. 
However, even when the hypothesis holds, the reward signal is not necessarily the simplest tool for expressing goals and purposes.

Designing rewards is notoriously difficult.
For instance, \citet{knox2023reward} present a systematic examination of the perils of designing effective rewards for autonomous driving. They found that, among publicly available reward functions for autonomous driving, ``the most risk-averse reward function [\ldots] would approve driving by a policy that crashes 2000 times as often as our estimate of drunk 16–17 year old US drivers'' (p.~7). 
Earlier work by \citet{hadfield2017inverse} reveals the difficulty of hand-designing rewards, with common failures including unintentional positive reward cycles. 

We contend that, in some cases, the framework of \stock{}-augmented \problem{} eliminates the need for bespoke reward design.
To support this claim, we extend a reward-design result by \citet{bowling2023settling} to the \stock{}-augmented setting,
showing, once the objective functional has been chose, how to define an RL reward signal so that the RL objective is equivalent to the \stock{}-augmented \problem{} objective.
The result also shows that this reduction between objectives is only possible if the statistical functional is an expected utility and indifferent to $\gamma$.
\begin{restatable}{theorem}{rewarddesigntheorem}
\label{thm:reward-design}
A \stock{}-augmented \problem{} objective functional $\Uf$ can be reduced to an equivalent \stock{}-augmented reinforcement learning objective (expected return) with discount $\alpha \in (0, 1]$ with $\gamma < 1 \Rightarrow \alpha < 1$ and reward proportional to
\begin{equation}
    \Rtilde_{t+1} \doteq \alpha f(\xC_{t+1}) - f(\xC_t) + (1 - \alpha)f(0)
    \label{eq:rtilde-reward-design}
\end{equation}
if $f$ satisfies, for all $\xc \in \cs$,
\begin{equation}
    f(\gamma \xc) = \alpha f(\xc) + (1 - \alpha)f(0),
    \label{eq:reward-design-condition}
\end{equation}
and:
\begin{itemize}[label=-]
    \item in the finite-horizon case, 
    \begin{equation}
        \sup_{s, \xc, a \in \states \times \cs \times \actions}\E \left( \left. | \Rtilde_{t+1} | \, \right| S_t = s, \xC_t = \xc, A_t = a \right) < \infty;
        \label{eq:rtilde-bounded-first-moment}
    \end{equation}
    \item in the discounted case, $f$ is Lipschitz.
\end{itemize}

A \stock{}-augmented \problem{} objective that is not an expected utility or not indifferent to $\gamma$ cannot be reduced via reward design to a \stock{}-augmented reinforcement learning objective.
\end{restatable}

The reward construction used in \cref{thm:reward-design} may seem obvious in hindsight, but we believe that it can be much less evident if the corresponding $\Uf$ has not been identified, and that this may account for some of the challenges in designing rewards straight from imprecise ``goals and purposes''.
However, once $\Uf$ has been identified, the construction essentially automates away one step in the design of RL agents.
For example, the construction used in \cref{thm:reward-design} can be seen to be the same as the one used by \citet{keramati2011reinforcement} to reduce homeostatic regulation to an RL problem, and \cref{thm:reward-design} provides this reduction immediately.

\Cref{thm:reward-design} allows us to optimize certain \stock{}-augmented \problem{} objectives with classic DP.
In the discounted case, these are the same objectives we have shown that can be solved with distributional DP.
In the finite-horizon undiscounted case, there are two main differences.
First, distributional DP can optimize (arguably pathological) objectives where \cref{ass:bounded-first-moment-rewards} is satisfied, but not \Cref{eq:rtilde-bounded-first-moment}.\footnote{For example, consider optimizing $\Uf$ with $f(\xc) = \xc^2$ in the finite-horizon case with $\gamma = 1$ ($f$ need not be Lipschitz in this case), where rewards for each state-action are $\mathrm{Pareto}(2, 1)$ random variables multiplied by $\pm 1$.
These pseudo-rewards have bounded first moment, but unbounded second moment, so $\Rtilde_{t+1}$ may have unbounded first moment, and Bellman equations may be invalid.}
Second, and more importantly, distributional DP can optimize certain objective functionals that are not expected utilities, whereas classic DP, at least via reward design, cannot.

\subsection{Beyond Expected Utilities}
\label{sec:beyond-expected-utilities}

In all the applications we have presented so far, the objective functionals being optimized by distributional DP were expected utilities.
While expected utilities cover many common use cases of \stock{}-augmented \problem{}, it is worth considering which non-expected utilities distributional DP can optimize.
Without \stock{} augmentation, distributional DP cannot optimize non-expected utilities, even in the finite-horizon undiscounted case~\citep{marthe2024beyond}, which is the most permissive as far as conditions for optimizing $F_K$ go.
We also saw in \cref{thm:reward-design} that, at least through reward design, classic DP cannot optimize non-expected utilities, even with \stock{} augmentation.
What about distributional DP with \stock{} augmentation?

The answers differ depending on whether we consider the infinite-horizon discounted case, or the finite-horizon case.
In the infinite-horizon discounted case, the following theorem states that only Lipschitz expected utilities satisfy indifference to mixtures and Lipschitz continuity, which are required in our distributional DP guarantees.
\begin{restatable}{theorem}{thmlipschitzcharacterization}
\label{thm:lipschitz-characterization}
If $K : (\crvs, \wass) \rightarrow \reals$ is indifferent to mixtures and Lipschitz,
then $F_K$ is an expected utility, that is, there exists an $f: \cs \rightarrow \reals$ such that $K\nu = \E f(G)$ ($G \sim \nu$) and $f$ is Lipschitz.
\end{restatable}

\Cref{thm:lipschitz-characterization} does not necessarily rule out distributional DP optimizing non-expected utilities in the infinite-horizon discounted case, because it is still an open question whether Lipschitz continuity is necessary. 
However, it does rule out Lipschitz functionals that are not expected utilities, including, for example, the $\tau$-CVaR:
\begin{equation}
    K\nu = \frac{1}{\tau}\int_{0}^\tau \qf_\nu\mathrm{d}t.
    \label{eq:cvar-direct}
\end{equation}
This choice of $K$ is Lipschitz, but $F_K$ is not an expected utility.\footnote{$K$ violates the von-Neumann-Morgenstern \citep{von2007theory} axiom of independence.
See \cref{ax:independence} in \cref{app:reward-design} with $\nu$ uniform in $\{ 0 \}$, $\nu'$ uniform in $\{ -1, 2 \}$, $\overline{\nu}$ uniform in $\{ 2 \}$ and $\tau, p = \frac{1}{2}$.}
This may seem to contradict the claims in \cref{sec:cvar}, but it does not.
\Cref{thm:cvar} shows that distributional DP can optimize the $\tau$-CVaR by transforming the problem into the optimization of an expected utility, and specifying how to select $\xc_0$.
The objective that distributional DP cannot optimize is $F_K$ with $K$ set to be exactly the $\tau$-CVaR functional (as in \cref{eq:cvar-direct}).
To emphasize the difference between the two cases, compare which $K$ is used in the greedy policies of \cref{thm:value-iteration,thm:policy-iteration}.

As another example of non-expected utilities with Lipschitz $K$, consider minimizing the $1$-Wasserstein distance to a reference distribution $\overline{\nu}$ in the scalar case ($\cs = \reals$), that is, $K\nu = -\wass(\nu, \overline{\nu})$.
This $K$ is Lipschitz (by the triangle inequality), however $F_K$ is not an expected utility unless $\overline{\nu}$ is a Dirac.
By \cref{thm:lipschitz-characterization}, distributional DP cannot optimize this objective functional if $\overline{\nu}$ is not a Dirac.
We can verify that the $K$ is not indifferent to mixtures, for example, when $\overline{\nu}$ is the distribution of a Bernoulli-$\frac{1}{2}$ random variable (in this case, $K\delta_0 = K\delta_1$, so indifference to mixtures requires that $K\left(\frac{1}{2}\delta_0 + \frac{1}{2}\delta_0\right)$ equal $K\left(\frac{1}{2}\delta_0 + \frac{1}{2}\delta_1\right)$, which is not the case).
When $\overline{\nu} = \delta_{\xc}$ for some $\xc \in \reals$, it is easy to see that $K\nu = -\E|G - \xc|$ ($G \sim \nu$), and we have already established that $K$ is indifferent to $\gamma < 1$ iff $\xc = 0$.

Turning to the finite-horizon case, can we claim that distributional DP cannot optimize non-expected utilities?
A positive answer here would imply that distributional and classic DP are essentially equivalent in the finite-horizon \emph{undiscounted} case, with \stock{} augmentation as well as without.\footnote{Save for the extreme case where the pseudo-rewards in the \stock{}-augmented \problem{} have bounded first moment, but not the designed \stock{}-augmented RL rewards, as discussed in the context of \cref{thm:reward-design}.}

As the next result shows, it \emph{is} possible for distributional DP to optimize non-expected utilities in the finite-horizon case.
The choice of functional in \cref{prop:non-expected-utility} can be phrased as ``any negative return is (equally) unacceptable,'' and is known not to be an expected utility~\citep{juan2020neumann,bowling2023settling}.
\begin{restatable}{proposition}{propnonexpectedutility}
\label{prop:non-expected-utility}
The statistical functional $K : (\crvs, \wass) \rightarrow \reals$ satisfying, for $\nu \in (\crvs, \wass)$,
\[
    K\nu = \Ind( \nu([0, \infty)) = 1 )
\]
is indifferent to mixtures and $F_K$ is not an expected utility.
\end{restatable}

The choice of $K$ in \cref{prop:non-expected-utility} does not allow for a reduction to a \stock{}-augmented RL objective via reward design (cf.~\cref{thm:reward-design}), because it is not an expected utility. 
However, since it is indifferent to mixtures, distributional DP can optimize the corresponding $F_K$ in the finite-horizon undiscounted case.

\section{\texorpdfstring{D$\eta$N}{DIN}}
\label{sec:din}

To highlight the practical potential of distributional DP for solving \problem{} problems,
we adapted QR-DQN~\citep[DQN with quantile regression;][]{dabney2018distributional} to optimize expected utilities $\Uf$ and evaluated it empirically.
We call this new method \emph{Deep $\eta$-Networks}, or \emph{D$\eta$N} (pronounced \emph{din}).
In this section introduce D$\eta$N and describe how it incorporates the principles of distributional DP.
We present the empirical study in \cref{sec:gridworld-experiments,sec:atari}.

D$\eta$N uses a neural-network estimator for the \stock{}-augmented return distribution, similar to QR-DQN, with one difference: The \stock{} embedding.
In D$\eta$N, we input the \stock{} to a linear layer\footnote{While this simple design decision proved sufficient for our experiments, we believe that improved scalar embedding should be considered in the future \citep[for example,][]{springenberg2024offline}.} and then add the output of this linear layer to output to the of the agent's vision network.\footnote{In practice, the MDP state $s$ is converted to an image observation before being input to the vision network, and the conversion is domain-dependent.}
The architecture diagrams for DQN~\citep{mnih2015human}, QR-DQN~\citep{dabney2018distributional}) and D$\eta$N are given in \cref{fig:din-architecture}.
\begin{figure}[tb]
    \begin{center}
    \hfill
    \begin{subfigure}[b]{0.3\textwidth}
    \begin{tikzpicture}[
        into/.style={->,shorten <=1pt,>=stealth'},
        nohead/.style={-,shorten <=1pt,>=stealth'},
        mlp/.style={shape=rectangle,draw=black,fill=white,minimum width=3em, minimum height=1.5em},
        rightangle/.style={to path={|- (\tikztotarget)}},
    ]
        \node (output) {$Q_\theta(s, a)$ ($a \in \actions$) };
        \node[mlp,below=1.5em of output,rounded corners] (head) {{\scriptsize MLP}}
        edge[into](output);
        \node[mlp,below=1.5em of head,rounded corners] (relu) {{\scriptsize ReLU}}
        edge[into](head);
        \node[below=1em of relu,fill=white] (add) {};
        \node[mlp,below=1em of add,rounded corners] (s_linear) {{\scriptsize Linear}}
        edge[into](relu);
        \node[mlp,below=1.5em of s_linear,rounded corners] (s_conv) {{\scriptsize Vision}}
        edge[into](s_linear);
        \node[below=1.5em of s_conv] (s) {$s$}
        edge[into](s_conv);
    \end{tikzpicture}
    \end{subfigure}
    \hfill
    \begin{subfigure}[b]{0.3\textwidth}
    \begin{tikzpicture}[
        into/.style={->,shorten <=1pt,>=stealth'},
        nohead/.style={-,shorten <=1pt,>=stealth'},
        mlp/.style={shape=rectangle,draw=black,fill=white,minimum width=3em, minimum height=1.5em},
        rightangle/.style={to path={|- (\tikztotarget)}},
    ]
        \node (output) {$\xi_\theta(s, a)$ ($a \in \actions$) };
        \node[mlp,below=1.5em of output,rounded corners] (head) {{\scriptsize MLP}}
        edge[into](output);
        \node[mlp,below=1.5em of head,rounded corners] (relu) {{\scriptsize ReLU}}
        edge[into](head);
        \node[below=1em of relu,fill=white] (add) {};
        \node[mlp,below=1em of add,rounded corners] (s_linear) {{\scriptsize Linear}}
        edge[into](relu);
        \node[mlp,below=1.5em of s_linear,rounded corners] (s_conv) {{\scriptsize Vision}}
        edge[into](s_linear);
        \node[below=1.5em of s_conv] (s) {$s$}
        edge[into](s_conv);
    \end{tikzpicture}
    \end{subfigure}
    \hfill
    \begin{subfigure}[b]{0.3\textwidth}
    \begin{tikzpicture}[
        into/.style={->,shorten <=1pt,>=stealth'},
        nohead/.style={-,shorten <=1pt,>=stealth'},
        mlp/.style={shape=rectangle,draw=black,fill=white,minimum width=3em, minimum height=1.5em},
        rightangle/.style={to path={|- (\tikztotarget)}},
    ]
        \node (output) {$\xi_\theta(s, \xc, a)$ ($a \in \actions$) };
        \node[mlp,below=1.5em of output,rounded corners] (head) {{\scriptsize MLP}}
        edge[into](output);
        \node[mlp,below=1.5em of head,rounded corners] (relu) {{\scriptsize ReLU}}
        edge[into](head);
        \node[below=1em of relu,fill=white] (add) {\textcolor{red}{$+$}}
        edge[into](relu);
        \node[mlp,below=1em of add,rounded corners] (s_linear) {{\scriptsize Linear}}
        edge[nohead](add);
        \node[mlp,below=1.5em of s_linear,rounded corners] (s_conv) {{\scriptsize Vision}}
        edge[into](s_linear);
        \node[below=1.5em of s_conv] (s) {$s$}
        edge[into](s_conv);
        \node[mlp,right=1em of s_linear,rounded corners,draw=red] (b_linear) {\textcolor{red}{{\scriptsize Linear}}}
        edge[nohead,rightangle,draw=red](add);
        \node[] (b) at (intersection of s.east--s.west and b_linear.north--b_linear.south) {\textcolor{red}{$\xc$}}
        edge[into,draw=red](b_linear);
    \end{tikzpicture}
    \end{subfigure}
    \hfill
    \caption{Architecture diagrams for DQN (left), QR-DQN (center) and D$\eta$N (right). 
    In red, the elements introduced specifically for D$\eta$N.
    The QR-DQN and D$\eta$N networks output return distribution quantiles for each input ($s$ or $(s, \xc)$) and action.
    \label{fig:din-architecture}}
    \end{center}
\end{figure}
The output $\xi_\theta(s, \xc, a)$ of the network is a return distribution parameterized as quantiles (see \cref{app:din:implementation-details} for implementation details).

To explain the remaining differences between QR-DQN and D$\eta$N, it is useful to understand how QR-DQN is adapted from classic DP, and then see how distributional DP is adapted into D$\eta$N.
This adaptation is necessary because DP is designed for a \emph{planning setting} (where the transition and reward dynamics of the MDP are known), but planning methods are rarely tractable or feasible in practice (where state spaces can be very large and the dynamics can only be observed through interaction with the environment).
Practical settings are more closely modeled as prediction and control settings \citep{sutton2018reinforcement} with a function approximator learned through deep learning, that is, the typical setting for deep reinforcement learning.

Given a (state-) value function $V \in (\reals^\states, \|\cdot\|_{\infty})$, the corresponding \emph{action-value function} is defined as
\[
    Q(s, a) = (\classicBellman_{\pi_a}V)(s),
\]
where $\pi_a$ denotes the policy that selects action $a$ with probability one at all states.
It is convenient to denote this transformation with an operator, commonly known as the classic \emph{Bellman lookahead}~\cite[p.~30,][]{szepesvari2022algorithms}:
\[
    (A V)(s, a) \doteq (\classicBellman_{\pi_a}V)(s).
\]
We also let $M : (\reals^{\states \times \actions}, \|\cdot\|_{\infty}) \rightarrow (\reals^\states, \|\cdot\|_{\infty})$ be the \emph{max operator on action-value functions} defined as
\[
    (MQ)(s) \doteq \max_a Q(s, a) = \sup_{p \in \Delta(\actions)} \E Q(s, A).
    \tag{$A \sim p$}
\]
Each iterate $V_n$ in classic value iteration has a corresponding $Q_n \doteq A V_n$, and it holds that $V_{n+1} = M Q_n$.
Thus,
we can equivalently carry out value iteration on action-value functions, via the relation
\begin{equation}
    Q_{n+1} = A M Q_n.
    \label{eq:action-value-iterate}
\end{equation}

Q-learning~\citep{watkins1989learning,sutton2018reinforcement} aims to approximate value iteration through multiple asynchronous stochastic updates per transition.
Given a \emph{transition} $(s_t, a_t, r_{t+1}, s_{t+1})$, the Q-learning update is:
\begin{equation}
    Q_{\theta}(s_t, a_t) \leftarrow (1 - \alpha)Q_{\theta}(s_t, a_t) + \alpha \cdot \left(r_{t+1} + \gamma( M Q_{\theta})(s_{t+1}) \right),
    \label{eq:q-learning-update}
\end{equation}
where $Q_{\theta}$ is the action-value function estimator being learned and $\alpha$ is a learning rate.
Note how the term in parentheses resembles the right-hand side of \Cref{eq:action-value-iterate}.
Roughly speaking, it serves as an estimate of $A M Q_{\theta}$ on the given transition.\footnote{The precise relationship between the two quantities can be understood from the analysis of Q-learning~\citep{dayan1992q}.}

DQN~\citep{mnih2015human} implements the Q-learning update with a deep neural network estimator for $Q_{\theta}$, and in addition, an estimator $Q_{\overline{\theta}}$ with \emph{target parameters} $\overline{\theta}$ on the right-hand side of \Cref{eq:q-learning-update}.
The target parameters slowly track $\theta$, and the DQN value update only modifies $\theta$.
The updates to $\theta$ are performed through regression, similar to fitted Q-iteration~\citep{ernst2005tree} with a Huber loss, and with the \emph{prediction targets}
\[
    r_{t+1} + \gamma(M Q_{\overline{\theta}})(s_{t+1}),
\]
which, as before, are meant to serve as an estimate of $A M Q_{\overline{\theta}}$ on the given transition.

The implementation of D$\eta$N can be thought of as applying the adaptations above to distributional DP with an expected-utility objective $\Uf$.
This is a \stock{}-augmented setting, so note the use of the augmented state $(s, \xc) \in \states \times \cs$, in contrast to the use of the plain states $s \in \states$ for classic DP, Q-learning, DQN and QR-DQN.
The \emph{\stock{}-augmented distributional Bellman lookahead} operator is defined as
\[
    (A\eta)(s, \xc, a) \doteq (T_{\pi_a}\eta)(s, \xc), 
\]
where, as before, $\pi_a$ selects $a$ with probability one at all $(s, \xc) \in \states \times \cs$.
The distributional analogue of action-value functions are action-dependent return distribution functions.
From a return distribution $\eta$, the distributional Bellman lookahead gives the corresponding action-dependent return distribution function $\xi = A \eta$.

The analogue of the max operator $M$ for optimizing $\Uf$ must take $f$ into account, so we denote it by $M_f$ to highlight this dependence, and we define it so that:
\[
    \Uf(M_f\xi)(s, \xc) = \sup_{p \in \Delta(\actions)} \E f(\xc + G(s, \xc, A)).
    \tag{$A \sim p$, $G(s, \xc, a) \sim \xi(s, \xc, a)$}
\]
$M_f$ may not be unique because $\Uf$ may allow multiple policies to realize the supremum on the right-hand side, but any valid $M_f$ is acceptable.
Because the right-hand side above is linear in $\pi$, we can write $M_f$ via a simple maximization over actions:
\[
    \Uf(M_f\xi)(s, \xc) = \max_{a} \E f(\xc + G(s, \xc, a)).
    \tag{$G(s, \xc, a) \sim \xi(s, \xc, a)$}
\]

As in the classic case, we can carry out
distributional value iteration on action-dependent return distribution function iterates:
\[
    \xi_{n+1} = AM_f\xi_n.
\]

D$\eta$N adapts distributional value iteration similarly to how QR-DQN adapts classic value iteration.
QR-DQN replaces DQN's action-value function estimator with a return distribution estimator (see the middle diagram in \cref{fig:din-architecture}), and employs quantile regression to fit it, rather than ordinary scalar regression with a Huber loss.
The return distribution estimator used by D$\eta$N is $\xi_{\theta} : \states \times \cs \times \actions \rightarrow \crvs$
and the distributional prediction target can be written as
\begin{equation}
    \df\left(r_{t+1} + \gamma(M_f\xi_{\overline{\theta}})(s_{t+1}, \xc_{t+1})\right),
    \label{eq:din-quantile-regression-target}
\end{equation}
and QR-DQN is analogous, but without the \stock{} augmentation.
In analogy to DQN, the distributional prediction target in \cref{eq:din-quantile-regression-target} is meant to serve as an estimate of $AM_f\xi_{\overline{\theta}}$ on the observed data.

In QR-DQN, $f$ is the identity function and $\Uf$ is the standard RL objective, so
\[
    \E(M_f\xi_{\overline{\theta}})(s_{t+1}) = \max_{a} \E\left(G(s_{t+1}, a)\right).
    \tag{$G(s, a) \sim \xi_{\overline{\theta}}(s, a)$}
\]
This is an equation over action-values, and it naturally resembles the action choice used in the Q-learning update and DQN's prediction targets.
Similar to how the greedy action for Q-learning and DQN is a maximizing action,
D$\eta$N's greedy action at $(s_t, \xc_t)$ maximizes $\Uf$:
\begin{equation}
    \E f(\xc_t + G(s_t, \xc_t, a_t)) = \max_{a} \E f(\xc_t + G(s_t, \xc_t, a)).
    \label{eq:din-action}
\end{equation}
with $G(s, \xc, a) \sim \xi_{\overline{\theta}}(s, \xc, a)$.

In summary, D$\eta$N is similar to QR-DQN in many ways, with two notable differences: 
The neural network supports \stock{} augmentation (\cref{fig:din-architecture}), and the \stock{} and the utility factor into the action selection, both for the quantile regression targets (\cref{eq:din-quantile-regression-target}) and for the agent's interaction with the environment (\cref{eq:din-action}).

\section{Gridworld Experiments}
\label{sec:gridworld-experiments}

In this section we present experiments to illustrate how D$\eta$N solves different toy instances of \stock{}-augmented \problem{},
corresponding to some of the applications discussed in \cref{sec:applications}.
These experiments are also interesting because they reveal practical challenges of training \stock{}-augmented \problem{} agents.

The environments are $4 \times 4$ gridworlds~\citep{sutton2018reinforcement}.
The agent's actions are up, down, left, right, and no-op.
If the agent takes a no-op action or attempts to go outside the grid, it stays in the same cell.
The starting cell is always the top-left corner of the grid, which we denote by $s_0 = \sinit$, and the starting \stock{} $\xc_0$ is set per experiment.
For a transition $(s, \xc), a, r', (s', \xc')$,
if $s$ is terminal, then $\xc' = \xc$, $s' = s$ and $r' = 0$.
Otherwise, $\xc' = \gamma^{-1}(\xc + r')$ (as in \cref{eq:b-def}).
Some cells are terminating; if the agent enters a terminating cell, then $s'$ is terminal (and absorbing).
Some cells are rewarding: If $s$ is non-terminal and $s'$ is rewarding, then the agent receives $r'$ associated with $s'$.
The reward may be deterministic, or it may be $r' \cdot B$ where $B \sim \mathrm{Bernoulli}\left(\frac{1}{2}\right)$ (independently for each transition).
A cell may be both rewarding and terminal, in which case the agent receives the reward for the cell upon entering it, but not afterwards.
\Cref{fig:gridworld-example} gives an example gridworld with the notation we use.
\begin{figure}[tb]
    \begin{center}
    \begin{tikzpicture}
        \draw[step=1cm,black,thin] (0,0) grid (4,4);
        \draw[draw=black,fill=blue!25] (0, 3) rectangle (1,4);
        \draw[align=center] (0.5,3.5) node{$\sinit$};
        \draw[draw=black,fill=black!25] (3,0) rectangle (4,1);
        \draw[align=center] (3.5,0.5) node{T};
        \draw[draw=black,fill=black!25] (2,1) rectangle (3,2);
        \draw[text width=0.5cm, align=center] (2.5,1.5) node{$3B$ T};
        \draw[draw=black,fill=red!25] (0,0) rectangle (1,1);
        \draw[align=center] (0.5,0.5) node{$1$};
        \draw[draw=black,fill=yellow!25] (3,3) rectangle (4,4);
        \draw[align=center] (3.5,3.5) node{$-2B$};
    \end{tikzpicture}
    \caption{Example gridworld (with cells indexed as matrix entries). 
    The starting cell $\sinit$ is the upper-left corner cell $(1, 1)$. 
    The bottom-left corner (red, $(4, 1)$) has a deterministic reward of $1$.
    The upper-right corner (yellow, $(1, 4)$) has a stochastic reward $-2B$, where $B \sim \mathrm{Bernoulli}\left(\frac{1}{2}\right)$ (sampled independently each time step).
    The bottom-right corner (gray, $(4, 4)$) is terminal.
    The cell $(3, 3)$ (gray) is terminal and has a stochastic reward of $3B$.
    \label{fig:gridworld-example}}
    \end{center}
\end{figure}
At an augmented state $(s, \xc)$,
besides the \stock{} $\xc$, the
input to D$\eta$N's vision network (see \cref{fig:din-architecture}) is a one-channel $4 \times 4$ frame with $1$ in the cell corresponding to $s$ and zero otherwise.

During training, it was essential to randomize the starting $\xc_0$, by sampling values uniformly from a range (implementation details are given in \cref{app:gridworld:implementation-details}).
This was meant to introduce diversity in the training data and ensure that the agent could solve problems for a variety of $\xc_0$.

\subsection{Generating Desired Returns}
\label{sec:gridworld:generating-returns}

Our two first experiments illustrate how D$\eta$N with $\cs = \reals$ and $f(x) = -|x|$ can generate desired outcomes in a deterministic environment (see the application discussed in \cref{sec:generating-returns}).
In this setting the trained D$\eta$N agent displays different behaviors depending on $\xc_0$.

We first consider generating specific returns in the gridworld given in \cref{fig:abs-combining-rewards-gridworld}.
\begin{figure}[tb]
    \begin{center}
    \begin{tikzpicture}
        \draw[step=1cm,black,thin] (0,0) grid (4,4);
        \draw[draw=black,fill=blue!25] (0, 3) rectangle (1,4);
        \draw[align=center] (0.5,3.5) node{$\sinit$};
        \draw[draw=black,fill=black!25] (3,0) rectangle (4,1);
        \draw[align=center] (3.5,0.5) node{T};
        \draw[draw=black,fill=red!25] (0,0) rectangle (1,1);
        \draw[align=center] (0.5,0.5) node{$-1$};
        \draw[draw=black,fill=red!25] (3,3) rectangle (4,4);
        \draw[align=center] (3.5,3.5) node{$2$};
    \end{tikzpicture}
    \caption{Gridworld for the first experiment for generating returns.
    \label{fig:abs-combining-rewards-gridworld}}
    \end{center}
\end{figure}
Because this gridworld is deterministic, we can set $\xc_0$ to different values to generate different desired discounted returns, and the agent must do so by combining the rewards of $2$ on the top-right corner and the rewards of $-1$ on the bottom-left corner.

Because in practice DQN-like agents tend not to cope well with $\gamma = 1$, we set $\gamma = 0.997$ and assessed whether the agent can approximately generate the values of $\xc_0$ provided.
\Cref{tab:abs-combining-rewards-results} shows the agent's average return for different choices of $\xc_0$, with confidence interval bounds in parentheses.
In each independent run, we trained the agent and then measured its average discounted return (over $200$ episodes) for each of the values of $\xc_0$ considered.
We then computed $95\%$-confidence intervals based on the $30$ independent averages using bias-corrected and accelerated bootstrap~\citep{james2013introduction,virtanen2020scipy}.
Each row of \cref{tab:abs-combining-rewards-results} shows the ``desired'' return ($-\xc_0$), the average discounted return obtained by the agent ($\E G(s_0, \xc_0)$) and the ``error'' $\E |\xc_0 + G(s_0, \xc_0)|$, the negative of the objective.
\begin{table}[tb]
    \centering
    \begin{tabular}{ccc}
        \toprule
        Desired discounted return & Discounted return & Error \\
        $-\xc_0$ & $\E G(s_0, \xc_0)$ & $\E |\xc_0 + G(s_0, \xc_0)| $  \\
        \midrule
        $7.00$ & $6.95 \  (6.95, 6.95)$ & $0.05 \  (0.05, 0.05)$ \\
        $5.00$ & $4.98 \  (4.98, 4.98)$ & $0.02 \  (0.02, 0.02)$ \\
        $3.00$ & $3.00 \  (3.00, 3.00)$ & $0.00 \  (0.00, 0.00)$ \\
        $1.00$ & $1.01 \  (1.01, 1.01)$ & $0.01 \  (0.01, 0.01)$ \\
        $-2.00$ & $-1.85 \  (-1.99, -1.59)$ & $0.15 \  (0.01, 0.41)$ \\
        $-4.00$ & $-3.96 \  (-3.96, -3.96)$ & $0.04 \  (0.04, 0.04)$ \\
        $-6.00$ & $-5.92 \  (-5.92, -5.92)$ & $0.08 \  (0.08, 0.08)$ \\
        $-8.00$ & $-7.87 \  (-7.87, -7.87)$ & $0.13 \  (0.13, 0.13)$ \\
        \bottomrule
    \end{tabular}
    \caption{Evaluation results for D$\eta$N optimizing $\Uf$ with $f(x) = -|x|$ in the gridworld from \cref{fig:abs-combining-rewards-gridworld}, and $\gamma = 0.997$. 
    Entries are averages with bootstrap confidence intervals in the format  ``average (low, high)'' where low and high are the interval bounds.
    \label{tab:abs-combining-rewards-results}}
\end{table}
We can see that, as intended, the trained D$\eta$N agent can approximately produce the desired discounted returns.

The mismatch between $-\xc_0$ and average discounted returns is likely due to the function approximation and discounting, which makes the exact $\xc_0$ challenging to realize for arbitrary $\xc_0$.
However, the agent should generate returns equal to $-\xc_0$ when it corresponds to a realizable discounted return. 
To test this hypothesis, we carried out a follow-up evaluation where, for each trained agent, each choice of $\xc_0$, and each evaluation episode generated with discounted return $G(s_0, \xc_0)$, we ran that agent starting from $(s_0, \xc'_0)$ with $\xc'_0 = -G(s_0, \xc_0)$, and measured the discounted return $G(s_0, \xc'_0)$ obtained. 
The observed values for $|\xc'_0 + G(s_0, \xc'_0)|$ were less than $3.02 \cdot 10^{-2}$ uniformly for \emph{all runs} (across all independent runs, $\xc_0$ and episodes).
Thus D$\eta$N can closely reproduce realizable discounted returns, and the mismatches in \cref{tab:abs-combining-rewards-results} are likely related to $\gamma$ and function approximation.

This first experiment is an illustration of the ability of methods like D$\eta$N to control deterministic environments and generate desired outcomes, which is a desirable capability for artificial agents.
Besides combining different rewards, another means to control the returns is to use the discounting. Intuitively, in this case, instead of collecting a unit of reward as soon as possible, the agent may choose to ``wait'' for a few time steps until the discounted reward (from the starting state) achieves the desired value. To illustrate this point, in our second experiment we removed the negative reward from the gridworld in the first experiment, and set $\gamma = \frac{1}{2}$.
The gridworld diagram is given in \cref{fig:abs-using-discount-gridworld}.
\begin{figure}[tb]
    \begin{center}
    \begin{tikzpicture}
        \draw[step=1cm,black,thin] (0,0) grid (4,4);
        \draw[draw=black,fill=blue!25] (0, 3) rectangle (1,4);
        \draw[align=center] (0.5,3.5) node{$\sinit$};
        \draw[draw=black,fill=black!25] (3,0) rectangle (4,1);
        \draw[align=center] (3.5,0.5) node{T};
        \draw[draw=black,fill=red!25] (3,3) rectangle (4,4);
        \draw[align=center] (3.5,3.5) node{$2$};
    \end{tikzpicture}
    \caption{Gridworld for the second experiment.
    \label{fig:abs-using-discount-gridworld}}
    \end{center}
\end{figure}

The results are in \cref{tab:abs-using-discount-results}, and the agent successfully generates the desired discounted returns.
From an observer's point of view, the perceived behavior of the agent  is that it ``correctly times'' the rewarding transitions; in reality, the agent uses the \stock{} to decide whether or not to collect a reward at a particular augmented state.
\begin{table}[tb]
    \centering
    \begin{tabular}{ccc}
        \toprule
        Desired discounted return & Discounted return & Error \\
        $-\xc_0$ & $\E G(s_0, \xc_0)$ & $\E |\xc_0 + G(s_0, \xc_0)| $  \\
        \midrule
        $1.00$ & $1.00 \  (1.00, 1.00)$ & $0.00 \  (0.00, 0.00)$ \\
        $0.50$ & $0.50 \  (0.50, 0.50)$ & $0.00 \  (0.00, 0.00)$ \\
        $0.25$ & $0.25 \  (0.25, 0.25)$ & $0.00 \  (0.00, 0.00)$ \\
        $0.12$ & $0.12 \  (0.12, 0.12)$ & $0.00 \  (0.00, 0.00)$ \\
        $0.06$ & $0.06 \  (0.06, 0.06)$ & $0.00 \  (0.00, 0.00)$ \\
        \bottomrule
    \end{tabular}
    \caption{Evaluation results for D$\eta$N optimizing $\Uf$ with $f(x) = -|x|$ in the gridworld from \cref{fig:abs-using-discount-gridworld} and $\gamma = \frac{1}{2}$. Entries are averages with bootstrap confidence intervals in the format  ``average (low, high)'' where low and high are the interval bounds.
    \label{tab:abs-using-discount-results}}
\end{table}

\subsection{Maximizing the \texorpdfstring{$\tau$}{tau}-CVaR}
\label{sec:gridworld:risk-averse}

We can use D$\eta$N to optimize $\tau$-CVaR of the return, the risk-averse RL setup outlined in \cref{sec:cvar}.
The $1$-CVaR is risk-neutral (\stock{}-augmented RL), and as $\tau$ goes to zero optimizing the $\tau$-CVaR requires more risk aversion.
In this setting, D$\eta$N displays behaviors with different risk profiles in response to changing $\tau$.

The objective functional is $\Uf$ with $f(x) = x_-$, but we do not specify $\xc_0$ directly.
Instead, given a desired $\tau$, we compute $\xc^*_0$ according to \cref{thm:cvar-c-star} and start the agent in the augmented state $(s_0, \xc^*_0)$.
The gridworld for this experiment is given in \cref{fig:risk-averse-trap-door-gridworld}.
\begin{figure}[tb]
    \begin{center}
    \begin{tikzpicture}
        \draw[step=1cm,black,thin] (0,0) grid (4,4);
        \draw[draw=black,fill=blue!25] (0, 3) rectangle (1,4);
        \draw[align=center] (0.5,3.5) node{$\sinit$};
        \draw[draw=black,fill=black!25] (3,3) rectangle (4,4);
        \draw[align=center,text width=0.5cm] (3.5,3.5) node{$3$ T};
        \draw[draw=black,fill=yellow!25] (2, 3) rectangle (3,4);
        \draw[align=center] (2.5,3.5) node{$-2B$};
        \draw[draw=black,fill=yellow!25] (3, 2) rectangle (4,3);
        \draw[align=center] (3.5,2.5) node{$-2B$};
        \draw[draw=black,fill=black!25] (0, 0) rectangle (1,1);
        \draw[align=center,text width=0.5cm] (0.5,0.5) node{$1$ T};
    \end{tikzpicture}
    \caption{Gridworld for the first risk-averse RL experiment.
    \label{fig:risk-averse-trap-door-gridworld}}
    \end{center}
\end{figure}
It has a ``safe'' terminating cell in the bottom-left corner, and a ``high-risk'' terminating cell in the upper-right corner. This cell has high risk because it is surrounded by cells that give $-2$ reward with probability $\frac{1}{2}$ (and zero otherwise).
With $\gamma = 0.997$ the high-risk cell is better in expectation, so an optimal risk-neutral agent ($\tau = 1$) would go there.
However, an optimal risk-averse agent (with respect to the $\tau$-CVaR and for small enough $\tau$) will avoid the high-risk cell and go to the safe cell in the bottom-left corner.

D$\eta$N's performance is consistent with these behaviors, as we see in \cref{fig:risk-averse-trap-door-histograms}, which shows the histograms of the returns obtained by D$\eta$N over several runs.
\begin{figure}
    \centering
    \includegraphics[width=\textwidth]{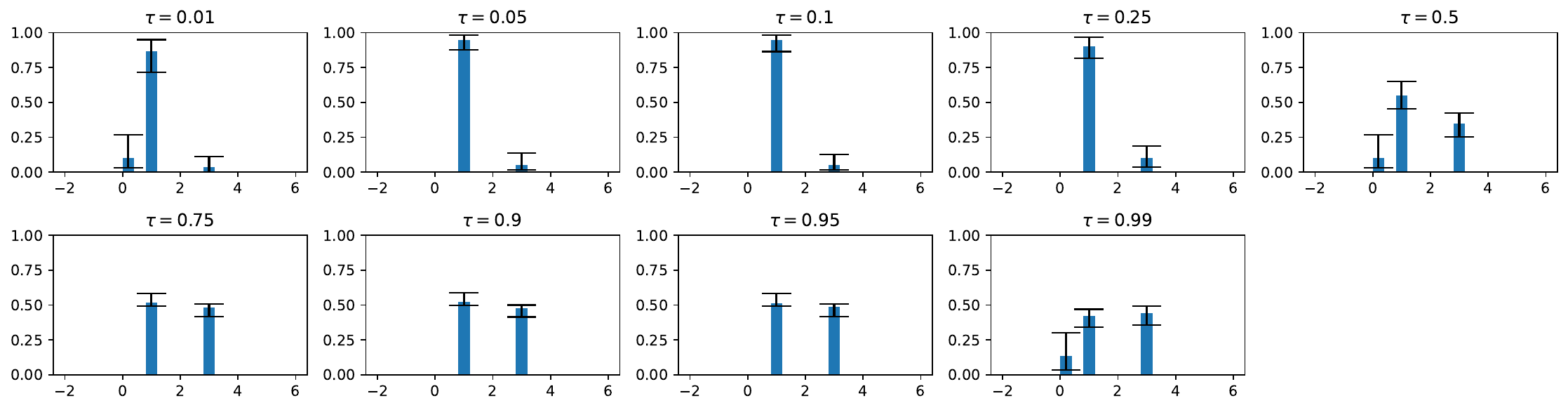}
    \caption{Discounted return histogram for different values of $\tau$, obtained by a trained D$\eta$N agent.
    Error bars correspond to bootstrap confidence intervals.
    \label{fig:risk-averse-trap-door-histograms}}
\end{figure}
As before, we trained the D$\eta$N agent in $30$ independent training runs.
After training the agent in each of the runs, we ran the agent with different values of $\tau$ for $200$ episodes. It is worth emphasizing that we run the \emph{same} trained agent with different values of $\tau$, as discussed in \cref{sec:generating-returns}.
We binned the observed returns and computed their frequencies for each independent run, and we report the average frequencies per bin with $95\%$ bootstrap confidence intervals.
For smaller $\tau$, the agent goes to the safe terminating cell.
As $\tau$ increases, the frequency of returns corresponding to the high-risk cell also increases.

D$\eta$N generated zero returns in some instances, which are suboptimal behaviors regardless of $\tau$.
The selection of $\xc^*_0$ uses grid search and approximate return estimates from $\xi_{\overline{\theta}}$, and estimation errors may cause $\E(\xc^*_0 + \xi_{\overline{\theta}}(s_0, \xc^*_0, a))_-$ to be zero for all actions, even for the down action.
When this is the case, D$\eta$N selects actions uniformly at random (because all actions are greedy).
The \stock{}, which starts often at a negative value, inflates due to the $\gamma^{-1}$ factor and becomes more negative.
Eventually it is so large in magnitude that the future discounted return can never exceed the \stock{}, and the result is degenerate behavior.

\subsection{Maximizing the Optimistic \texorpdfstring{$\tau$}{tau}-CVaR}
\label{sec:gridworld:risk-seeking}

Similar to how we can use D$\eta$N to produce risk-averse behavior, we can also use it to produce risk-seeking behavior, by following the outline in \cref{sec:ocvar}.
In this case we also observe D$\eta$N display behaviors with different risk profiles:
When the agent is risk-seeking, it tries to maximize its best-case expected performance, and as it becomes more risk neutral its performance resembles that of an RL agent maximizing value.

The objective functional is $\Uf$ with $f(x) = x_+$ and as before we do not specify $\xc_0$ directly.
Instead, given $\tau$, we compute $\xc^*_0$ according to \cref{thm:ocvar-c-star}, and run the agent from $(s_0, \xc^*_0)$.
The optimistic $1$-CVaR is risk-neutral, and as $\tau$ goes to zero the optimistic $\tau$-CVaR demands more risk-seeking behavior.
The gridworld for this experiment is given in \cref{fig:risk-seeking-gridworld}.
\begin{figure}[tb]
    \begin{center}
    \begin{tikzpicture}
        \draw[step=1cm,black,thin] (0,0) grid (4,4);
        \draw[draw=black,fill=blue!25] (0, 3) rectangle (1,4);
        \draw[align=center] (0.5,3.5) node{$\sinit$};
        \draw[draw=black,fill=red!25] (1, 3) rectangle (2,4);
        \draw[align=center] (1.5,3.5) node{$1$};
        \draw[draw=black,fill=red!25] (2, 3) rectangle (3,4);
        \draw[align=center] (2.5,3.5) node{$1$};
        \draw[draw=black,fill=black!25,text width=0.5cm] (3, 3) rectangle (4,4);
        \draw[align=center] (3.5,3.5) node{$1$ T};
        \draw[draw=black,fill=red!25] (0, 2) rectangle (1,3);
        \draw[align=center] (0.5,2.5) node{$\frac{3}{2}B$};
        \draw[draw=black,fill=red!25] (1, 2) rectangle (2,3);
        \draw[align=center] (1.5,2.5) node{$\frac{3}{2}B$};
        \draw[draw=black,fill=black!25,text width=0.5cm] (2, 2) rectangle (3,3);
        \draw[align=center] (2.5,2.5) node{$1$ T};
        \draw[draw=black,fill=red!25] (0, 1) rectangle (1,2);
        \draw[align=center] (0.5,1.5) node{$\frac{3}{2}B$};
        \draw[draw=black,fill=black!25,text width=0.5cm] (1, 1) rectangle (2,2);
        \draw[align=center] (1.5,1.5) node{$1$ T};
        \draw[draw=black,fill=black!25,text width=0.5cm] (0, 0) rectangle (1,1);
        \draw[align=center,text width=0.5cm] (0.5,0.5) node{$\frac{3}{2}B$ T};
    \end{tikzpicture}
    \caption{Gridworld for the risk-seeking RL experiment. The only allowed actions are down and right.
    \label{fig:risk-seeking-gridworld}}
    \end{center}
\end{figure}
The only allowed actions are down and right, and $\gamma = 0.997$. In this environment, the higher the risk, the higher the best-case return, but the lower the expected return.
A risk-neutral agent will go right twice and then either right or down, terminating with a discounted return of $1 + \gamma + \gamma^2$.
These are the low-risk paths.
In any given cell and whatever the \stock{}, moving to a cell with Bernoulli rewards increases the risk relative to choosing a cell with deterministic reward.
Going down three times is the path with highest risk, 
with expected discounted return $\frac{3}{4}(1 + \gamma + \gamma^2)$, but twice that amount with probability $\frac{1}{8}$ (the best case).

D$\eta$N's performance is consistent with the risk profile given by $\tau$, as we see in \cref{fig:risk-seeking-histograms}, which shows the histograms of the returns obtained by D$\eta$N over several runs.
\begin{figure}
    \centering
    \includegraphics[width=\textwidth]{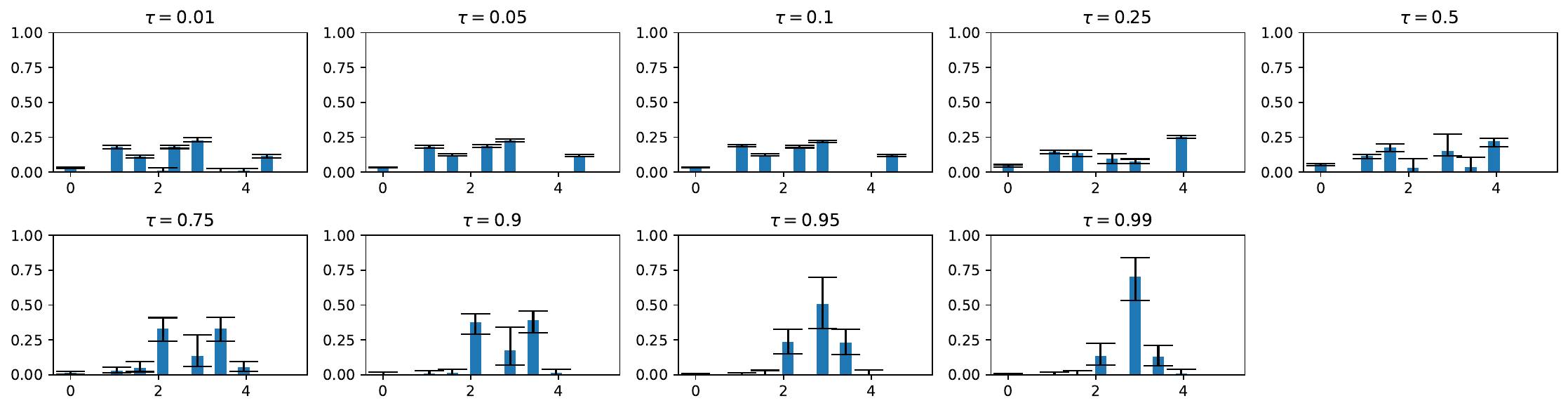}
    \caption{Discounted return histogram for different values of $\tau$, obtained by a trained D$\eta$N agent.
    Error bars correspond to bootstrap confidence intervals.
    \label{fig:risk-seeking-histograms}}
\end{figure}
We trained the D$\eta$N agent and computed histograms in the same way as in \cref{fig:risk-averse-trap-door-histograms}.

For $\tau \leq 0.1$ we see that the agent is maximally risk-seeking, as the support of the distribution includes the maximum possible return (approximately $4.5$) with probability around $\frac{1}{8}$.
As $\tau$ increases, the agent becomes less risk-seeking, and eventually ($\tau = 0.25$) the agent stops going for the riskiest path and visits cells with deterministic rewards more often.
At $\tau \approx 1$ the agent is nearly risk-neutral, with a mean discounted return of $2.6 \pm 0.485$.
The optimal risk-neutral expected discounted return is approximately $2.99$, and we believe the mismatch is due to approximation errors on the choice of the starting $\xc^*_0$.

To highlight the agent's ability to adapt to different stochastic outcomes, notice how the frequency of zero returns is quite low, even for the highly risk-seeking behavior ($\tau = 0.01$).
This may seem counter-intuitive if we consider that the highest-risk path has the same probability of a best discounted return ($4.48$ with probability $\frac{1}{8}$) as of a worst discounted return (zero).
Yet D$\eta$N with $\tau = 0.01$ observes a discounted return of $4.48$ with probability around $\frac{1}{8}$, and worst-case returns with probability around $0.03 \pm 0.02$.
This happens because D$\eta$N adapts its behaviors to the observed returns, through \stock{} augmentation.
If we look back at \cref{fig:risk-seeking-gridworld}, we can see that there is always a path such that, if the agent observes a zero reward at one of the non-terminal cells with Bernoulli rewards, it can go right and avoid a return of zero.
For example, for a low enough $\tau$, the agent's starting \stock{} will be $\xc^*_0 \leq -4$.
If the agent goes down on its first action and observes a reward of zero, it is no longer able to generate a discounted return above $4$.
Because $f(x) = x_+$, all actions will have expected utility zero (modulo estimation errors), and because D$\eta$N breaks ties by uniform sampling, the agent will follow a uniformly random policy.
So the probability of observing a zero discounted return is $\Prob(R_1 = 0, A_1 = \mathrm{down}, R_2 = 0, A_2 = \mathrm{down}, R_3 = 0)$. 
Since there are only two actions, this probability is $\left(\frac{1}{2}\right)^5 = 0.03125$, which is consistent with our data.

\subsection{Trading Off Minimizing Constraint Violation and Maximizing Expected Return}
\label{sec:gridworld:vector-rewards}

In this section, we consider the application outlined at the end of \cref{sec:constraint-satisfaction}: To obtain a certain amount of reward in as few steps as possible.
This application requires D$\eta$N to optimize an objective functional with vector-valued rewards.

In this setting, we have $\cs = \reals^2$.
The first coordinate of the reward is always $-1$, and corresponds to the ``time-to-termination'' penalty to be minimized.
The values observed in the second coordinate of the reward vector are given in \cref{fig:regularized-mdp-gridworld}.
\begin{figure}[tb]
    \begin{center}
    \begin{tikzpicture}
        \draw[step=1cm,black,thin] (0,0) grid (4,4);
        \draw[draw=black,fill=blue!25] (0, 3) rectangle (1,4);
        \draw[align=center] (0.5,3.5) node{$\sinit$};
        \draw[draw=black,fill=red!25] (0, 0) rectangle (1,1);
        \draw[align=center] (0.5,0.5) node{$1$};
        \draw[draw=black,fill=red!25] (1, 3) rectangle (2,4);
        \draw[align=center] (1.5,3.5) node{$-2$};
        \draw[draw=black,fill=red!25] (2, 3) rectangle (3,4);
        \draw[align=center] (2.5,3.5) node{$-2$};
        \draw[draw=black,fill=black!25] (3, 3) rectangle (4,4);
        \draw[align=center,text width=0.5cm] (3.5,3.5) node{$-2$ T};
        \draw[draw=black,fill=red!25] (2, 2) rectangle (3,3);
        \draw[align=center] (2.5,2.5) node{$-2$};
        \draw[draw=black,fill=red!25] (3, 2) rectangle (4,3);
        \draw[align=center] (3.5,2.5) node{$-2$};
        \draw[draw=black,fill=black!25] (3, 1) rectangle (4,2);
        \draw[align=center] (3.5,1.5) node{T};
    \end{tikzpicture}
    \caption{Gridworld for the experiment with trading off minimizing constraint violation and maximizing expected return.
    \label{fig:regularized-mdp-gridworld}}
    \end{center}
\end{figure}
The objective functional is $\Uf$ with $f(x) = -x_1 + \alpha \cdot (x_2)_-$.
We set $\alpha = 50$ to encourage prioritizing the term on the second coordinate of the reward vector,
so the semantics of the objective functional is to get to termination as fast as possible, keeping $G(s, \xc)_2 \geq -(\xc_0)_2$, but allowing for small violations to be traded off for faster termination.
For this experiment, we estimate the marginal distributions (per coordinate) of the vector-valued returns.
This simplifies the prediction in D$\eta$N, and is sufficient for the expected utility being optimized.\footnote{When $f(x)$ does not decouple as $\sum_i f_i(x_i)$ for some choice of $f_i$ (for example, $f(x) = -\|x\|_2$), the distribution of the quantile vectors is needed. For those cases, one may consider building on results for multivariate distributional RL~\citep{zhang2021distributional,wiltzer2024foundations}.}

An optimal policy with respect to $\Uf$ will display different behaviors depending on the choice of $(\xc_0)_2$.
If $-(\xc_0)_2 \leq -2(\gamma^2 + \gamma^1 + \gamma^0) \approx -5.98$, the policy will go straight from $\sinit$ to terminate at the top-right corner.
This is the shortest possible path to termination, but it is ``costly'' in terms of the cell rewards.
With $-2(\gamma^2 + \gamma^1 + \gamma^0) < -(\xc_0)_2 \leq 0$, the policy goes to the ``lower'' terminating cell ($(3, 4)$) in $5$ steps and with $G(s_0, \xc_0)_2 = 0$.
For $-(\xc_0)_2 > 0$, the policy must stay at the cell in the lower-left corner for multiple steps before going to the ``lower'' terminating cell ($(3, 4)$).
The number of steps it stays will depend on $\alpha$ and $-(\xc_0)_2$: As $\alpha \rightarrow \infty$ the policy will stay longer to make $G(s_0, \xc_0)_2$ closer to $-(\xc_0)_2$ (either larger or slightly smaller).
For example, it would take the optimal policy at most $8$ steps to reach termination with $-(\xc_0)_2 = 1$, $9$ steps with $-(\xc_0)_2 = 2$ and $10$ steps with $-(\xc_0)_2 = 3$.

The results for D$\eta$N are in \cref{tab:regularized-mdp-results}.
\begin{table}[tb]
    \centering
    \begin{tabular}{cccc}
        \toprule
        Lower-bound & Discounted Return & Penalty term & Episode duration \\
        $-(\xc_0)_2$ & $\E G(s_0, \xc_0)_2$ & $\E \left((\xc_0)_2 + G(s_0, \xc_0)_2\right)_- $  \\
        \midrule
        $3.00$ & $3.62 \  (3.39, 3.88)$ & $-0.05 \  (-0.18, -0.01)$ & $10.83 \  (10.23, 11.70)$ \\
        $2.00$ & $2.47 \  (2.14, 2.77)$ & $-0.14 \  (-0.41, -0.04)$ & $11.00 \  (10.00, 12.20)$ \\
        $1.00$ & $1.41 \  (1.08, 1.77)$ & $-0.20 \  (-0.37, -0.10)$ & $11.57 \  (10.20, 12.97)$ \\
        \midrule
        $0.00$ & $0.20 \  (0.07, 0.55)$ & $0.00 \  (0.00, 0.00)$ & $5.87 \  (5.37, 7.20)$ \\
        $-1.00$ & $0.06 \  (0.00, 0.39)$ & $0.00 \  (0.00, 0.00)$ & $5.37 \  (5.00, 6.83)$ \\
        $-2.00$ & $0.03 \  (0.00, 0.16)$ & $0.00 \  (0.00, 0.00)$ & $5.37 \  (5.00, 6.83)$ \\
        \midrule
        $-6.00$ & $-0.40 \  (-1.20, 0.00)$ & $0.00 \  (0.00, 0.00)$ & $4.93 \  (4.67, 5.00)$ \\
        $-7.00$ & $-4.79 \  (-5.58, -3.79)$ & $0.00 \  (0.00, 0.00)$ & $3.40 \  (3.13, 3.73)$ \\
        \bottomrule
    \end{tabular}
    \caption{Performance of D$\eta$N trading off minimizing constraint violation and maximizing expected return.
    The weight of the second term is $\alpha = 50$.
    Entries are averages with bootstrap confidence intervals in the format ``average (low, high)'' where low and high are the interval bounds.
    \label{tab:regularized-mdp-results}}
\end{table}
D$\eta$N did not produce optimal behaviors, but aligned with them.
In the first three settings (upper rows of the table), visiting the bottom-left corner was required by $\Uf$.
The agent did that (albeit overstaying) and then went to the lower terminating cell. 
In the second three settings (middle rows of the table), visiting the bottom-left corner was not required by $\Uf$; the agent went to the lower terminating cell.
In the last two settings (bottom rows of the table), $\Uf$ allowed the agent to suffer the $-2$ rewards on the path to the upper terminating cell, in exchange for a shorter time to termination.
An optimal agent would go in a straight line to the right and terminate in three steps, but D$\eta$N behaved suboptimally most of the time.
For $\xc_0 = 7$ (last row), we see that the agent often took the path to the upper terminating cell, however, for $\xc_0 = 6$ (second to last line) the agent rarely did so, often going for the lower terminating cell. 

Why did D$\eta$N overshoot the second coordinate of the discounted return on the first three settings, and why did it rarely go for the upper terminating cell when $\xc_0 = 6$?
We hypothesize that the cause was inaccuracy in the return distribution estimates.
A small underestimation of $\E \left((\xc_t)_2 + G(s_t, \xc_t)_2\right)_-$ will be amplified by $\alpha = 50$ and may cause the agent to become ``conservative'' in optimizing for this term of the objective, relative to term on the first coordinate of the discounted return.
To test this hypothesis, 
we ran a second version of our experiment with $\alpha = 500$.
The choice of $\alpha \in \{ 50, 500 \}$ should have little impact on an optimal agent's behavior with the values of $\xc_0$ we considered, however, larger $\alpha$ should make an agent with imperfect return estimates seem more conservative.
The results are in \cref{tab:regularized-mdp-results-larger-alpha}.
\begin{table}[tb]
    \centering
    \begin{tabular}{cccc}
        \toprule
        Lower-bound & Discounted Return & Penalty term & Episode duration \\
        $-(\xc_0)_2$ & $\E G(s_0, \xc_0)_2$ & $\E \left((\xc_0)_2 + G(s_0, \xc_0)_2\right)_- $  \\
        \midrule
        $3.00$ & $5.83 \  (5.09, 7.06)$ & $-0.00 \  (-0.00, 0.00)$ & $12.97 \  (12.17, 13.83)$ \\
        $2.00$ & $4.75 \  (3.89, 5.92)$ & $-0.04 \  (-0.20, 0.00)$ & $12.47 \  (11.43, 13.53)$ \\
        $1.00$ & $3.38 \  (2.73, 4.40)$ & $-0.00 \  (-0.01, 0.00)$ & $11.83 \  (10.73, 13.03)$ \\
        $0.00$ & $1.73 \  (1.24, 2.44)$ & $0.00 \  (0.00, 0.00)$ & $12.07 \  (10.77, 13.30)$ \\
        \midrule
        $-1.00$ & $0.36 \  (0.13, 0.84)$ & $0.00 \  (0.00, 0.00)$ & $6.80 \  (5.80, 8.47)$ \\
        $-2.00$ & $0.19 \  (-0.07, 0.63)$ & $0.00 \  (0.00, 0.00)$ & $6.77 \  (5.67, 8.47)$ \\
        $-6.00$ & $-0.27 \  (-1.14, -0.01)$ & $0.00 \  (0.00, 0.00)$ & $6.50 \  (5.43, 8.30)$ \\
        $-7.00$ & $-0.74 \  (-1.74, -0.07)$ & $0.00 \  (0.00, 0.00)$ & $5.97 \  (5.03, 7.67)$ \\
        \bottomrule
    \end{tabular}
    \caption{Performance of D$\eta$N trading off minimizing constraint violation and maximizing expected return.
    The weight of the second term is $\alpha = 500$.
    Entries are averages with bootstrap confidence intervals in the format  ``average (low, high)'' where low and high are the interval bounds.
    \label{tab:regularized-mdp-results-larger-alpha}}
\end{table}
Consistent with our hypothesis, we observe that D$\eta$N with $\alpha = 500$ appears more conservative, with longer episodes than with $\alpha = 50$, especially for $\xc_0 = 0$ and $\xc_0 = 7$.
For $\xc_0 = 0$, the agent did not take the zero-reward path to the lower terminating cell, but first visited the rewarding cell in the bottom-left corner, and for $\xc_0 = 7$ the agent did not go to the upper terminating cell.

\section{Atari Experiment}
\label{sec:atari}

Atari 2600~\citep{bellemare2013arcade} is a popular RL benchmark where several deep RL agents have been evaluated, including DQN~\citep{mnih2015human} and QR-DQN~\citep{dabney2018distributional}.
It provides us with a more challenging setting for deep RL agents than gridworld instances, since agents must overcome multiple learning challenges---to name a few: perception, exploration and control over longer timescales.

Atari 2600 is very much an RL benchmark, with games framed as RL problems in which the goal is to maximize the score. However, we can use the game of Pong to create an interesting setting for generating returns---an Atari analogue of the gridworld experiments in \cref{sec:gridworld:generating-returns}.
In Pong, the agent plays against an opponent controlled by the environment.
The goal of the game is for each player to get the ball to cross the edge of the opponent's side of the screen.
Each time this happens, the player gets a point.
Each player controls a paddle that can be used for hitting back the ball, preventing the opponent from scoring a point and sending the ball toward the opponent in a straight trajectory.

In a typical RL setting, we train agents to maximize the score (the difference between the player's and the opponent's scores), but in this section we are interested in using D$\eta$N to achieve different scores, which entails both scoring against the opponent, and being scored upon.
We trained D$\eta$N and evaluated the trained agent with different values of $\xc_0$, corresponding to different desired discounted returns, $\gamma = 0.997$, and reduced episode duration from thirty minutes to twenty-five seconds (implementation details are given in \cref{app:atari:implementation-details}).
This dramatic reduction is related to the interaction between $\gamma$ and the objective functional.
The goal is to control the distribution of the discounted return from the start of the episode.
A reward at time step $t + 1$ offsets this discounted return by $\gamma^t R_{t+1}$.
The rewards in Pong are $\pm 1$ and the agent acts at $15\mathrm{Hz}$, so after $25\mathrm{s}$ an observed reward only offsets the discounted return by approximately $\pm 0.32$.
As the episode advances, the effect of the agent's actions on the value of the objective decreases, and at a minute this effect has reduced to $\pm 0.07$.
The agent's behavior after that is unlikely to make any meaningful difference to the return and collected data may be less useful for training.
For these experiments, we have sidestepped the issue by reducing the episode duration, but the interaction between the timescale and $\gamma$ for \stock{}-augmented \problem{} is an important practical consideration that deserves a systematic study in future work.

\Cref{tab:atari-results} shows the performance of D$\eta$N.
Similar to the setting in \cref{tab:abs-combining-rewards-results}, we trained the agent and, for evaluation, conditioned its policy on different values of $\xc_0$ corresponding to the negative of the desired discounted return.
We measured the agent's average discounted return ($\E G(s_0, \xc_0)$) and the ``error'' $\E |\xc_0 + G(s_0, \xc_0)|$.
The confidence intervals correspond to $95\%$-confidence bootstrap intervals over $12$ independent repetitions of training and evaluation (differently from the $30$ independent runs in the gridworld setting).
\begin{table}[tb]
    \centering
    \begin{tabular}{ccc}
        \toprule
        Desired discounted return & Discounted return & Error \\
        $-\xc_0$ & $\E G(s_0, \xc_0)$ & $\E |\xc_0 + G(s_0, \xc_0)| $  \\
        \midrule
        $4.00$ & $2.26 \  (2.22, 2.28)$ & $1.74 \  (1.72, 1.78)$ \\
        $2.00$ & $1.90 \  (1.88, 1.92)$ & $0.15 \  (0.13, 0.18)$ \\
        $1.00$ & $0.88 \  (0.82, 0.95)$ & $0.23 \  (0.21, 0.27)$ \\
        $0.00$ & $-0.23 \  (-0.33, -0.15)$ & $0.29 \  (0.22, 0.37)$ \\
        $-1.00$ & $-1.03 \  (-1.09, -0.95)$ & $0.19 \  (0.16, 0.21)$ \\
        $-2.00$ & $-2.06 \  (-2.11, -1.96)$ & $0.18 \  (0.16, 0.22)$ \\
        $-4.00$ & $-3.97 \  (-4.01, -3.94)$ & $0.14 \  (0.11, 0.16)$ \\
        \bottomrule
    \end{tabular}
    \caption{Evaluation results generating discounted returns with D$\eta$N in Pong and $\gamma = 0.997$.
    Entries are averages with bootstrap confidence intervals in the format  ``average (low, high)'' where low and high are the interval bounds.
    \label{tab:atari-results}}
\end{table}
D$\eta$N approximately and reliably generated the desired discounted returns for various choices of $\xc_0$, with the exception of discounted returns to approximate $-\xc_0 = 4$ (first row).
We believe that the agent's training regime explains the successes, as well as the failure for $-\xc_0 = 4$.

We used D$\eta$N's policy for data collection during training, which required us to select $\xc_0$ during training.
At the beginning of each episode, we sampled a value for $\xc_0$ uniformly at random from $[-9, 9)$.
This was the strategy used in the gridworld experiments (albeit with a different interval) and it was meant to increase data diversity.
Because the episodes in Atari were much longer than in the gridworld experiment ($375$ versus $16$ steps), this strategy likely yielded little diversity in the \stock{}s observed later in the episode.
Diversity is important because we need to train the \stock{}-augmented agent to optimize the objective for a variety of augmented states.
Similar to how certain RL problems may pose exploration challenges in the state space $\states$, \stock{}-augmented problems may suffer from exploration challenges in the augmented-state space ($\states \times \cs$).

Fortunately, we can reintroduce diversity across \stock{}s after generating data, based on the following observation:
When the state dynamics are independent of the \stock{}, from a single transition $(S_t, \xC_t), A_t, R_{t+1}, (S_{t+1}, \xC_{t+1})$, it is possible to generate counterfactual transitions with the correct distribution for the whole spectrum of \stock{}s $\xc \in \cs$, that is, the following transitions:
\[
    \left\{ (S_t, \xc), A_t, R_{t+1}, (S_{t+1}, \gamma^{-1}(\xc + R_{t+1})) : \xc \in \cs \right\}.
\]
We refer to this change on $\xC_t$ and $\xC_{t+1}$ as \emph{\stock{} editing}.
D$\eta$N updates parameters using a minibatch of trajectories with subsequent transitions.
In this setting, before performing each update, we edited the \stock{}s in the minibatch as follows:
We sampled a value of $\xC'_0$ uniformly at random from $[-9, 9)$ for the first step of each trajectory, and edited the whole trajectory to create new transitions $(S_{t+k}, \xC'_k), A_{t+k}, R_{t+k+1}, (S_{t+k+1}, \xC'_{k+1})$
with, for $k \geq 0$,
\[
    \xC'_{k+1} = \gamma^{-k}\left(\xC'_0 + \sum_{i=0}^k \gamma^i R_{t+i+1}\right).
\]
\Stock{} editing was essential for our results, and we were unable to reproduce the outcomes in \cref{tab:atari-results} without it.

We believe that the failure for $-\xc_0 = 4$ happened because there was not enough data for learning to generate discounted returns of approximately $4$.
As $-\xc_0$ increases, the behaviors generated for the diverse \stock{}s through \stock{} editing are likely not as useful for solving the problem at $\xc_0$.
In other words, we conjecture that the data was diverse but imbalanced, and we pose this issue of data balance as a question for future work.

\section{Conclusion}
\label{sec:conclusion}

While standard RL has been successfully employed to solve various practical problems, its formulation as maximizing expected return limits its use in the design of intelligent agents.
The problem of \problem{} aims to address this limitation by posing the optimization of a statistical functional of the return distribution.
While this is a more general problem, the additional flexibility cannot be exploited by DP, as distributional DP can only solve the instances that classic DP can solve~\citep{marthe2024beyond}.
We showed that this limitation can be addressed by augmenting the state of the MDP with \emph{\stock{}} (\cref{eq:b-def}), a statistic originally introduced by \citet{bauerle2011markov} for optimizing the $\tau$-CVaR with classic DP, and recurrent within the risk-sensitive RL literature~\citep{lim2022distributional,moghimi2025beyond}, but not beyond.
It is through the combination of distributional RL, \stock{} augmentation and optimizing statistical functionals of the return distribution that distributional DP can tackle a broader class of \problem{} problems than what is possible when any of the components are missing.

We introduced distributional value iteration and distributional policy iteration as principled distributional DP methods for \stock{}-augmented \problem{}, that is, optimizing various objective functionals $F_K$ of the return distribution.
These methods enjoy performance bounds that resemble the classic DP bounds, and they can be applied to various RL-like problems that have been the subject of interest in previous work, including instances of risk-sensitive RL~\citep{bauerle2011markov,chow2014algorithms,noorani2022risk,moghimi2025beyond}, homeostatic regulation~\citep{keramati2011reinforcement} and constraint satisfaction.

Distributional DP offers a clear path for developing practical \problem{} methods based on existing deep RL agents, as exemplified by our empirical results.
We adapted QR-DQN~\citep{dabney2018distributional} to incorporate the principles of distributional DP into a novel agent called D$\eta$N (Deep $\eta$-Networks, pronounced \emph{din}), and illustrated that it works as intended in different simple scenarios for \problem{} in gridworld and Atari.

We believe there are a number of interesting directions for future work in \stock{}-augmented \problem{}.
Besides open theoretical questions, there are various practical challenges to be studied systematically on the path to developing strong practical methods for \problem{}.
Because \problem{} formalizes a wide range of problems, these solution methods can have broad applicability in practice.

\subsection{Open Theoretical Questions}

\emph{Does an optimal return distribution exist when $K$ is indifferent to $\gamma$, indifferent to mixtures and Lipschitz?}
If this is the case, the proofs of \cref{thm:value-iteration,thm:policy-iteration} can be simplified and the bounds can be tightened to depend on the optimal return distribution, similar to how the classic DP error bounds depend on the optimal value function.

\emph{What is needed for DP to optimize an objective functional in the infinite-horizon discounted case?}
We conjecture some form of uniform continuity may be necessary (see \cref{app:lipschitzness}, where we show a failure case with $\Uf$ and $f(x) = \Ind(x > 0)$).
We also conjecture that Lipschitz continuity is needed for uniform bounds to be possible.

\emph{Can we develop distributional DP methods to solve constrained problems?}
We have come close to constrained problems in \cref{sec:constraint-satisfaction}, and it would be interesting to develop a theory of \stock{}-augmented constrained \problem{}, somewhat like constrained MDPs~\citep{altman1999constrained} are related to RL.

\subsection{Addressing \texorpdfstring{D$\eta$N}{DIN}'s Limitations}
D$\eta$N is a proof-of-concept \stock{}-augmented agent that we used for illustrating how the principles underlying distributional value/policy iteration can be incorporated into a deep reinforcement learning agent.
Below, we list some limitations of the method that we believe should be addressed on the path to developing full-fledged \stock{}-augmented agents for optimizing return distributions in challenging environments.

\emph{How to embed the \stock{}?}
We have employed a simple embedding strategy for the \stock{} in D$\eta$N's network, which relies on inputting the \stock{} to an MLP and adding out result to the output of the agent's vision network (see \cref{fig:din-architecture}).
This was sufficient for our experiments, however improved scalar embedding should be considered in the future \citep[for example,][]{springenberg2024offline}, as it may improve the agent's data efficiency and performance, especially in more challenging environments.

\emph{How to go beyond expected utilities?}
The fact that D$\eta$N can only optimize expected utilities is also a limitation worth addressing.
D$\eta$N relies on the existence of greedy actions, which holds for expected utilities, but not for other objective functionals.
That is, other \stock{}-augmented \problem{} problems may only admit optimal stochastic policies.
Perhaps an approach based on policy gradient~\citep{sutton2018reinforcement,espeholt2018impala} or policy optimization~\citep{schulman2017proximal,abdolmaleki2018maximum} may be therefore more suited for going beyond expected utilities.

\emph{How to estimate distributions of vector-valued returns?}
D$\eta$N maintains estimates of the marginal distributions (per coordinate) of the vector-valued returns (see \cref{app:din:implementation-details}).
This was enough for our experiments, but our simplification highlights an important consideration:
We want practical methods that can estimate the distributions of vector-valued returns.
This capability is needed, for example, to tackle the formulation of homeostatic regulation proposed by \citet{keramati2011reinforcement}.
\citet{zhang2021distributional,wiltzer2024foundations} have studied learning distributional estimates with vector-valued returns, so their results can inform the design of distributional estimators for vector-valued returns.

\subsection{Practical Challenges}
Our experimental results revealed a number of interesting challenges in \stock{}-augmented \problem{} that we believe should be addressed in order to develop effective agents for practical settings.

In our experiments we mitigated these issues with simple ideas, and we were helped by the simplicity of the experimental settings, but stronger solutions may be required in more challenging environments.
We typically need to apply interventions to the \stock{} during training, in order to generate diverse data (\cref{sec:gridworld-experiments,sec:atari}).
The interaction of objective functional, $\xc_0$ and approximate return distribution estimates may result in degenerate behavior (\cref{sec:gridworld:risk-averse,sec:gridworld:risk-seeking}) and this can be worsened when $\xc_0$ is selected through a procedure like grid-search to optimize an approximate objective (as in the case of $\tau$-CVaR, both risk-averse and risk-seeking).
Depending on the objective functional, near-optimal decision making may require substantially accurate return estimates (\cref{sec:gridworld:vector-rewards}).
Over long timescales, the discount factor may limit the agent's ability to influence the returns (\cref{sec:atari}).
In more complex environments, we need to ensure the training data is not only diverse across the \stock{} spectrum, but also balanced, lest the learned policies underperform for certain choices of $\xc_0$.

\acks{We thank Csaba Szepesv\'ari for reviewing our draft of this work.
We thank Kalesha Bullard, No\'emi \'Eltet\H{o}, Andr\'as Gy\"orgy, Lucia Cipolina Kun, Dale Schuurmans, and Yunhao Tang for helpful discussions.
We thank Yang Peng for identifying issues on a previous version of this paper and proposing fixes, along with technical feedback and discussions.
We also thank the anonymous JMLR Reviewers, for their technical review and the thoughtful suggestions for improvement, and Martha White for her work as Action Editor to this work.
Our experimental infrastructure was built using Python 3, Flax~\citep{flax2020github}, Haiku~\citep{haiku2020github}, JAX~\citep{jax2018github}, and NumPy~\citep{harris2020array}.
We have used  Matplotlib~\citep{hunter2007matplotlib}, NumPy~\citep{harris2020array}, pandas~\citep{mckinney2010data,reback2020pandas} and SciPy~\citep{virtanen2020scipy} for analyzing and plotting our experimental data.}

\appendix

\section{Additional Theoretical Results}
\label{app:additional-theoretical-results}

\subsection{Complete Spaces}

\begin{lemma}
\label{lem:complete-spaces}
The spaces $(\crvs, \wass)$ and $(\crvs^{\states \times \cs}, \overline{\wass})$ are complete.
\end{lemma}

\begin{proof}
We know that $(\crvs, \wass)$ is complete \citep[Theorem 6.18, p.~116;][]{villani2009optimal}, so it remains to show that $(\crvs^{\states \times \cs}, \overline{\wass})$ is complete.
Let $\eta_1, \eta_2, \ldots$ be a Cauchy sequence in $(\crvs^{\states \times \cs}, \overline{\wass})$.
For each $(s, \xc)$, the sequence $\eta_1(s, \xc), \eta_2(s, \xc), \ldots$ is Cauchy in $(\crvs, \wass)$ and by completeness it has a limit
$\eta_\infty(s, \xc)$.

We claim that $\eta_\infty$ is the limit of $\eta_1, \eta_2, \ldots$ in $(\crvs^{\states \times \cs}, \overline{\wass})$.
Given $\varepsilon > 0$, we can take $n$ such that $\sup_{n' \geq n} \overline{\wass}(\eta_{n'}, \eta_n) < \varepsilon$,
which means 
\begin{align*}
    \varepsilon 
    &> \sup_{n' \geq n} \overline{\wass}(\eta_{n'}, \eta_n) \\
    &= \sup_{n' \geq n}\sup_{s, \xc} \wass(\eta_{n'}(s, \xc), \eta_n(s, \xc)) \\
    &\geq \sup_{n' \geq n}\sup_{s, \xc} \wass(\eta_{n'}(s, \xc), \eta_\infty(s, \xc)) \\
    &=  \sup_{n' \geq n}\overline{\wass}(\eta_{n'}, \eta_\infty),
\end{align*}
and since this holds for all $\varepsilon > 0$ we have that
$\limsup_{n \rightarrow \infty} \overline{\wass}(\eta_n, \eta_\infty) = 0$.
Combining the above with the fact that $\overline{\wass}$ is a norm gives
\[
    0 \leq \liminf_{n \rightarrow \infty} \overline{\wass}(\eta_n, \eta_\infty) \leq \limsup_{n \rightarrow \infty} \overline{\wass}(\eta_n, \eta_\infty) = 0,
\]
so, indeed, $\eta_\infty$ is the limit of $\eta_1, \eta_2, \ldots$.

It remains to show that $\eta_\infty \in (\crvs^{\states \times \cs}, \overline{\wass})$, that is, that $\overline{\wass}(\eta_\infty) < \infty$.
Fix $\varepsilon > 0$ and $n$ such that $\overline{\wass}(\eta_n, \eta_\infty) < \varepsilon$.
We have $\overline{\wass}(\eta_n) < \infty$ since $\eta_n \in (\crvs^{\states \times \cs}, \overline{\wass})$, and, by the triangle inequality,
$\overline{\wass}(\eta_n, \eta_\infty) \geq \overline{\wass}(\eta_\infty) - \overline{\wass}(\eta_n)$,
so $\overline{\wass}(\eta_\infty) \leq \overline{\wass}(\eta_n) + \varepsilon < \infty$.
\end{proof}

\section{Analysis of Distributional Dynamic Programming}
\label{app:ddp}

\subsection{History-based policies}
\label{app:history-based-policies}

We start by reducing the \stock{}-augmented \problem{} problem to an optimization over \emph{Markov policies}.
\begin{proposition}
\label{prop:markov-policies-are-sufficient}
If \cref{ass:bounded-first-moment-rewards} holds
and
$K : (\crvs, \wass) \rightarrow \reals$ is indifferent to mixtures and indifferent to $\gamma$,
and if: i) the MDP has finite horizon; or ii) $\gamma < 1$ and $K$ is Lipschitz,
then
\[
    \sup_{\pi \in \historybased} F_K \eta^\pi = \sup_{\pi \in \markov} F_K \eta^\pi = \sup_{\pi \in \stationary} F_K \eta^\pi.
\]
\end{proposition}

\begin{proof}
We write $F = F_K$.
First note that
\[
    \sup_{\pi \in \stationary} F \eta^\pi \leq \sup_{\pi \in \markov} F \eta^\pi \leq \sup_{\pi \in \historybased} F \eta^\pi,
\]
so it suffices to show that
\[
    \sup_{\pi \in \historybased} F \eta^\pi \leq \sup_{\pi \in \stationary} F \eta^\pi.
\]

We will first consider history-based policies that are eventually stationary Markov.
Recall the definition of a history from \cref{sec:preliminaries}:
\[
  h_t \doteq (s_0, \xc_0), a_0, r_1, (s_1, \xc_1), \ldots, r_t, (s_t, \xc_t)
\]
with $h_0 \doteq (s_0, \xc_0)$.
Let $\Pi_{\mathrm{H}, n}$ be the set of all history-based policies $\rho \in \historybased$ for which there exists a stationary $\pi \in \stationary$ such that, for all $n' \geq n$ and every history $h_{n'}$, we have $\rho(h_{n'}) = \pi(s_{n'}, \xc_{n'})$.
In particular, $\Pi_{\mathrm{H}, 0} = \stationary$.

Assume, by means of induction, that for some $n \in \naturals_0$ we have
\[
    \sup_{\rho \in \Pi_{\mathrm{H}, n}} F \eta^\rho \leq \sup_{\pi \in \stationary} F \eta^\pi.
\]
Given a $\rho \in \Pi_{\mathrm{H}, n+1}$ and its corresponding stationary policy $\pi \in \stationary$, let $\pibar$ satisfy
\[
    F T_{\pibar}\eta^\pi = \sup_{\pi' \in \stationary} FT_{\pi'}\eta^\pi.
\]
By \cref{lem:policy-improvement}, we have $F \eta^{\pibar} \geq F \eta^\pi$.
Now, define the policy $\rho'$ by
\[
    \rhobar(h_t) \doteq 
    \begin{cases}
        \rho(h_t) & \mbox{$t < n$}, \\
        \pibar(s_t, \xc_t) & \mbox{$t \geq n$}.
    \end{cases}
\]
We have that $\rhobar \in \Pi_{\mathrm{H}, n}$, and we now show that $F\eta^{\rhobar} \geq F\eta^\rho$.

Define, for all $(s, \xc) \in \states \times \cs$, $G^\pi(s, \xc) \sim \eta^\pi(s, \xc)$ (and independent from all other random variables) and $G^{\pibar}(s, \xc) \sim \eta^{\pibar}(s, \xc)$ (and independent from all other random variables).
Fix $(S_0, \xC_0) = (s_0, \xc_0)$ (with probability one) and let
$H_n$ be the (random) history and $G^\rho_0$ the return generated by following $\rho$ from $(S_0, \xC_0)$.
Similarly, define the respective $\overline{H}_n$ and $G^{\rhobar}_0$ corresponding to $\rhobar$.

\Cref{eq:balance-non-stationary-invariant} and the definitions above give
\[
    \xC_0 + G^\rho_0 \distequiv \gamma^{-n}( \xC_n + R_{n+1} + \gamma G^\pi(S_{n+1}, \xC_{n+1}) )
\]
and
\[
    \xC_0 + G^{\rhobar}_0 \distequiv \gamma^{-n}( \xC_n + G^{\pibar}(S_n, \xC_n) ).
\]
The choice of $\pibar$ and the fact that $K$ is indifferent to mixtures means that
\[
    K(\xC_n + G^{\pibar}(S_n, \xC_n)) \geq K(\xC_n + R_{n+1} + \gamma G^\pi(S_{n+1}, \xC_{n+1}))
\]
with probability one.
$K$ is also indifferent to $\gamma$, so
\[
    K(\gamma^{-k}(\xC_n + G^{\pibar}(S_n, \xC_n))) \geq K(\gamma^{-k}(\xC_n + R_{n+1} + \gamma G^\pi(S_{n+1}, \xC_{n+1}))),
\]
which implies that $K(\xC_0 + G^{\rhobar}_0) \geq K(\xC_0 + G^\rho_0)$ 
and this holds for every choice of $(s_0, \xc_0)$, so $F\eta^{\rhobar} \geq F\eta^\rho$.
Thus, by induction, we have that for all $n \in \naturals_0$
\begin{equation}
    \sup_{\rho \in \Pi_{\mathrm{H}, n}} F \eta^\rho \leq \sup_{\pi \in \stationary} F \eta^\pi.
    \label{eq:n-step-history-policy-approximation}
\end{equation}

\Cref{eq:n-step-history-policy-approximation} is sufficient for the finite-horizon case, since we can take $n$ large enough so that
\[
    \sup_{\rho \in \Pi_{\mathrm{H}, n}} F \eta^\rho = \sup_{\pi \in \historybased} F \eta^\pi.
\]

For the infinite-horizon discounted case, we proceed as follows.
Fix $n \in \naturals_0$, and fix $\pi \in \historybased$ and $\rho \in \Pi_{\mathrm{H}, n}$ such that $\pi$ and $\rho$ are identical for all histories of size strictly less than $n$.
For $t \in \naturals_0$, let $G^\pi_t(s, \xc)$ denote the return from time step $t$ onward generated by following $\pi$ from starting augmented state $(s, \xc)$.
Note that the arguments $(s, \xc)$ are the initial state of the history, not the augmented state at time step $t$.
Similarly, define the corresponding $G^{\rho}_t(s, \xc)$ for $\rho$.
Because $F$ is Lipschitz, we have, for all $(s, \xc) \in \states \times \cs$
\[
    |F\eta^\pi(s, \xc) - F\eta^\rho(s, \xc)| \leq \gamma^n \wass\left(\df(G^\pi_n(s, \xc)), \df(G^{\rho}_n(s, \xc))\right).
\]
By \cref{ass:bounded-first-moment-rewards}, there exists a constant $\kappa$ such that
\[
    \sup_{s \in \states, \xc \in \cs}\wass\left(\df(G^\pi_n(s, \xc)), \df(G^{\rho}_n(s, \xc))\right) \leq \kappa
\]
uniformly for all $\pi$, $\rho$ and $n$.
Thus, for all $n \in \naturals_0$,
\begin{equation}
    \sup_{\pi \in \historybased}\inf_{\rho \in \Pi_{\mathrm{H}, n}}\sup_{s \in \states, \xc \in \cs}|F\eta^\pi(s, \xc) - F\eta^\rho(s, \xc)| \leq \gamma^n \kappa,
    \label{eq:n-step-history-policy-bound}
\end{equation}
and
\begin{align*}
    \sup_{\pi \in \historybased} F\eta^\pi 
    &\leq \sup_{\rho \in \Pi_{\mathrm{H}, n+1}} F\eta^\rho + \gamma^n \kappa 
    \tag{\cref{eq:n-step-history-policy-bound}} \\
    &= \sup_{\pi \in \stationary} F\eta^\pi + \gamma^n \kappa 
    \tag{\cref{eq:n-step-history-policy-approximation}}
\end{align*}
Taking the limit of $n \rightarrow \infty$ gives the result.
\end{proof}

\Cref{prop:markov-policies-are-sufficient} implies that under the conditions on $F_K$,
for every history-based policy $\pi \in \historybased$ we can find a Markov policy $\pibar \in \markov$ that is no worse than $\pi$ simultaneously for all $(s, \xc)$.
In this sense, the quantity $\sup_{\pi \in \markov} F_K \eta^\pi$
is well-defined, even though it is a supremum of a vector-valued quantity.

\subsection{Distributional Policy Evaluation}

For our analysis, we also employ existing distributional RL theory for policy evaluation:
\begin{theorem}[from Proposition 4.15, p.~88, \citealp{bellemare2023distributional}]
\label{thm:T-pi-contractive}
For every stationary policy $\pi \in \stationary$,
the distributional Bellman operator
$T_\pi$ is a non-expansion in the supremum $1$-Wasserstein distance.
If $\gamma < 1$, then $T_\pi$ is a $\gamma$-contraction in the supremum $1$-Wasserstein distance.
\end{theorem}

\begin{proof}
The proof is as presented by \citet{bellemare2023distributional}, with the caveat that to obtain the result for $\cs = \reals^m$ with $m > 1$ we apply Proposition 4.15 to each coordinate of the vector-valued rewards individually.
\end{proof}

The following lemma uses \cref{thm:T-pi-contractive} to give us a policy evaluation result for the infinite-horizon case.
\begin{lemma}[Distributional Policy Evaluation]
\label{lem:distributional-policy-evaluation}
If $\gamma < 1$ or the MDP has finite horizon,
for any $\eta, \eta' \in (\crvs^{\states \times \cs}, \overline{\wass})$ 
and $\pi \in \markov$ we have
\[
    \lim_{n \rightarrow \infty} \overline{\wass}(T_{\pi_1} \cdots T_{\pi_n}\eta, T_{\pi_1} \cdots T_{\pi_n}\eta') = 0.
\]
\end{lemma}

\begin{proof}
\emph{Discounted Case}.
In this case, $\gamma < 1$ and $T_\pi$ is a $\gamma$-contraction by \cref{thm:T-pi-contractive}.
Letting $\eta_n \doteq T_{\pi_1} \cdots T_{\pi_n}\eta$ and $\eta'_n \doteq T_{\pi_1} \cdots T_{\pi_n}\eta'$ for $n \geq 1$,
for every $n \geq 1$, we have
\[
    \overline{\wass}(\eta_n, \eta'_n) \leq \gamma^n \overline{\wass}(\eta, \eta'),
\]
and 
\[
    \overline{\wass}(\eta, \eta') \leq \overline{\wass}(\eta) + \overline{\wass}(\eta') < \infty.
\]
so $\limsup_{n \rightarrow \infty} \overline{\wass}(\eta_n, \eta'_n) = 0$, which implies the result.

\emph{Finite-horizon Case.}
In finite-horizon MDPs, if $n$ is greater or equal to the horizon,
then
\[
    T_{\pi_1} \cdots T_{\pi_n} \eta = T_{\pi_1} \cdots T_{\pi_n} \eta',
\]
for all $\eta, \eta' \in (\crvs^{\states \times \cs}, \overline{\wass})$, so
\[
    \overline{\wass}(T_{\pi_1} \cdots T_{\pi_n} \eta) = \overline{\wass}(T_{\pi_1} \cdots T_{\pi_n} \eta'),
\]
and we must show is that $T_{\pi_1} \cdots T_{\pi_n} \eta \in (\crvs^{\states \times \cs}, \overline{\wass})$.
When the MDP has finite horizon, $T_\pi$ is a non-expansion (by \cref{thm:T-pi-contractive}),
which implies that $\sup_{\pi \in \stationary}\overline{\wass}(T_{\pi} \eta) < \infty$ and $\overline{\wass}(T_{\pi_1} \cdots T_{\pi_n} \eta) \leq \overline{\wass}(\eta) < \infty$ for all $n \geq 1$.
\end{proof}

We refer to \cref{lem:distributional-policy-evaluation} as the distributional policy evaluation result because it implies that for a stationary policy $\pi \in \stationary$ the sequence discounted return functions given by $\eta_n \doteq T^n_{\pi}\eta$ converges in $1$-Wasserstein distance to $\eta^{\pi}$, the distribution of discounted returns obtained by $\pi$.
Moreover, the sequence of returns $G_n \sim \eta_n$ (which are distributed independently from each other) converges almost surely to a $G^\pi \distequiv \sum_{t=0}^\infty \gamma_t R_{t+1}$ \citep[Skorokhod's Theorem, p.~114;][]{shorack2017probability}

\subsection{Local Policy Improvement}

Informally, DP builds a globally optimal policy by ``chaining'' locally optimal decisions at each time step.
A ``distributional max operator'' gives a return distribution where the first decision is locally optimal:
\begin{definition}[Distributional Max Operator]
\label{def:max-operator}
Given $F : (\crvs^{\states \times \cs}, \overline{\wass}) \rightarrow \reals^{\states \times \cs}$,
an operator $T_* : (\crvs^{\states \times \cs}, \overline{\wass}) \rightarrow (\crvs^{\states \times \cs}, \overline{\wass})$
is a \emph{distributional max operator} if it satisfies, for all $\eta \in (\crvs^{\states \times \cs}, \overline{\wass})$,
\[
    F T_* \eta = \sup_{\pi \in \stationary} F T_\pi \eta.
\]
\end{definition}

The mechanism for locally optimal decision-making is the greedy policy, which is a policy that realizes a distributional max operator:
\begin{definition}[Greedy Policy]
\label{def:greedy-policy}
Given $F : (\crvs^{\states \times \cs}, \overline{\wass}) \rightarrow \reals^{\states \times \cs}$,
a policy $\pi \in \stationary$ is \emph{greedy} with respect to $\eta \in (\crvs^{\states \times \cs}, \overline{\wass})$ if
\[
    F T_{\pi} \eta = F T_* \eta.
\]
\end{definition}

Given $F_K$, it is possible that $K$ is such that for some $\nu \in (\crvs, \wass)$ we have $K\nu$ degenerate and ``infinite'' (for example, the expected utility $\Uf$ with $f(x) = x^{-1}$).
In this case, we interpret $K$ as encoding a preference where if $\nu_1, \nu_2, \ldots (\crvs, \wass)$ converges to $\nu_\infty$ and $K\nu_n < \infty$ for all $n$, but $\liminf_{n \rightarrow \infty} K\nu_n = \infty$, so there is no $\nu \in (\crvs, \wass)$ that is strictly preferred over $\nu_\infty$.
In this sense, we write $K\nu_\infty \geq \sup_{\nu \in (\crvs, \wass)} K\nu$.
Similarly, for $F_K$ and $\pibar$ greedy with respect to $\eta$, we write
\[
    F_K T_*\eta = F_K T_{\pibar} \eta \geq \sup_{\pi \in \stationary} F_K T_\pi \eta
\]
even if the right-hand side is infinite for some $\eta \in (\crvs^{\states \times \cs}, \overline{\wass})$ and $(s, \xc) \in \states \times \cs$.

\subsection{Monotonicity}
\label{app:monotonicity}

The following intermediate result will be useful for proving monotonicity, and it highlights a phenomenon in \stock{}-augmented problems where the rewards are absorbed into the augmented state:
\begin{lemma}[Reward absorption]
\label{lem:absorption}
For every stationary policy $\pi \in \stationary$, $\eta \in (\crvs^{\states \times \cs}, \overline{\wass})$ and $(s, \xc) \in \states \times \cs$,
if $(S_t, \xC_t) = (s, \xc)$, $A_t \sim \pi(S_t, \xC_t)$, $\lookaheadG(s, \xc) \sim (T_\pi\eta)(s, \xc)$ and $G(s, \xc) \sim \eta(s, \xc)$,
then
\[
    \xC_t + \lookaheadG(S_t, \xC_t) \distequiv \gamma \left( \xC_{t+1} + G(S_{t+1}, \xC_{t+1}) \right).
\]
\end{lemma}

\begin{proof}
We have that
\begin{align*}
    \xC_t + \lookaheadG(S_t, \xC_t)
    &\distequiv \xC_t + R_{t+1} + \gamma G(S_{t+1}, \xC_{t+1})
    & \tag{Definition of $T_{\pi}$} \\
    &\distequiv \gamma \xC_{t+1} + \gamma G(S_{t+1}, \xC_{t+1})
    & \tag{\cref{eq:b-def}} \\
    &\distequiv \gamma \left( \xC_{t+1} + G(S_{t+1}, \xC_{t+1}) \right). 
    \tag*{\jmlrqedhere}
\end{align*}
\end{proof}

\monotonicitylemma*

\begin{proof}
Fix a stationary policy $\pi \in \stationary$ and $\eta, \eta' \in (\crvs^{\states \times \cs}, \overline{\wass})$ satisfying $F_K\eta \geq F_K\eta'$.
Fix also $(s, \xc) \in \states \times \cs$, and let $(S_t, \xC_t) = (s, \xc)$, $A_t \sim \pi(S_t, \xC_t)$, $G(s, \xc) \sim \eta(s, \xc)$, $G'(s, \xc) \sim \eta'(s, \xc)$, $\lookaheadG(s, \xc) \sim (T_{\pi}\eta)(s, \xc)$ and $\lookaheadG'(s, \xc) \sim (T_{\pi}\eta')(s, \xc)$

By assumption, we have $K(\xc + G(s, \xc)) \geq K(\xc + G'(s, \xc))$ for all $(s, \xc)$.
Combining the above with indifference to mixtures, we get
\[
    K(\xC_{t+1} + G(S_{t+1}, \xC_{t+1})) \geq K(\xC_{t+1} + G'(S_{t+1}, \xC_{t+1})),
\]
and, thanks to indifference to $\gamma$,
\[
    K(\gamma \left(\xC_{t+1} + G(S_{t+1}, \xC_{t+1})\right)) \geq K( \gamma \left(\xC_{t+1} + G'(S_{t+1}, \xC_{t+1})\right)).
\]
From \cref{lem:absorption} we have that
\begin{align*}
    \xC_t + \lookaheadG(S_t, \xC_t) &\distequiv \gamma \left( \xC_{t+1} + G(S_{t+1}, \xC_{t+1}) \right), \\
    \xC_t + \lookaheadG'(S_t, \xC_t) &\distequiv \gamma \left( \xC_{t+1} + G'(S_{t+1}, \xC_{t+1}) \right),
\end{align*}
so it follows that
\[
    K(\xC_t + \lookaheadG(S_t, \xC_t)) \geq  K(\xC_t + \lookaheadG'(S_t, \xC_t)). 
    \tag*{\jmlrqedhere}
\]
\end{proof}

\subsection{Convergence}

\begin{definition}[Lipschitz Continuity for Objective Functionals]
\label{def:lipschitz-F}
The objective functional
$F : (\crvs^{\states \times \cs}, \overline{\wass}) \rightarrow \reals^{\states \times \cs}$ is \emph{$L$-Lipschitz} (or \emph{Lipschitz}, for simplicity) if there exists $L \in \reals$ such that
\[
     \sup_{\substack{\eta, \eta':\\ \overline{\wass}(\eta) < \infty \\ \overline{\wass}(\eta') < \infty \\ \overline{\wass}(\eta, \eta') > 0}} \frac{\|F\eta  - F\eta'\|_\infty}{\overline{\wass}(\eta, \eta')} \leq L.
\]
$L$ is the \emph{Lipschitz constant} of $F$.
\end{definition}

\begin{proposition}
\label{prop:K-Lipschitz-F-Lipschitz}
Given $K : (\crvs, \wass) \rightarrow \reals$, $F_K$ is $L$-Lipschitz iff $K$ is $L$-Lipschitz.
\end{proposition}

\begin{proof}
If $F_K$ is $L$-Lipschitz, then
\begin{align*}
    L 
    &\geq \sup_{\substack{\eta, \eta':\\ \overline{\wass}(\eta) < \infty \\ \overline{\wass}(\eta') < \infty \\ \overline{\wass}(\eta, \eta') > 0}} \frac{\|F_K\eta  - F_K\eta'\|_\infty}{\overline{\wass}(\eta, \eta')} \\
    &\geq \sup_{\xc \in \cs}\sup_{\substack{\nu, \nu':\\ \wass(\nu) < \infty \\ \wass(\nu') < \infty \\ \wass(\nu, \nu') > 0}} \frac{|K(\xc + G) - K(\xc + G')|}{\wass(\df(\xc + G), \df(\xc + G'))}
    \tag{$G \sim \nu$, $G' \sim \nu'$} \\
    &\geq \sup_{\substack{\nu, \nu':\\ \wass(\nu) < \infty \\ \wass(\nu') < \infty \\ \wass(\nu, \nu') > 0}} \frac{|K\nu - K\nu'|}{\wass(\nu, \nu')}
    \tag{$\xc = 0$},
\end{align*}
so $K$ is $L$-Lipschitz. If, on the other hand, $K$ is $L$-Lipschitz, then, for all $\eta, \eta' \in (\crvs^{\states \times \cs}, \overline{\wass})$,
\begin{align*}
    \|F_K\eta  - F_K\eta'\|_\infty
    &=\sup_{(s, \xc) \in \states \times \cs} |K(\xc + G(s, \xc)) - K(\xc + G'(s, \xc))|
    \tag{$G(s, \xc) \sim \eta(s, \xc)$, $G'(s, \xc) \sim \eta'(s, \xc)$} \\
    &\leq \sup_{(s, \xc) \in \states \times \cs} L \cdot \wass(\eta(s, \xc), \eta'(s, \xc)) \\
    &= L \cdot \overline{\wass}(\eta, \eta'),
\end{align*}
so $F_K$ is $L$-Lipschitz.
\end{proof}

\begin{proposition}
\label{prop:F-convergence}
If $F : (\crvs^{\states \times \cs}, \overline{\wass}) \rightarrow \reals^{\states \times \cs}$ is Lipschitz and the sequence $\eta_1, \eta_2, \ldots \in (\crvs^{\states \times \cs}, \overline{\wass})$ converges in $\overline{\wass}$ to some $\eta_\infty$, then $F \eta_1, F \eta_2, \ldots \in \reals^{\states \times \reals}$ converges in supremum norm to $F \eta_\infty$.
\end{proposition}

\begin{proof}
If $\eta_1, \eta_2, \ldots \in (\crvs^{\states \times \reals}, \overline{\wass})$ converges in $\overline{\wass}$ to some $\eta_\infty$ and $F$ is $L$-Lipschitz,  then
\[
    \limsup_{n \rightarrow \infty} \| F \eta_n - F \eta_\infty \|_\infty \leq L \cdot \limsup_{n \rightarrow \infty} \overline{\wass}(\eta_n, \eta_\infty) = 0,
\]
which gives the result.
\end{proof}

The convergence highlighted in \cref{prop:F-convergence} is somewhat surprising:
If we consider $K\nu = \E(G)$ ($G \sim \nu)$, we have
\[
    \| F_K\delta_0 \|_{\infty} = \sup_{\xc \in \cs}|K\df(\xc + 0)| = \sup_{\xc \in \cs}|\xc| = \infty,
\]
so these objective functionals may have unbounded supremum norm.
However, the difference of the objective functionals for $\eta, \eta' \in (\crvs^{\states \times \reals}, \overline{\wass})$ (namely, $F \eta - F \eta'$) does have bounded supremum norm when $F$ is Lipschitz, and we can show convergence of $F\eta_n$ to $F\eta_\infty$.

\begin{lemma}
\label{lem:optimal-dp-objective}
If $K : (\crvs, \wass) \rightarrow \reals$ is indifferent to mixtures and indifferent to $\gamma$,
and if: i) the MDP has finite horizon; or ii) $\gamma < 1$ and $K$ is Lipschitz,
then for all $\eta \in (\crvs^{\states \times \cs}, \overline{\wass})$
\begin{equation}
    \sup_{\pi \in \markov}F_K\eta^\pi = \lim_{n \rightarrow \infty} \sup_{\pi_1,\ldots,\pi_n \in \stationary} F_K T_{\pi_n} \cdots T_{\pi_1} \eta.
    \label{eq:optimality-limit}
\end{equation}
If $\gamma < 1$ and $K$ is $L$-Lipschitz, then for all $n \geq 0$,
\begin{equation}
    \sup_{\pi \in \markov}F_K\eta^\pi \leq \sup_{\pi_1,\ldots,\pi_n \in \stationary} F_K T_{\pi_n} \cdots T_{\pi_1} \eta + L\gamma^n \cdot \sup_{\pi' \in \markov}\overline{\wass}(\eta, \eta^{\pi'}).
    \label{eq:optimality-bound}
\end{equation}
\end{lemma}

\begin{proof}
We write $F = F_K$ for the rest of the proof.

If the MDP has finite horizon, then for all $\eta \in (\crvs^{\states \times \cs}, \overline{\wass})$
\[
    \sup_{\pi \in \markov}F\eta^\pi = \sup_{\pi_1,\ldots,\pi_n \in \stationary} F T_{\pi_n} \cdots T_{\pi_1} \eta,
\]
where $n$ is the horizon of the MDP.

Otherwise, assume that $\gamma < 1$ and assume that $K$ is $L$-Lipschitz.
Then $F$ is also $L$-Lipschitz, by \cref{prop:K-Lipschitz-F-Lipschitz}.
By the triangle inequality, the fact that $\eta \in (\crvs^{\states \times \cs}, \overline{\wass})$ and \cref{ass:bounded-first-moment-rewards} we have
\[
    \sup_{\pi' \in \markov}\overline{\wass}(\eta, \eta^{\pi'}) \leq \overline{\wass}(\eta) + \sup_{\pi' \in \markov}\overline{\wass}(\eta^{\pi'}) < \infty,
\]
so \Cref{eq:optimality-bound} implies \Cref{eq:optimality-limit} in limit $n \rightarrow \infty$.

It remains to prove \Cref{eq:optimality-bound}.
Let
\[
    g_{s, \xc}(n) \doteq \sup_{\pi_1,\ldots,\pi_n \in \stationary} (F T_{\pi_n} \cdots T_{\pi_1} \eta)(s, \xc) - \sup_{\pi \in \markov}(F\eta^\pi)(s, \xc)
\]
and
\[
    h(n) \doteq \sup_{\pi_1,\ldots,\pi_n \in \stationary}\sup_{\pi' \in \markov}\| FT_{\pi_1}\cdots T_{\pi_n}\eta - FT_{\pi_1}\cdots T_{\pi_n}\eta^{\pi'} \|_\infty.
\]
We will show that, for all $n \geq 0$ and $(s, \xc) \in \states \times \cs$, we have
\[
    | g_{s, \xc}(n) | \leq h(n) \leq L\gamma^n \cdot \sup_{\pi' \in \markov}\overline{\wass}(\eta, \eta^{\pi'}).
\]

For all $n \geq 0$ and $(s, \xc) \in \states \times \cs$, we have
\begin{align*}
    g_{s, \xc}(n)
    &= \sup_{\pi' \in \markov}(F\eta^{\pi'})(s, \xc) - \sup_{\pi_1,\ldots,\pi_n \in \stationary} (F T_{\pi_n} \cdots T_{\pi_1} \eta)(s, \xc) \\
    &= \sup_{\pi' \in \markov}\inf_{\pi_1,\ldots,\pi_n \in \stationary}\left((F\eta^{\pi'})(s, \xc) - (F T_{\pi_n} \cdots T_{\pi_1} \eta)(s, \xc) \right) \\
    &= \sup_{\pi'_1,\ldots,\pi'_n \in \stationary}\sup_{\pi' \in \markov}\inf_{\pi_1,\ldots,\pi_n \in \stationary}\left((FT_{\pi'_1}\cdots T_{\pi'_n}\eta^{\pi'})(s, \xc) - (F T_{\pi_n} \cdots T_{\pi_1} \eta)(s, \xc)\right)
    \tag{$\pi'$ is non-stationary} \\
    &\leq \sup_{\pi_1,\ldots,\pi_n \in \stationary}\sup_{\pi' \in \markov} \left( (FT_{\pi_1}\cdots T_{\pi_n}\eta^{\pi'})(s, \xc) - (FT_{\pi_1}\cdots T_{\pi_n}\eta)(s, \xc) \right) \\
    &\leq \sup_{\pi_1,\ldots,\pi_n \in \stationary}\sup_{\pi' \in \markov} \left| (FT_{\pi_1}\cdots T_{\pi_n}\eta^{\pi'})(s, \xc) - (FT_{\pi_1}\cdots T_{\pi_n}\eta)(s, \xc) \right| \\
    &= \sup_{\pi_1,\ldots,\pi_n \in \stationary}\sup_{\pi' \in \markov} \| FT_{\pi_1}\cdots T_{\pi_n}\eta^{\pi'} - FT_{\pi_1}\cdots T_{\pi_n}\eta \|_\infty \\
    &= h(n).
\end{align*}
and
\begin{align*}
    -g_{s, \xc}(n)
    &= \sup_{\pi_1,\ldots,\pi_n \in \stationary} (F T_{\pi_n} \cdots T_{\pi_1} \eta)(s, \xc) - \sup_{\pi' \in \markov}(F\eta^{\pi'})(s, \xc) \\
    &= \sup_{\pi_1,\ldots,\pi_n \in \stationary}\inf_{\pi' \in \markov} \left( (F T_{\pi_n} \cdots T_{\pi_1} \eta)(s, \xc) - (F\eta^{\pi'})(s, \xc) \right) \\
    &= \sup_{\pi_1,\ldots,\pi_n \in \stationary}\inf_{\pi'_1,\ldots,\pi'_n \in \stationary}\inf_{\pi' \in \markov} \left( (F T_{\pi_n} \cdots T_{\pi_1} \eta)(s, \xc) - (FT_{\pi'_1}\cdots T_{\pi'_n}\eta^{\pi'})(s, \xc) \right)
    \tag{$\pi'$ is non-stationary} \\
    \\
    &\leq \sup_{\pi_1,\ldots,\pi_n \in \stationary}\sup_{\pi' \in \markov} \left( (FT_{\pi_1}\cdots T_{\pi_n}\eta)(s, \xc) - (FT_{\pi_1}\cdots T_{\pi_n}\eta^{\pi'})(s, \xc) \right) \\
    &\leq \sup_{\pi_1,\ldots,\pi_n \in \stationary}\sup_{\pi' \in \markov} \left| (FT_{\pi_1}\cdots T_{\pi_n}\eta)(s, \xc) - FT_{\pi_1}\cdots T_{\pi_n}\eta^{\pi'})(s, \xc) \right| \\
    &\leq \sup_{\pi_1,\ldots,\pi_n \in \stationary}\sup_{\pi' \in \markov} \|FT_{\pi_1}\cdots T_{\pi_n}\eta - FT_{\pi_1}\cdots T_{\pi_n}\eta^{\pi'} \|_\infty \\
    &= h(n)
\end{align*}
Thus, $-h(n) \leq g_{s, \xc}(n) \leq h(n)$, which implies $| g_{s, \xc}(n) | \leq h(n).$

Finally, for all $n \geq 0$, we have
\begin{align*}
    h(n)
    &= \sup_{\pi_1,\ldots,\pi_n \in \stationary}\sup_{\pi' \in \markov}\| FT_{\pi_1}\cdots T_{\pi_n}\eta - FT_{\pi_1}\cdots T_{\pi_n}\eta^\pi \|_\infty \\
    &\leq L \cdot \sup_{\pi_1,\ldots,\pi_n \in \stationary}\sup_{\pi' \in \markov}\overline{\wass}(T_{\pi_1}\cdots T_{\pi_n}\eta, T_{\pi_1}\cdots T_{\pi_n}\eta^\pi)
    \tag{$F$ is $L$-Lipschitz} \\
    &\leq L \gamma^n \cdot \sup_{\pi' \in \markov}\overline{\wass}(\eta, \eta^\pi).
    \tag*{($\gamma$-contraction)\quad\quad\jmlrqedhere}
\end{align*}
\end{proof}

\subsection{Distributional Dynamic Programming}
\label{app:ddp-main-results-proofs}

\distributionalvalueiterationtheorem*

\begin{proof}
We use $F = F_K$ and note that if $K$ $L$-Lipschitz then $F$ is also $L$-Lipschitz (by \cref{prop:K-Lipschitz-F-Lipschitz}.
Fix $\eta_0 \in (\crvs^{\states \times \cs}, \overline{\wass})$
and let $\eta_n \doteq T^n_*\eta_0$ for $n \geq 1$.

The sequence $\pibar_0, \pibar_1, \pibar_2, \ldots$ satisfies $F\eta_{n+1} = FT_{\pibar_n}\eta_n = FT_*\eta_n$ for all $n \geq 0$.
The definition of a distributional max operator (\cref{def:max-operator}) gives us
\[
    F T_* \eta = \sup_{\pi \in \stationary} F T_{\pi} \eta,
\]
and, by monotonicity (\cref{lem:monotonicity}) and induction, we have for every $n \geq 1$
\begin{equation}
    F T^{n+1}_* \eta_0 = FT_{\pibar_n}\cdots T_{\pibar_0}\eta_0 = \sup_{\pi_0,\ldots,\pi_n \in \stationary} F T_{\pi_n} \cdots T_{\pi_0} \eta_0.
\end{equation}
Then \Cref{eq:finite-horizon-value-iterates,eq:discounted-value-iterates} follow from \cref{lem:optimal-dp-objective} combined with \cref{prop:markov-policies-are-sufficient}, which ensures that
\[
    \sup_{\pi \in \historybased} F \eta^\pi = \sup_{\pi \in \markov} F \eta^\pi
\]
(the conditions of \cref{lem:optimal-dp-objective,prop:markov-policies-are-sufficient} are satisfied).

\Cref{eq:finite-horizon-value-greedy} follows from \Cref{eq:finite-horizon-value-iterates} combined with distributional policy improvement (\cref{lem:policy-improvement}).
To see that \cref{lem:policy-improvement} applies, note that, since the MDP has horizon $n$,
\[
    F T_{\pibar_n} T_{\pibar_{n-1}} \cdots T_{\pibar_0} \eta_0 = F T_{\pibar_{n-1}} \cdots T_{\pibar_0} \eta_0,
\]
which satisfies \Cref{eq:one-step-improvement} (with $\eta = T_{\pibar_{n-1}} \cdots T_{\pibar_0} \eta_0$).
Then \cref{lem:policy-improvement} gives
\[
    F \eta^{\pibar_n} \geq F T_{\pibar_{n-1}} \cdots T_{\pibar_0} \eta_0 = \sup_{\pi_0,\ldots,\pi_{n-1} \in \stationary} F T_{\pi_{n-1}} \cdots T_{\pi_0} \eta_0.
\]

It remains to prove \Cref{eq:discounted-value-greedy}.
We start by bounding the following quantity, for $n, k \geq 0$:
\[
    \| FT^k_{\pibar_n}\eta_{n+1} - FT^k_{\pibar_n}\eta_n \|_{\infty}.
\]
For all $n, k \geq 0$ and $(s, \xc) \in \states \times \cs$, we have
\begin{align*}
    &(FT^k_{\pibar_n}\eta_{n+1})(s, \xc) - (FT^k_{\pibar_n}\eta_n)(s, \xc) \\
    &= (FT^k_{\pibar_n}T^n_*\eta_1)(s, \xc) - (FT^k_{\pibar_n}T^n_*\eta_0)(s, \xc) \\
    &= \sup_{\pi_1,\ldots,\pi_n}(FT^k_{\pibar_n}T_{\pi_1}\cdots T_{\pi_n}\eta_1)(s, \xc) - \sup_{\pi'_1,\ldots,\pi'_n}(FT^k_{\pibar_n}T_{\pi'_1}\cdots T_{\pi'_n}\eta_0)(s, \xc) \\
    &\leq \sup_{\pi_1,\ldots,\pi_n}\left((FT^k_{\pibar_n}T_{\pi_1}\cdots T_{\pi_n}\eta_1)(s, \xc) - (FT^k_{\pibar_n}T_{\pi_1}\cdots T_{\pi_n}\eta_0)(s, \xc) \right) \\
    &\leq \sup_{\pi_1,\ldots,\pi_n}\left\|FT^k_{\pibar_n}T_{\pi_1}\cdots T_{\pi_n}\eta_1 - FT^k_{\pibar_n}T_{\pi_1}\cdots T_{\pi_n}\eta_0 \right\|_\infty \\
    &\leq L \cdot \sup_{\pi_1,\ldots,\pi_n}\overline{\wass}(T^k_{\pibar_n}T_{\pi_1}\cdots T_{\pi_n}\eta_1, T^k_{\pibar_n}T_{\pi_1}\cdots T_{\pi_n}\eta_0) 
    \tag{$F$ $L$-Lipschitz} \\
    &\leq L \gamma^{n+k}\overline{\wass}(\eta_1, \eta_0) 
    \tag{$\gamma$-contraction}
\end{align*}
and by a symmetric argument it also holds that for all $n, k \geq 0$ and $(s, \xc) \in \states \times \cs$
\[
    (FT^k_{\pibar_n}\eta_{n+1})(s, \xc) - (FT^k_{\pibar_n}\eta_n)(s, \xc) \geq - L \gamma^{n+k}\overline{\wass}(\eta_1, \eta_0).
\]
so
\begin{align}
    \| FT^k_{\pibar_n}\eta_{n+1} - FT^k_{\pibar_n}\eta_n \|_{\infty} 
    &\leq L \gamma^{n+k}\overline{\wass}(\eta_1, \eta_0) 
    \notag \\
    &\leq L \gamma^{n+k}\sup_{\pi \in \stationary}\overline{\wass}(T_\pi\eta_0, \eta_0)    
    \label{eq:vi-telescoping-term-bound}
\end{align}

Recall that $\pibar_n$ realizes $T_*\eta_n$, so $T_{\pibar_n}\eta_n = \eta_{n+1}$.
Then, for all $n \geq 0$, we have
\begin{align*}
    \| F\eta^{\pibar_n} - F\eta_n \|_\infty
    &\leq \sum_{k=0}^\infty \| FT^{k+1}_{\pibar_n}\eta_n - FT^k_{\pibar_n}\eta_n \|_{\infty}
    \tag{Telescoping and triangle inequality} \\
    &= \sum_{k=0}^\infty \| FT^k_{\pibar_n}\eta_{n+1} - FT^k_{\pibar_n}\eta_n \|_{\infty}
    \tag{$T_{\pibar_n}\eta_n = \eta_{n+1}$} \\
    &\leq \sum_{k=0}^\infty L \gamma^{n+k}\sup_{\pi \in \stationary}\overline{\wass}(T_\pi\eta_0, \eta_0) 
    \tag{\cref{eq:vi-telescoping-term-bound}} \\
    &= \frac{L\gamma^n}{1 - \gamma}\sup_{\pi \in \stationary}\overline{\wass}(T_\pi\eta_0, \eta_0).
\end{align*}
We have already established (in \cref{eq:discounted-value-iterates}) that
\[
    \sup_{\pi \in \historybased}F\eta^\pi - F\eta_n \leq L\gamma^n \sup_{\pi \in \markov}\overline{\wass}(\eta_0, \eta^{\pi}),
\]
so
\begin{align*}
    \sup_{\pi \in \historybased}F\eta^\pi - F\eta^{\pibar_n}
    &= \sup_{\pi \in \historybased}F\eta^\pi -F\eta_n + F\eta_n - F\eta^{\pibar_n} \\
    &\leq L\gamma^n \sup_{\pi \in \markov}\overline{\wass}(\eta, \eta^{\pi}) + \frac{L\gamma^n}{1 - \gamma}\sup_{\pi \in \stationary}\overline{\wass}(T_\pi\eta_0, \eta_0),
    \tag*{\jmlrqedhere}
\end{align*}
\end{proof}

A surprising technical detail about \cref{thm:value-iteration} is that distributional value iteration ``works'' (and $F\eta_n$ converges) under the given conditions, even though:
\begin{itemize}[label=-]
    \item $T_*$ may not be a $\gamma$-contraction when $\gamma < 1$,
    \item $T_*$ may not have a unique fixed point (for example, when multiple policies realize $T_*$),
    \item $\eta_n$ may not converge (depending how ties are broken when realizing $T_*$),
    \item an optimal return distribution may not exist, that is, $\eta^*$ such that $F\eta^* = \sup_{\pi \in \historybased}F\eta^\pi$.
\end{itemize}

We can use the basic ideas from \cref{thm:value-iteration} so that distributional policy iteration also works under the same conditions as distributional value iteration.
While distributional value iteration can start from any return distribution iterate $\eta \in (\crvs^{\states \times \cs}, \overline{\wass})$, for policy iteration we require the initial iterate to be a stationary policy $\pi_0 \in \Pi$, so that distributional policy improvement is guaranteed to work (see the discussion of \cref{lem:policy-improvement}).

\distributionalpolicyiterationtheorem*

\begin{proof}
We use $F = F_K$.
For any $n \geq 0$, we have that
\begin{align*}
    F \eta^{\pi_{n+1}}
    &= F T_{\pi_{n+1}}\eta^{\pi_{n+1}}
    \tag{Distributional Bellman equation} \\
    &\geq F T_{\pi_{n+1}}\eta^{\pi_n}
    \tag{\cref{lem:monotonicity,lem:policy-improvement}} \\
    &= F T_*\eta^{\pi_n}
    \tag{Definition of $\pi_{n+1}$ and \cref{def:max-operator}} \\
    &\geq F T^{n+1}_*\eta^{\pi_0}
    \tag{Induction} \\
    &= \sup_{\pi_1,\ldots,\pi_{n+1} \in \stationary} F T_{\pi_1} \cdots T_{\pi_{n+1}} \eta^{\pi_0}.
    \tag{\cref{def:max-operator}}
\end{align*}
Then both \Cref{eq:finite-horizon-policy-iterates,eq:discounted-policy-iterates} follow by combining the above with \cref{lem:optimal-dp-objective} and \cref{prop:markov-policies-are-sufficient}, which ensures that
\[
    \sup_{\pi \in \historybased} F \eta^\pi = \sup_{\pi \in \markov} F \eta^\pi,
\]
(the conditions of \cref{lem:optimal-dp-objective,prop:markov-policies-are-sufficient} are satisfied).
\end{proof}

\optimalitytheorem*

\begin{proof}
We write $F = F_K$.
We first prove the second statement, for which we assume that $\eta^*$ as described exists.
For every policy $\pi \in \markov$, we have $F \eta^* \geq F \eta^\pi$,
so by monotonicity (\cref{lem:monotonicity}), we also have, for all $\pibar \in \stationary$ and $\pi \in \markov$,
\[
    F T_{\pibar}\eta^* \geq F T_{\pibar}\eta^\pi,
\]
so, for all $\pibar \in \stationary$,
\[
    F T_{\pibar} \eta^* \geq \sup_{\pi \in \markov} F T_{\pibar} \eta^\pi.
\]
and thus
\[
    F T_*\eta^* = \sup_{\pibar \in \stationary} F T_{\pibar}\eta^* \geq \sup_{\pibar \in \stationary}\sup_{\pi \in \markov} F T_{\pibar}\eta^\pi = \sup_{\pi \in \markov} F \eta^\pi.
\]
Now, let $\pi^*$ be greedy with respect to $\eta^*$.
Then
\[
    F T_{\pi^*}\eta^* = F T_* \eta^* = F \eta^*,
\]
so \cref{lem:policy-improvement} implies that $F \eta^{\pi^*} \geq F \eta^*$.
The result then follows by using
\cref{prop:markov-policies-are-sufficient}, for which the conditions are satisfied, and which states that
\[
    \sup_{\pi \in \markov} F \eta^\pi = \sup_{\pi \in \historybased} F \eta^\pi.
\]

For the first statement, under the assumption that the supremum is attained by a (possibly history-based) policy $\pi^*$, we can take $\eta^* = \eta^{\pi^*}$.
If, on the other hand, we assume that $\eta^*$ exists, then we have already shown that an optimal stationary policy exists, which implies that the supremum is attained.
\end{proof}

\section{Analysis of the Conditions for Distributional Dynamic Programming}
\label{app:conditions}

\subsection{Proofs}
\label{app:analysis-proofs}

We start with some supporting results for the proof of \cref{lem:expected-utility-conditions}, \cref{lem:expected-utility-conditions:indifferent-to-gamma}.

\begin{proposition}
\label{prop:indifference-to-gamma-equality}
If $\nu \mapsto \E(f(G))$ (with $G \sim \nu$) is indifferent to $\gamma$, then for all $\nu, \nu' \in (\crvs, \wass)$, if $G \sim \nu$, $G' \sim \nu'$ and $\E(f(G)) = \E(f(G'))$: then $\E(f(\gamma G)) = \E(f(\gamma G'))$.
\end{proposition}

\begin{proof}
The result follows by applying indifference to $\gamma$ in both directions:
$\E(f(G)) \geq \E(f(G'))$ implies $\E(f(\gamma G)) \geq \E(f(\gamma G'))$, and $\E(f(G')) \geq \E(f(G))$ implies $\E(f(\gamma G')) \geq \E(f(\gamma G))$.
\end{proof}

\begin{lemma}
\label{lem:indifference-to-gamma-condition}
If $\nu \mapsto \E(f(G))$ (with $G \sim \nu$) is indifferent to $\gamma$, then there exists $\alpha > 0$ such that for all $\xc \in \cs$
\begin{equation}
    f(\gamma \xc) = \alpha f(\xc) + (1 - \alpha) f(0).
    \label{eq:indifference-to-gamma-condition}
\end{equation}
\end{lemma}

\begin{proof}
Assume $\nu \mapsto \E(f(G))$ (with $G \sim \nu$) is indifferent to $\gamma$.

\emph{Case 1: $f(0) = 0$.}
If $f(\xc) = 0$ for all $\xc$, then the result holds trivially (for example, we can take $\alpha = \gamma$.
Otherwise, find $\overline{\xc} \in \cs$ such that $f(\overline{\xc}) \neq 0$.
We will first show that we can satisfy \Cref{eq:indifference-to-gamma-condition} with $\alpha \doteq \frac{f(\gamma \overline{\xc})}{f(\overline{\xc})}$, and later show that $\alpha > 0$.

Fix $\xc \in \cs$ arbitrary.
If $f(\xc) = 0$, then, by \cref{prop:indifference-to-gamma-equality}, we have $f(\gamma \xc) = 0$ and \Cref{eq:indifference-to-gamma-condition} holds for the chosen $\alpha$.
Let us consider the case where $f(\xc) \neq 0$.

If $\frac{f(\xc)}{f(\overline{\xc})} \leq 1$, we proceed as follows:
Define $\nu, \nu'$ such that
$\nu(\overline{\xc}) \doteq \frac{f(\xc)}{f(\overline{\xc})}$, $\nu(0) \doteq 1 - \nu(\overline{\xc})$, $\nu'(\xc) \doteq 1$.
Let $G \sim \nu$ and $G' \sim \nu'$.
Then
\[
    \E(f(G)) = \frac{f(\xc)}{f(\overline{\xc})}f(\overline{\xc}) = f(\xc) = \E(f(G')).
\]
By indifference to $\gamma$ and \cref{prop:indifference-to-gamma-equality}, we have $\E(f(\gamma G)) = \E(f(\gamma G'))$, thus:
\[
    \frac{f(\xc)}{f(\overline{\xc})}f(\gamma \overline{\xc}) = f(\gamma \xc).
\]
Rearranging, we get that
\[
    f(\gamma \xc) = \frac{f(\gamma \overline{\xc})}{f(\overline{\xc})}f(\xc),
\]
which means we can satisfy \Cref{eq:indifference-to-gamma-condition} with $\alpha = \frac{f(\gamma \overline{\xc})}{f(\overline{\xc})}$.

If $\frac{f(\xc)}{f(\overline{\xc})} > 1$, we proceed as follows:
Define $\nu, \nu'$ such that
$\nu(\xc) \doteq \frac{f(\overline{\xc})}{f(\xc)}$, $\nu(0) \doteq 1 - \nu(\xc)$, $\nu'(\overline{\xc}) \doteq 1$.
Let $G \sim \nu$ and $G' \sim \nu'$.
Then
\[
    \E(f(G)) = \frac{f(\overline{\xc})}{f(\xc)}f(\xc) = f(\overline{\xc}) = \E(f(G')).
\]
By indifference to $\gamma$ and \cref{prop:indifference-to-gamma-equality}, we have $\E(f(\gamma G)) = \E(f(\gamma G'))$, thus:
\[
    \frac{f(\overline{\xc})}{f(\xc)}f(\gamma \xc) = f(\gamma \overline{\xc}).
\]
Rearranging, we get that
\[
    f(\gamma \xc) = \frac{f(\gamma \overline{\xc})}{f(\overline{\xc})}f(\xc),
\]
which means we can satisfy \Cref{eq:indifference-to-gamma-condition} with $\alpha = \frac{f(\gamma \overline{\xc})}{f(\overline{\xc})}$.

We have established that \Cref{eq:indifference-to-gamma-condition} holds for all $\xc \in \cs$ with $\alpha = \frac{f(\gamma \overline{\xc})}{f(\overline{\xc})}$, provided that $f(0) = 0$.
It only remains to show that $\alpha > 0$.
If $f(\xc) > 0$, by indifference to $\gamma$ we have $f(\gamma \xc) \geq f(0)$ (since $f(0) = 0$).
Likewise, if $f(\xc) < 0$, then by indifference to $\gamma$ we have $f(0) \geq f(\gamma \xc)$.
In either case, $\alpha \geq 0$.
\Cref{eq:indifference-to-gamma-condition} with $\xc = \gamma^{-1}\cdot\overline{\xc}$ gives $f(\overline{\xc}) = \alpha f(\gamma^{-1}\cdot\overline{\xc})$,
so $\alpha \neq 0$ (since we picked $\overline{\xc}$ such that $f(\overline{\xc}) \neq 0$).
Thus, $\alpha > 0$.

\emph{Case 2: $f(0) \neq 0$.}
We can reduce this to the previous case with $f'(\xc) \doteq f(\xc) - f(0)$, so
there exists $\alpha > 0$ such $f'(\gamma \xc) = \alpha f'(\xc)$ for all $\xc \in \cs$,
which means $f(\gamma \xc) - f(0) = \alpha f(\xc) - \alpha f(0)$, and rearranging gives \Cref{eq:indifference-to-gamma-condition}.
\end{proof}

\expectedutilityconditions*

\begin{proof}
\Cref{lem:expected-utility-conditions:indifferent-to-mixtures} follows essentially from the tower rule.
Letting $G(s, \xc) \sim \eta(s, \xc)$ and $G'(s, \xc) \sim \eta'(s, \xc)$,
we have $K(G(S, \xC)) = \E\left( \E\left( K(G(S, \xC)) | S, \xC \right)\right)$.
If $K\eta \geq K\eta'$, then
\begin{align*}
    K(G(S, \xC)) 
    &= \E f(G(S, \xC)) \\
    &= \E\left( \E \left( f(G(S, \xC)) | S, \xC \right) \right) \\
    &= \E\left( \E \left( K(G(S, \xC)) | S, \xC \right) \right) \\
    &\geq \E\left( \E \left( K(G'(S, \xC)) | S, \xC \right) \right) \\
    &= \E\left( \E \left( f(G'(S, \xC)) | S, \xC \right) \right) \\
    &= \E f(G'(S, \xC)) \\
    &= K(G'(S, \xC)).
\end{align*}

For \cref{lem:expected-utility-conditions:indifferent-to-gamma}, we first establish it for $\alpha > 0$, then we show that \Cref{eq:expected-utility-conditions:indifferent-to-gamma} holds for some $\alpha \in (0, 1]$ with $\gamma < 1 \Rightarrow \alpha < 1$.

\cref{lem:expected-utility-conditions:indifferent-to-gamma} ($\Rightarrow$) follows from \cref{lem:indifference-to-gamma-condition}.
To see the converse ($\Leftarrow$), we proceed as follows.
Assume there exists $\alpha > 0$ such that for all $\xc \in \cs$ \Cref{eq:expected-utility-conditions:indifferent-to-gamma} holds and that $K(G(s, \xc)) \geq K(G'(s, \xc))$.
Then 
\begin{align*}
    K(\gamma G(s, \xc)) 
    &= \E f(\gamma G(s, \xc)) \\
    &= \alpha \E f(G(s, \xc)) + (1 - \alpha) f(0) \\
    &= \alpha K(G(s, \xc)) + (1 - \alpha) f(0) \\
    &\geq \alpha K(G'(s, \xc)) + (1 - \alpha) f(0) \\
    &= \alpha \E f(G'(s, \xc)) + (1 - \alpha) f(0) \\
    &= \E f(\gamma G'(s, \xc)) \\
    &= K(\gamma G'(s, \xc)).
\end{align*}

Now, define $g(\xc) \doteq f(\xc) - f(0)$, and assume \Cref{eq:expected-utility-conditions:indifferent-to-gamma} holds for some $\alpha > 0$.
If $\gamma = 1$ or $f$ is constant, then \Cref{eq:expected-utility-conditions:indifferent-to-gamma} holds trivially for $\alpha = \gamma$.
Let us assume that $\gamma < 1$ and $f$ is not constant. Then, by induction, we have, for all $n \in \naturals_0$, that $g(\gamma^n) = \alpha^n g(1)$, and
\[
    0 = \liminf_{n \rightarrow \infty} g(\gamma^n) = \liminf_{n \rightarrow \infty} \alpha^n g(1) = g(1) \cdot \lim_{n \rightarrow \infty} \alpha^n,
\]
so we must have $\alpha < 1$.

For \cref{lem:expected-utility-conditions:lipschitz}, we proceed as follows:
If $K$ is $L$-Lipschitz, then
\[
    L = \sup_{\substack{\nu, \nu':\\ \wass(\nu) < \infty \\ \wass(\nu') < \infty \\ \wass(\nu, \nu') > 0}} \frac{|K\nu  - K\nu'|}{\wass(\nu, \nu')} \geq \sup_{x \neq x'} \frac{|f(x) - f(x')|}{\wass(\delta_x, \delta_{x'})} = \sup_{x \neq x'} \frac{|f(x) - f(x')|}{\|x - x'\|_1},
\]
which means $f$ is $L$-Lipschitz.
If $f$ is $L$-Lipschitz, then, for all $\nu, \nu' \in (\crvs, \wass)$,
\begin{align*}
    |K\nu  - K\nu'|
    &=|\E f(G)  - \E f(G')|
    \tag{$G \sim \nu$, $G' \sim \nu'$} \\
    &=\inf\left\{ |\E f(X)  - \E f(X')| : \df(X) = \nu, \df(X') = \nu' \right\} \\
    &\leq \inf\left\{ \E |f(X)  - f(X')| : \df(X) = \nu, \df(X') = \nu' \right\} \\
    &\leq L \cdot \inf\left\{ \|X - X'\|_1 : \df(X) = \nu, \df(X') = \nu' \right\} \\
    &= L \cdot \wass(\nu, \nu'),
\end{align*}
which means $K$ is $L$-Lipschitz.
\end{proof}

To get a better understanding of the limits of distributional DP, it is useful to inspect the necessary conditions for it to work.
In the absence of indifference to mixtures or indifference to $\gamma$ we can construct MDPs where greedy optimality (\cref{thm:optimality}) fails due to a lack of monotonicity:
\indifferencesnecessary*

\begin{proof}
\emph{Case 1: $K$ is not indifferent to mixtures.}
Consider $\eta, \eta' \in (\crvs^{\states \times \cs}, \overline{\wass})$ and a mixture distribution $\lambda$ over $\states \times \cs$ such that 
$K\eta \geq K\eta'$ but $K(G(S, \xC)) < K(G'(S, \xC))$ with $(S, \xC) \sim \lambda$, $G(s, \xc) \sim \eta(s, \xc)$, $G'(s, \xc) \sim \eta'(s, \xc)$.

Let $\gamma = 1$ and consider an MDP with state space $\{\sinit, \sterm\} \cup \states$ and action space $\{a, a'\}$ as follows:
State $\sterm$ is terminal;
either action in $(\sinit, 0)$ leads to $(S, \xC)$ where $(S, \xC) \sim \lambda$ (in this case, the reward is $\xC$);
action $a$ on $(s, \xc)$ leads to $\sterm$ with reward sampled according to $\eta(s, \xc)$;
action $a'$ on $(s, \xc)$ leads to $\sterm$ with reward sampled according to $\eta'(s, \xc)$.

In this instance, there exists an optimal non-stationary policy $\pi^*_1\pi^*_2$ such that $\pi^*_2(a' | s, \xc) = 1$ for all $(s, c) \in \states \times \cs$.
Let $\eta^*$ be the return distribution function for this policy.
There exists a stationary policy $\pibar \in \Pi$ that is greedy respect to $\eta^*$ and such that $\pibar(a | s, \xc) = 1$ for all $(s, c) \in \states \times \cs$.
Thus, letting $G(s, \xc) \sim \eta(s, \xc)$, $G'(s, \xc) \sim \eta'(s, \xc)$
\[
    K \eta^{\pibar}(\sinit, 0) 
    = K(G(S, \xC)) 
    < K(G'(S, \xC))
    = K \eta^*(\sinit, 0),
\]
which proves the result.

\emph{Case 2: $K$ is not indifferent to $\gamma$.}
Consider $\nu, \nu'$ for which $K\nu \geq K\nu'$ but $K(\gamma G) < K(\gamma G')$,
with $G \sim \nu$, $G' \sim \nu'$.

Consider an MDP with $\states = \{\sinit, s_{\mathrm{mid}}, \sterm \}$ and $\actions = \{a, a'\}$ as follows:
State $\sinit$ is initial, state $\sterm$ terminal;
state $\sinit$ transitions to state $s_{\mathrm{mid}}$ with either action and zero rewards; 
state $s_{\mathrm{mid}}$ transitions to state $\sterm$ with either action, but with reward distributed according to $\nu$ for $a$ and $\nu'$ for $a'$.

There is an optimal non-stationary policy, corresponding to $\pi^*_1\pi^*_2$, where, for all $\xc \in \cs$, $\pi^*_2(a' | s_{\mathrm{mid}}, \xc) = 1$.
Let $\eta^*$ be the return distribution function for this policy.
The stationary policy $\pibar$ that selects $a$ always is greedy with respect to $\eta^*$, however
\[
    K \eta^{\pibar}(\sinit, 0)
    = K(\gamma G) 
    < K(\gamma G')
    = K \eta^*(\sinit, 0),
\]
which proves the result.
\end{proof}

\subsection{Exploring Lipschitz Continuity}
\label{app:lipschitzness}

We can use the examples in the second part of \cref{tab:utility-examples} to motivate why we may need Lipschitz continuity in the infinite-horizon setting. Neither $f(x) = \Ind(x > 0)$ nor $f(x) = -x^2$ are Lipschitz.
$f(x) = \Ind(x > 0)$ is also not continuous, and it is informative to first consider how the lack of continuity can break distributional value/policy iteration.

Consider, by means of a counter-example, a single-state MDP with two actions $\{a_0, a_1\}$, $\gamma < 1$, and $r(a_i) = i$.
The objective functional is $\Uf$ with $f(x) = \Ind(x > 0)$.
Let $\pi_i$ be the policy that always selects $a_i$.
The return of $\pi_i$ is deterministic and equal to $(1 - \gamma)^{-1}i$.
The policy $\pi_1$ and its return distribution $\eta^{\pi_1}$ are optimal.
The following is a valid greedy policy with respect to $\eta^{\pi_1}$:
\[
    \pibar(\xc) = \begin{cases}
        a_0 & \xc + (1 - \gamma)^{-1}\gamma > 0 \\
        a_1 & \mbox{otherwise.}
    \end{cases}
\]
When starting from the \stock{} $\xc = 0$, taking $\pibar$ for $k$ steps followed by $\pi_1$ yields a return of $(1 - \gamma)^{-1}\gamma^{k} > 0$ (since the first $k$ actions are $a_0$).
We know that the sequence $T_{\pibar}^1 \eta^{\pi_1}, T_{\pibar}^2 \eta^{\pi_1}, \ldots$ converges in supremum $1$-Wasserstein distance to $T_{\pibar}^\infty \eta^{\pi_1} = \eta^{\pibar}$ (see \cref{lem:distributional-policy-evaluation}).
We also have that, for every $k \in \naturals$, $(\Uf T_{\pibar}^k \eta^{\pi_1})(0) = 1$
and $(\Uf \eta^{\pi_1})(0) = 1$, 
so $(\Uf T_{\pibar}^k \eta^{\pi_1})(0) \geq (\Uf \eta^{\pi_1})(0)$. 
However, the inequality fails in the limit: $(\Uf T_{\pibar}^\infty \eta^{\pi_1})(0) = (\Uf \eta^{\pibar})(0) = 0$, whereas $(\Uf \eta^{\pi_1})(0) = 1$.
For this reason, if $\pi_0$ is the chosen greedy policy with respect to $\eta^\pi_1$, then policy improvement (\cref{lem:policy-improvement}) fails, greedy optimality (\cref{thm:optimality}) fails, distributional value iteration starting from $\eta^* = \eta^{\pi_1}$ fails, and distributional policy iteration starting from $\pi^* = \pi_1$ fails.

It is less clear how to design a counter-example when $f$ is continuous but not Lipschitz, 
however we can show a case where where basic ``evaluation'' fails.
Considering $f(x) = -x^2$, which is continuous but not Lipschitz,
and the trivial MDP where $\cs = \reals$ and all rewards are zero.
Consider the function $\eta_0 \doteq (s, \xc) \mapsto \delta_1$.
This is not a value function in the MDP (no policy satisfies $\eta^\pi = \eta_0$), but we may want to use it for bootstrapping in distributional value iteration.
In this particular MDP, $T_*$ with $\gamma < 1$ is a contraction, since $\overline{\wass}(T_*\eta, T_*\eta') \leq \gamma \overline{\wass}(\eta, \eta')$,
and the sequence $\eta_1, \eta_2, \ldots$ where $\eta_{n+1} = T_*\eta_n$ for $n \geq 0$ is Cauchy with respect to $\overline{\wass}$, since
$\overline{\wass}(\eta_n, \eta_{n+k}) = \gamma^n(1 - \gamma^k)$ for all $n, k \geq 0$.
Therefore $\eta_n$ converges to $\eta_{\infty} = (s, \xc) \mapsto \delta_0$.
However, letting $G_n(s, \xc) \sim \eta_n(s, \xc)$,
\begin{align*}
    \|\Uf \eta_n - \Uf \eta_{n+k} \|_\infty 
    &= \sup_{s \in \states, \xc \in \cs} |\E f(\xc + G_n(s, \xc)) - \E f(\xc + G_{n+k}(s, \xc))| \\
    &= \sup_{\xc \in \cs} |(\xc + \gamma^n)^2 - (\xc + \gamma^{n+k})^2| \\
    &= \sup_{\xc \in \cs} |(2\xc + \gamma^n + \gamma^{n+k})(\gamma^n - \gamma^{n+k})| \\
    &= \sup_{\xc \in \cs} |(2\xc + \gamma^n + \gamma^{n+k})| \cdot \gamma^n \cdot (1 - \gamma^k) \\
    &= \infty,
\end{align*}
which means the sequence $\Uf\eta_n$ does not converge uniformly to $\Uf\eta_{\infty}$ as $n \rightarrow \infty$.
We have not been able to translate this failure of convergence to a failure of distributional DP, so it is unclear exactly what kind of convergence-related property of $F_K$ is necessary for distributional DP to work in the infinite-horizon discounted case.

\section{Proofs for \texorpdfstring{\cref{sec:cvar}}{Section~\ref{sec:cvar}}}
\label{app:cvar}

To prove \cref{thm:cvar-c-star}, we follow the strategy used by \citet{bauerle2011markov}, where we reduce $\tau$-CVaR optimization to solving the \stock{}-augmented \problem{} problem with the expected utility $\Uf$ and $f(x) = x_-$, but where the starting \stock{} $\xc_0$ must be chosen in a specific way as a function of $s_0$.

We start with a reduction of the $\tau$-CVaR to an optimization problem, as shown in previous work, and some intermediate results.
\begin{theorem}[\citealp{rockafellar2000optimization}]
\label{thm:cvar}
For all $\nu \in (\Delta(\reals), \wass)$ and $\tau \in (0, 1)$,
\[
    \cvar(\nu, \tau) = \max_{\xc} \left( \xc + \frac{1}{\tau}\E(G - \xc)_- \right),
    \tag{$G \sim \nu$}
\]
and the maximum is attained at $\qf_\nu(\tau)$.
\end{theorem}

\begin{proposition}
\label{prop:cvar-objective-1-lipschitz}
For all $s \in \states$, the function $\xc \mapsto -\xc + \frac{1}{\tau}\sup_{\pi \in \historybased} \E (\xc + G^\pi(s, \xc))_-$ is $1$-Lipschitz.
\end{proposition}

\begin{proof}
Fix $s \in \states$ and let
\[
    g(\xc) \doteq -\xc + \frac{1}{\tau}\sup_{\pi \in \historybased} \E (\xc + G^\pi(s, \xc))_-.
\]

For $\varepsilon \geq 0$, we have that
\begin{align*}
    \sup_{\pi \in \historybased}\E(\xc + G^\pi(s, \xc))_- 
    &\leq \sup_{\pi \in \historybased}\E(\xc + \varepsilon + G^\pi(s, \xc))_- 
    \tag{$(x + \varepsilon)_- \geq x_-$} \\
    &= \sup_{\pi \in \historybased}\E(\xc + \varepsilon + G^\pi(s, \xc + \varepsilon))_-,
\end{align*}
where the last line follows by noticing that the value in the \stock{} augmentation does not change the supremum over history-based policies.

We can apply the same reasoning to see that
\[
    \sup_{\pi \in \historybased}\E(\xc - \varepsilon + G^\pi(s, \xc - \varepsilon))_- \leq \sup_{\pi \in \historybased}\E(\xc - \varepsilon + \varepsilon + G^\pi(s, \xc - \varepsilon + \varepsilon))_- = \sup_{\pi \in \historybased}\E(\xc + G^\pi(s, \xc))_-.
\]
Thus for every $\varepsilon \geq 0$
\begin{align*}
    g(\xc - \varepsilon) 
    &= -(\xc - \varepsilon) + \frac{1}{\tau}\sup_{\pi \in \historybased}\E(\xc - \varepsilon + G^\pi(s, \xc - \varepsilon))_- \\
    &\leq -\xc  + \varepsilon + \frac{1}{\tau}\sup_{\pi \in \historybased}\E(\xc + G^\pi(s, \xc))_- \\
    &= g(\xc) + \varepsilon,
\end{align*}
and
\begin{align*}
    g(\xc + \varepsilon) 
    &= -(\xc + \varepsilon) + \frac{1}{\tau}\sup_{\pi \in \historybased}\E(\xc + \varepsilon + G^\pi(s, \xc + \varepsilon))_- \\
    &\geq -\xc  - \varepsilon + \frac{1}{\tau}\sup_{\pi \in \historybased}\E(\xc + G^\pi(s, \xc))_- \\
    &= g(\xc) - \varepsilon,
\end{align*}
That is:
\begin{equation}
    g(\xc - \varepsilon) - \varepsilon \leq g(\xc) \leq g(\xc + \varepsilon) + \varepsilon \label{eq:g-lipschitz-bound}
\end{equation}

Thus, for $\xc, \xc' \in \reals{}$, letting $\cmax = \max\{\xc, \xc'\}$ and $\cmin = \min\{\xc, \xc'\}$, we have
\begin{align*}
    -(\cmax - \cmin) 
    &\leq g(\cmax) - g(\cmax - (\cmax - \cmin)) 
    \tag{\cref{eq:g-lipschitz-bound} with $\varepsilon = \cmax - \cmin$} \\
    &=g(\cmax) - g(\cmin) \\
    &= g(\cmax) - g(\cmin + (\cmax - \cmin)) \\
    &\leq \cmax - \cmin
    \tag{\cref{eq:g-lipschitz-bound} with $\varepsilon = \cmax - \cmin$},
\end{align*}
so
\[
    |g(\xc) - g(\xc')| = |g(\cmax) - g(\cmin)| \leq |\cmax - \cmin| = |\xc - \xc'|,
\]
which means $g$ is $1$-Lipschitz.
\end{proof}

\cvartheorem*

\begin{proof}
By \cref{thm:cvar}, for all $s_0 \in \states$,
\begin{align*}
    \sup_{\pi \in \historybased} \cvar(\eta^\pi(s_0), \tau)
     &=\sup_{\pi \in \historybased}\max_{\xc_0} \left( \xc_0 + \frac{1}{\tau}\E(G^\pi(s_0) - \xc_0)_- \right) \\
     &= \sup_{\xc_0} \left( \xc_0 + \frac{1}{\tau}\sup_{\pi \in \historybased}\E(G^\pi(s_0) - \xc_0)_- \right) \\
     &= \sup_{\xc_0} \left( -\xc_0 + \frac{1}{\tau}\sup_{\pi \in \historybased}\E(\xc_0 + G^\pi(s_0))_- \right) \\
     &= \sup_{\xc_0} \left( -\xc_0 + \frac{1}{\tau}\sup_{\pi \in \historybased}\E(\xc_0 + G^\pi(s_0, \xc_0))_- \right)
     \tag{\cref{prop:markov-policies-are-sufficient}}
\end{align*}

It only remains to show that for all $s_0 \in \states$ there exists $\xc_0^*$ that realizes the supremum over $\xc_0$.
Note that by \cref{ass:bounded-first-moment-rewards},
we have
\begin{equation}
    \sup_{\pi \in \historybased}\E(G^\pi(s_0, \xc)) < \infty.
    \label{eq:finite-optimal-value}
\end{equation}
For all $s_0 \in \states$, we have
\begin{align*}
    &\lim_{\xc_0 \rightarrow \infty} -\xc_0 + \frac{1}{\tau}\sup_{\pi \in \historybased}\E(\xc_0 + G^\pi(s_0, \xc_0))_- \\
    &\leq \lim_{\xc_0 \rightarrow \infty} -\xc_0 
    \tag{\cref{eq:finite-optimal-value}}\\
    &= -\infty.
\end{align*}
and
\begin{align*}
    &\lim_{\xc_0 \rightarrow -\infty} -\xc_0 + \frac{1}{\tau}\sup_{\pi \in \historybased}\E(\xc_0 + G^\pi(s_0, \xc_0))_- \\
    &=\lim_{\xc_0 \rightarrow -\infty} \frac{1 - \tau}{\tau}\xc_0 + \sup_{\pi \in \historybased}\E(G^\pi(s_0, \xc_0)) - \E(\xc_0 + G^\pi(s_0, \xc_0))_+ \\
    &\leq\lim_{\xc_0 \rightarrow -\infty} \frac{1 - \tau}{\tau}\xc_0 + \sup_{\pi \in \historybased}\E(G^\pi(s_0, \xc))\\
    &= -\infty.
    \tag{\cref{eq:finite-optimal-value}}
\end{align*}

Therefore there exist $\cmin, \cmax \in \reals$ such that
\begin{align*}
    &\sup_{\xc_0} \left( -\xc_0 + \frac{1}{\tau}\sup_{\pi \in \historybased}\E(\xc_0 + G^\pi(s_0, \xc_0))_- \right) \\
    &= \sup_{\cmin \leq \xc_0 \leq \cmax} \left( -\xc_0 + \frac{1}{\tau}\sup_{\pi \in \historybased}\E(\xc_0 + G^\pi(s_0, \xc_0))_- \right).
\end{align*}
Moreover, \cref{prop:cvar-objective-1-lipschitz} implies $\xc \mapsto -\xc + \frac{1}{\tau}\sup_{\pi \in \historybased}\E(\xc + G^\pi(s_0, \xc))_-$ is continuous.
Therefore the supremum over $\xc_0$ is attained at a maximizer $\xc^*_0 \in \reals$.
\end{proof}

\approximatecvartheorem*

\begin{proof}
Let us fix $\tau$, $s_0 \in \states$, $\varepsilon > 0$, $f(x) = x_-$, and define
\[
	g(\xc_0) \doteq -\xc_0 + \frac{1}{\tau}\sup_{\pi \in \historybased}\E(\xc_0 + G^\pi(s_0, \xc_0))_-.
\]
\citet{bauerle2011markov} (\cref{thm:cvar-c-star}) established that
\[
    \sup_{\pi \in \historybased, \xc_0 \in \cs}\cvar(\eta^\pi(s_0, \xc_0), \tau) = \sup_{\xc_0}g(\xc_0).
\]

By using distributional DP (\cref{thm:value-iteration,thm:policy-iteration}), we can find a near-optimal policy for optimizing $\Uf$, that is, a $\pibar$ satisfying
\[
	\sup_{\pi \in \historybased}\Uf\eta^\pi - \Uf\eta^{\pibar} \leq \varepsilon.
\]
Let
\[
	\overline{g}(\xc_0) \doteq -\xc_0 + \frac{1}{\tau}\E(\xc_0 + G^{\pibar}(s_0, \xc_0))_-.
\]
Then $| g(\xc_0) - \overline{g}(\xc_0) | \leq \varepsilon$ for all $\xc_0 \in \cs$.
Moreover, by \cref{prop:cvar-objective-1-lipschitz}, $g$ is $1$-Lipschitz, so for all $\xc_0, \xc'_0 \in \cs$
\[
	| g(\xc_0) - g(\xc'_0) | \leq |\xc_0 - \xc'_0|,
\]
and
\[
    | \overline{g}(\xc_0) - \overline{g}(\xc'_0) | 
    \leq | \overline{g}(\xc_0) - g(\xc'_0) | + | g(\xc_0) - g(\xc'_0) |  + | g(\xc'_0) - \overline{g}(\xc'_0) |
    \leq |\xc_0 - \xc'_0| + 2\varepsilon.
\]
This means we can choose $\cmin \leq \cmax$ such that
\[
    \max_{\xc_0}\overline{g}(\xc_0) = \max_{\cmin \leq \xc_0 \leq \cmax}\overline{g}(\xc_0). 
\]
Define the grid $\overline{\cs} \doteq \{ \cmin + i\varepsilon : i \in \naturals_0, \cmin + i\varepsilon \leq \cmax \}$,
Then
\begin{align*}
    \sup_{\pi \in \historybased, \xc_0 \in \cs}\cvar(\eta^\pi(s_0, \xc_0), \tau)
    &= \sup_{\xc_0}g(\xc_0) 
    \tag{\cref{thm:cvar-c-star}} \\
    &\leq \max_{\xc_0}\overline{g}(\xc_0) + \varepsilon \\
    &= \max_{\cmin \leq \xc_0 \leq \cmax}\overline{g}(\xc_0) + \varepsilon \\
    &\leq \sup_{\cmin \leq \xc_0 \leq \cmax}g(\xc_0) + 2\varepsilon \\
    &\leq \max_{\xc_0 \in \overline{\cs}}g(\xc_0) + 3\varepsilon \\
    &\leq \max_{\xc_0 \in \overline{\cs}}\overline{g}(\xc_0) + 4\varepsilon \\
    &\leq \cvar(\eta^{\pibar}(s_0, \overline{\xc}^*_0), \tau) + 4\varepsilon
    \tag*{(\cref{thm:cvar})\quad\quad\jmlrqedhere}
\end{align*}
\end{proof}

\section{Proofs for \texorpdfstring{\cref{sec:ocvar}}{Section~\ref{sec:ocvar}}}
\label{app:ocvar}

\begin{lemma}
\label{lem:ocvar}
For all $\nu \in (\Delta(\reals), \wass)$ and $\tau \in (0, 1)$,
\[
    \ocvar(\nu, \tau) = \min_{\xc} \left( \xc + \frac{1}{\tau}\E(G - \xc)_+ \right),
    \tag{$G \sim \nu$}
\]
and the minimum is attained at $\qf_\nu(\tau)$.
\end{lemma}

\begin{proof}
The proof of this result is derived from the proof of \cref{thm:cvar} by \citet{rockafellar2000optimization}.
Fix $\nu \in (\Delta(\reals), \wass)$ and $\tau \in (0, 1)$ and let $G \sim \nu$ and
$g(\xc) \doteq \xc + \frac{1}{\tau}\E(G - \xc)_+$.
The function $x \mapsto x_+$ is convex, so
for $\xc, \xc' \in \cs$ and $\alpha \in [0, 1]$,
\[
    \E(G - \alpha \xc - (1 - \alpha)\xc')_+ \leq \alpha\E(G - \xc)_+ + (1 - \alpha)\E(G - \xc')_+,
\]
which means $g(\alpha \xc + (1 - \alpha)\xc') \leq \alpha g(\xc) + (1 - \alpha)g(\xc')$, that is, $g$ is convex.
Moreover,
\[
    \frac{\mathrm{d}}{\mathrm{d}\xc}g = 1 - \frac{1}{\tau}\Prob(G \geq \xc),
\]
which means $\qf_\nu(1 - \tau)$ is a minimizer of $g$.
Finally, with $\xc^* = \qf_\nu(1 - \tau)$, we have that
\begin{align*}
    \min_{\xc} g(\xc) 
    &= g(\xc^*) \\
    &= \xc^* + \frac{1}{\tau}\E(G - \xc^*)_+ \\
    &= \xc^* + \frac{1}{\tau}\E\max\{G - \xc^*, 0\} \\
    &= \frac{1}{\tau}\xc^* -\frac{1 - \tau}{\tau}\xc^* + \frac{1}{\tau}\E\max\{G - \xc^*, 0\} \\
    &= -\frac{1 - \tau}{\tau}\xc^* + \frac{1}{\tau}\E\max\{G, \xc^*\} \\
    &= -\frac{1 - \tau}{\tau}\xc^* + \frac{1}{\tau}\int_0^1\max\{\qf_\nu(t), \xc^*\}\mathrm{d}t \\
    &= -\frac{1 - \tau}{\tau}\xc^* + \frac{1 - \tau}{\tau}\xc^* + \frac{1}{\tau}\int_{1 - \tau}^{1}\qf_\nu(t)\mathrm{d}t \\
    &= \ocvar(\nu, \tau).
    \tag*{\jmlrqedhere}
\end{align*}
\end{proof}

\begin{proposition}
\label{prop:ocvar-objective-1-lipschitz}
For all $s \in \states$, the function $\xc \mapsto -\xc + \frac{1}{\tau}\sup_{\pi \in \historybased} \E (\xc + G^\pi(s, \xc))_+$ is $1$-Lipschitz.
\end{proposition}

\begin{proof}
This proof is essentially the proof of \cref{prop:cvar-objective-1-lipschitz} with $x_+$ instead of $x_-$.
Fix $s \in \states$ and let
\[
    g(\xc) \doteq -\xc + \frac{1}{\tau}\sup_{\pi \in \historybased} \E (\xc + G^\pi(s, \xc))_+.
\]

For $\varepsilon \geq 0$, we have that
\begin{align*}
    \sup_{\pi \in \historybased}\E(\xc + G^\pi(s, \xc))_+ 
    &\leq \sup_{\pi \in \historybased}\E(\xc + \varepsilon + G^\pi(s, \xc))_+ 
    \tag{$(x + \varepsilon)_+ \geq x_+$} \\
    &= \sup_{\pi \in \historybased}\E(\xc + \varepsilon + G^\pi(s, \xc + \varepsilon))_+,
\end{align*}
where the last line follows by noticing that the \stock{} augmentation does not change the supremum over history-based policies.
We can apply the same reasoning to see that
\[
    \sup_{\pi \in \historybased}\E(\xc - \varepsilon + G^\pi(s, \xc - \varepsilon))_+ \leq \sup_{\pi \in \historybased}\E(\xc - \varepsilon + \varepsilon + G^\pi(s, \xc - \varepsilon + \varepsilon))_+ = \sup_{\pi \in \historybased}\E(\xc + G^\pi(s, \xc))_+.
\]
Thus for every $\varepsilon \geq 0$
\begin{align*}
    g(\xc - \varepsilon) 
    &= -(\xc - \varepsilon) + \frac{1}{\tau}\sup_{\pi \in \historybased}\E(\xc - \varepsilon + G^\pi(s, \xc - \varepsilon))_+ \\
    &\leq -\xc  + \varepsilon + \frac{1}{\tau}\sup_{\pi \in \historybased}\E(\xc + G^\pi(s, \xc))_+ \\
    &= g(\xc) + \varepsilon,
\end{align*}
and
\begin{align*}
    g(\xc + \varepsilon) 
    &= -(\xc + \varepsilon) + \frac{1}{\tau}\sup_{\pi \in \historybased}\E(\xc + \varepsilon + G^\pi(s, \xc + \varepsilon))_+ \\
    &\geq -\xc  - \varepsilon + \frac{1}{\tau}\sup_{\pi \in \historybased}\E(\xc + G^\pi(s, \xc))_- \\
    &= g(\xc) - \varepsilon,
\end{align*}
That is:
\begin{equation}
    g(\xc - \varepsilon) - \varepsilon \leq g(\xc) \leq g(\xc + \varepsilon) + \varepsilon \label{eq:g-lipschitz-bound-ocvar}
\end{equation}

Thus, for $\xc, \xc' \in \reals{}$, letting $\cmax = \max\{\xc, \xc'\}$ and $\cmin = \min\{\xc, \xc'\}$, we have
\begin{align*}
    -(\cmax - \cmin) 
    &\leq g(\cmax) - g(\cmax - (\cmax - \cmin)) 
    \tag{\cref{eq:g-lipschitz-bound-ocvar} with $\varepsilon = \cmax - \cmin$} \\
    &=g(\cmax) - g(\cmin) \\
    &= g(\cmax) - g(\cmin + (\cmax - \cmin)) \\
    &\leq \cmax - \cmin
    \tag{\cref{eq:g-lipschitz-bound-ocvar} with $\varepsilon = \cmax - \cmin$},
\end{align*}
so
\[
    |g(\xc) - g(\xc')| = |g(\cmax) - g(\cmin)| \leq |\cmax - \cmin| = |\xc - \xc'|,
\]
which means $g$ is $1$-Lipschitz.
\end{proof}

\ocvartheorem*

\begin{proof}
Fix $\tau$, $s_0 \in \states$, and $f(x) = x_+$.
By \cref{lem:ocvar}, we have
\begin{align*}
    &\sup_{\pi \in \historybased, \xc_0 \in \cs}\ocvar(\eta^\pi(s_0, \xc_0), \tau) \\
    &= \sup_{\pi \in \historybased, \xc_0 \in \cs}\min_{\xc}\left(-\xc + \frac{1}{\tau}\E(\xc + G^\pi(s_0, \xc_0))_+ \right) \\
    &= \sup_{\pi \in \historybased}\min_{\xc}\left(-\xc + \frac{1}{\tau}\E(\xc + G^\pi(s_0, \xc))_+ \right).
\end{align*}
where in the last line we use the fact that the choice of $\xc_0$ is irrelevant since the supremum is over history-based policies.

For every $\varepsilon > 0$,
by using distributional DP (\cref{thm:value-iteration,thm:policy-iteration}), we can find a near-optimal policy for optimizing $\Uf$, that is, a $\pibar$ satisfying
\[
	\sup_{\pi \in \historybased}\Uf\eta^\pi - \Uf\eta^{\pibar} < \varepsilon.
\]
\begin{align*}
    &\sup_{\pi \in \historybased, \xc_0 \in \cs}\ocvar(\eta^\pi(s_0, \xc_0), \tau) \\
    &=\sup_{\pi \in \historybased}\min_{\xc_0}\left(-\xc_0 + \frac{1}{\tau}\E(\xc_0 + G^\pi(s_0, \xc_0))_+ \right) \\
    &\geq \min_{\xc_0}\left(-\xc_0 + \frac{1}{\tau}\E(\xc_0 + G^{\pibar}(s_0, \xc_0))_+ \right) \\
    &> \inf_{\xc_0}\left(-\xc_0 + \frac{1}{\tau}\sup_{\pi \in \historybased}\E(\xc_0 + G^\pi(s_0, \xc_0))_+ \right) - \varepsilon.
\end{align*}
Moreover,
\begin{align*}
    &\sup_{\pi \in \historybased, \xc_0 \in \cs}\ocvar(\eta^\pi(s_0, \xc_0), \tau) \\
    &=\sup_{\pi \in \historybased}\min_{\xc_0}\left(-\xc_0 + \frac{1}{\tau}\E(\xc_0 + G^\pi(s_0, \xc_0))_+ \right) \\
    &< \min_{\xc_0}\left(-\xc_0 + \frac{1}{\tau}\E(\xc_0 + G^{\pi'}(s_0, \xc_0))_+ \right) + \varepsilon \\
    &\leq \inf_{\xc_0}\left(-\xc_0 + \frac{1}{\tau}\sup_{\pi \in \historybased}\E(\xc_0 + G^\pi(s_0, \xc_0))_+ \right) + \varepsilon \\
\end{align*}
Since the above holds for all $\varepsilon > 0$, it means that
\[
    \sup_{\pi \in \historybased, \xc_0 \in \cs}\ocvar(\eta^\pi(s_0, \xc_0), \tau)
    = \inf_{\xc_0}\left(-\xc_0 + \frac{1}{\tau}\sup_{\pi \in \historybased}\E(\xc_0 + G^\pi(s_0, \xc_0))_+ \right).
\]

It only remains to show that for all $s_0 \in \states$ there exists $\xc_0^*$ that realizes the infimum over $\xc_0$.
Note that by \cref{ass:bounded-first-moment-rewards},
we have
\begin{equation}
    \sup_{\pi \in \historybased}\E(G^\pi(s_0, \xc)) < \infty.
    \label{eq:finite-optimal-value-ocvar}
\end{equation}
For all $s_0 \in \states$, we have
\begin{align*}
    &\lim_{\xc_0 \rightarrow -\infty} -\xc_0 + \frac{1}{\tau}\sup_{\pi \in \historybased}\E(\xc_0 + G^\pi(s_0, \xc_0))_+ \\
    &\leq \lim_{\xc_0 \rightarrow -\infty} -\xc_0 
    \tag{\cref{eq:finite-optimal-value-ocvar}}\\
    &= \infty.
\end{align*}
and
\begin{align*}
    &\lim_{\xc_0 \rightarrow \infty} -\xc_0 + \frac{1}{\tau}\sup_{\pi \in \historybased}\E(\xc_0 + G^\pi(s_0, \xc_0))_+ \\
    &=\lim_{\xc_0 \rightarrow \infty} \frac{1 - \tau}{\tau}\xc_0 + \sup_{\pi \in \historybased}\E(G^\pi(s_0, \xc_0)) - \E(\xc_0 + G^\pi(s_0, \xc_0))_- \\
    &\geq\lim_{\xc_0 \rightarrow \infty} \frac{1 - \tau}{\tau}\xc_0 + \sup_{\pi \in \historybased}\E(G^\pi(s_0, \xc))\\
    &= \infty.
    \tag{\cref{eq:finite-optimal-value-ocvar}}
\end{align*}

Therefore there exist $\cmin, \cmax \in \reals$ such that
\begin{align*}
    &\inf_{\xc_0} \left( -\xc_0 + \frac{1}{\tau}\sup_{\pi \in \historybased}\E(\xc_0 + G^\pi(s_0, \xc_0))_+ \right) \\
    &= \inf_{\cmin \leq \xc_0 \leq \cmax} \left( -\xc_0 + \frac{1}{\tau}\sup_{\pi \in \historybased}\E(\xc_0 + G^\pi(s_0, \xc_0))_+ \right).
\end{align*}
Moreover, \cref{prop:ocvar-objective-1-lipschitz} implies $\xc \mapsto -\xc + \frac{1}{\tau}\sup_{\pi \in \historybased}\E(\xc + G^\pi(s_0, \xc))_+$ is continuous.
Therefore the infimum over $\xc_0$ is attained at a minimizer $\xc^*_0 \in \reals$.
\end{proof}

\approximateocvartheorem*

\begin{proof}
The proof of this result is essentially the same as \cref{thm:cvar-approximate}, except that we use \cref{lem:ocvar,prop:ocvar-objective-1-lipschitz,thm:ocvar-c-star} instead of \cref{thm:cvar,prop:cvar-objective-1-lipschitz,thm:cvar-c-star}.
\end{proof}

\section{Proofs for \texorpdfstring{\cref{sec:reward-design}}{Section~\ref{sec:reward-design}}}
\label{app:reward-design}

The proof of the first statement in \cref{thm:reward-design} is relatively direct and self-contained; we show that from the designed reward we can construct a valid \stock{}-augmented RL objective, where the designed rewards $\Rtilde_{t+1}$ satisfy a bounded first moment condition similar to \cref{ass:bounded-first-moment-rewards}, and with a designed discount $\alpha < 1$ in the infinite-horizon discounted ($\gamma < 1$) case.

However, the second statement---that only expected utilities that are indifferent to $\gamma$ admit a reduction to a \stock{}-augmented RL objective---requires multiple supporting results from the theory of optimizing expected utilities.
This is because, when reducing a \stock{}-augmented RL objective to a \stock{}-augmented \problem{} objective, we need to make a statement about all the objectives $F_K$ whose optimization is equivalent to a \stock{}-augmented RL objective.
In our case, this is possible, thanks to the von-Neumann-Morgenstern theorem \citep{von2007theory} and the results from \citet{bowling2023settling}.

Without \stock{} augmentation, for each state $s \in \states$, the preference over policies induced by value can be mapped to a relation $\succeq$ on $(\crvs, \wass)$.
The von-Neumann-Morgenstern theorem (see \cref{thm:vnm} below) states that $\succeq$ is equivalent to an expected utility function iff $\succeq$ satisfies the von-Neumann-Morgenstern axioms (\cref{ax:completeness,ax:transitivity,ax:independence,ax:continuity} below).
Furthermore, any such expected utility function will be unique up to affine transformations.
This uniqueness is powerful, because it implies that an objective cannot be simultaneously equivalent to an expected utility and a non-expected utility.

\begin{axiom}[Completeness, adapted from \citealp{bowling2023settling}]
\label{ax:completeness}
For all $\nu, \nu' \in (\crvs, \wass)$, $\nu \succeq \nu'$ or $\nu' \succeq \nu$ (or both, if $\nu \simeq \nu'$).
\end{axiom}

\begin{axiom}[Transitivity, adapted from \citealp{bowling2023settling}]
\label{ax:transitivity}
For all $\nu, \nu', \nu'' \in (\crvs, \wass)$, if $\nu \succeq \nu'$ and $\nu' \succeq \nu''$, then $\nu \succeq \nu''$.
\end{axiom}

\begin{axiom}[Independence, adapted from \citealp{bowling2023settling}]
\label{ax:independence}
For all $\nu, \nu', \overline{\nu} \in (\crvs, \wass)$,
$\nu \succeq \nu'$ iff for all $p \in (0, 1)$ $p \nu + (1 - p)\overline{\nu} \succeq p \nu' + (1 - p)\overline{\nu}$.
\end{axiom}

\begin{axiom}[Continuity, adapted from \citealp{bowling2023settling}]
\label{ax:continuity}
For all $\nu, \nu', \overline{\nu} \in (\crvs, \wass)$,
if $\nu \succeq \overline{\nu} \succeq \nu'$
then there exists $p \in [0, 1]$ such that
$p\nu + (1 - p)\nu' \simeq \overline{\nu}$.
\end{axiom}

\begin{theorem}[von Neumann-Morgenstern Expected Utility Theorem]
\label{thm:vnm}
A preference relation $\succeq$ on $(\crvs, \wass)$ satisfies \cref{ax:completeness,ax:transitivity,ax:independence,ax:continuity} if and only if there exists an expected utility function $u : (\crvs, \wass) \rightarrow \reals$ such that
\begin{enumerate}
    \item for all $\nu, \nu' \in (\crvs, \wass)$, $\nu \succeq \nu' \iff u(\nu) \geq u(\nu')$, \label{item:vnm-equivalence}
    \item for all $\nu \in (\crvs, \wass)$, $u(\nu) = \E \left( u(\delta_G) \right)$ ($G \sim \nu)$. \label{item:vnm-linearity}
\end{enumerate}
Such $u$ is unique up to positive affine transformations.
\end{theorem}

The main result introduced by \citet{bowling2023settling} establishes that every Markovian reward function induces a value function that is equivalent to a preference $\succeq$ satisfying \cref{ax:completeness,ax:transitivity,ax:independence,ax:continuity} plus a fifth axiom called \emph{Temporal Discount Indifference}.
Their temporal discount indifference axiom allows the discount to be transition-dependent, but we are interested in making statements about RL objectives with a fixed discount, so we introduce an adaptation to this special case, which we refer to as \emph{Fixed Discount Indifference}.

\begin{axiom}[Fixed Discount Indifference]
\label{ax:discount-indifference}
There exists $\alpha \in (0, 1]$ such that
for all $\nu, \nu' \in (\crvs, \wass)$, with $G \sim \nu$ and $G' \sim \nu'$,
\[
    \frac{1}{1 + \alpha}\df(\gamma G) + 
    \frac{\alpha}{1 + \alpha}\nu' 
    \simeq 
    \frac{1}{1 + \alpha}\df(\gamma G') + 
    \frac{\alpha}{1 + \alpha}\nu.
\]
\end{axiom}

Surprisingly, for relations $\succeq$ that satisfy \cref{ax:completeness,ax:transitivity,ax:independence,ax:continuity} 
(and thus admit an equivalent expected utility $u$) we can show that  $\succeq$ satisfies \cref{ax:discount-indifference} iff $u$ is indifferent to $\gamma$ (cf.~\cref{def:indifference-to-gamma}).
We can prove this correspondence between the two properties (\cref{ax:discount-indifference,def:indifference-to-gamma}) by combining \cref{lem:expected-utility-conditions} \cref{lem:expected-utility-conditions:indifferent-to-gamma} and the following novel result.\footnote{Note how \Cref{eq:u-positively-homogeneous} in \cref{prop:positively-homogeneous} is the same condition as \Cref{eq:expected-utility-conditions:indifferent-to-gamma} in \cref{lem:expected-utility-conditions:indifferent-to-gamma} of \cref{lem:expected-utility-conditions}.}
\begin{proposition}
\label{prop:positively-homogeneous}
Let $\succeq$ be a relation over $(\crvs, \wass)$,
and let $u : (\crvs, \wass) \rightarrow \reals$ be an expected utility function satisfying \cref{thm:vnm} \cref{item:vnm-equivalence,item:vnm-linearity}.
\Cref{ax:discount-indifference} holds iff for all $\xc \in \cs$
\begin{equation}
    \alpha \cdot (u(\delta_{\xc}) - u(\delta_{0})) = u(\delta_{\gamma \xc}) - u(\delta_{0}). \label{eq:u-positively-homogeneous}
\end{equation}
\end{proposition}

\begin{proof}
Since $u$ is linear, for $\xc \in \cs$ we write $u(\xc) = u(\delta_{\xc})$.
We first prove the result under the assumption that $u(\delta_{0}) = 0$, in which case we want to show that $\alpha \cdot u(\delta_{\xc}) = u(\delta_{\gamma \xc})$.
\Cref{ax:discount-indifference} states that there exists $\alpha \in (0, 1]$ such that for all $\nu, \nu' \in (\crvs, \wass)$,
with $G \sim \nu$ and $G' \sim \nu'$,
\[
    \frac{1}{1 + \alpha}\df(\gamma G) + 
    \frac{\alpha}{1 + \alpha}\nu' 
    \simeq 
    \frac{1}{1 + \alpha}\df(\gamma G') + 
    \frac{\alpha}{1 + \alpha}\nu.
\]
Since $u$ is equivalent to the preference and linear, the above is equivalent to
\[
    \frac{1}{1 + \alpha}\E u(\gamma G) + 
    \frac{\alpha}{1 + \alpha}u(\nu') 
    = 
    \frac{1}{1 + \alpha}\E u(\gamma G') + 
    \frac{\alpha}{1 + \alpha}u(\nu).
\]
Thus, by rearranging the above, \cref{ax:discount-indifference} holds iff there exists $\alpha \in (0, 1]$ such that, for all $\nu, \nu' \in (\crvs, \wass)$,
\begin{equation}
    \E u(\gamma G) - \alpha \cdot u(\nu)  
    = 
    \E u(\gamma G') - \alpha \cdot u(\nu').
    \label{eq:discount-indifference-equivalent}
\end{equation}

\emph{\Cref{ax:discount-indifference} implies \Cref{eq:u-positively-homogeneous}.}
Using \Cref{eq:discount-indifference-equivalent} with $\nu = \delta_{\xc}$ and $\nu' = \delta_{0}$ gives
\[
    u(\gamma \xc) - \alpha \cdot u(\xc)  
    = 
    u(\delta_0) - \alpha \cdot u(\delta_0) = 0,
\]
which gives the result.

\emph{\Cref{eq:u-positively-homogeneous} implies \cref{ax:discount-indifference}.}
We have that for all $\xc , \xc' \in \cs$
\[
    u(\gamma \xc) - \alpha \cdot u(\xc)  
    = 0 = u(\gamma \xc') - \alpha \cdot u(\xc'),
\]
and since this holds ``pointwise'', it also holds in expectation (with random $\xC, \xC'$), so \Cref{eq:discount-indifference-equivalent} follows.

Let us now prove the general case, $u(\delta_{0}) \in \reals$.
Let $u'(\nu) \doteq u(\nu) - u(\delta_0)$.
We have already established that \cref{ax:discount-indifference} holds iff $\alpha \cdot u'(\delta_{\xc}) = u'(\delta_{\gamma \xc})$ for all $\xc \in \cs$, and expanding $u'$ in terms of $u$ gives \Cref{eq:u-positively-homogeneous}.
\end{proof}

We can now combine \cref{ax:completeness,ax:transitivity,ax:independence,ax:continuity,ax:discount-indifference,thm:vnm} into the core result for characterizing what objectives \stock{}-augmented RL can optimize---an analogue of the main result of \citet{bowling2023settling} (their Theorem 4.1).

As discussed earlier, in the standard case we use $\succeq$ to compare return distributions directly, so we can connect optimizing $\succeq$ to the RL problem by comparing return distributions of policies $\pi, \pi' \in \historybased$ at states $s \in \states$ as $\eta^\pi(s) \succeq \eta^{\pi'}(s)$.
Therefore, expected utilities that are equivalent to $\succeq$  are naturally $(\crvs, \wass) \rightarrow \reals$ functions.

With \stock{} augmentation, whether a return distribution $\nu$ is preferable to another $\nu'$ depends on the \stock{} $\xc$, and we compare distributions of policies $\pi, \pi' \in \historybased$ at \stock{}-augmented states $(s, \xc) \in \states \times \cs$ as $\df(\xc + G^\pi(s, \xc)) \succeq \df(\xc + G^{\pi'}(s, \xc))$.
We can see this as there being a different preference relation for each $\xc$, so we define one expected utility $(\crvs, \wass) \rightarrow \reals$ for each $\xc$, as is the case in \cref{thm:vnm-ccrl} below.

The main contribution of \cref{thm:vnm-ccrl} is that we can look at properties of a \stock{}-indexed expected utility, and make statements about the corresponding relation per \stock{} $\xc$, as we will discuss after presenting and proving \cref{thm:vnm-ccrl}.

\begin{theorem}
\label{thm:vnm-ccrl}
A preference relation $\succeq$ on $(\crvs, \wass)$ satisfies \cref{ax:completeness,ax:transitivity,ax:independence,ax:continuity,ax:discount-indifference} iff there exist \stock{}-indexed functions $\utilde_\xc : (\crvs, \wass) \rightarrow \reals$ (for all $\xc \in \cs$), a \stock{}-augmented reward function $\rtilde : \cs \times \cs \rightarrow \reals$ and $\alpha \in (0, 1]$ such that:
\begin{enumerate}
    \item for all $\xc, r', g \in \cs$: $\utilde_\xc(\delta_{r' + \gamma g}) = \rtilde(\xc, r') + \alpha \cdot \utilde_{\gamma^{-1}(\xc + r')}(\delta_g)$, \label{item:vnm-ccrl-bellman}
    \item for all $\xc \in \cs$ and $\nu, \nu' \in (\crvs, \wass)$: $\df(\xc + G) \succeq \df(\xc + G')$ ($G \sim \nu$, $G' \sim \nu'$) iff $\utilde_\xc(\nu) \geq \utilde_\xc(\nu')$, \label{item:vnm-ccrl-equivalent-value-function}
    \item for all $\xc \in \cs$ and $\nu \in (\crvs, \wass)$: $\utilde_\xc(\nu) =  \E \left( \utilde_\xc(\delta_{G}) \right)$ ($G \sim \nu$). \label{item:vnm-ccrl-linearity}
\end{enumerate}
\end{theorem}

\begin{proof}
This proof retraces the steps of the proof of Theorem 4.1 by \citet{bowling2023settling}.

\emph{\Cref{ax:completeness,ax:transitivity,ax:independence,ax:continuity,ax:discount-indifference} imply \cref{item:vnm-ccrl-bellman,item:vnm-ccrl-equivalent-value-function,item:vnm-ccrl-linearity}.}

From the von Neumann-Morgenstern theorem (\cref{thm:vnm}), we know that \cref{ax:completeness,ax:transitivity,ax:independence,ax:continuity}
imply the existence of a utility function $u : (\crvs, \wass) \rightarrow \reals$ that is equivalent to the preference (\cref{thm:vnm} \cref{item:vnm-equivalence}), linear (\cref{thm:vnm} \cref{item:vnm-linearity}) and unique up to positive affine transformations (\cref{thm:vnm}).

We define, for $\xc \in \cs$ and $\nu \in (\crvs, \wass)$, with $G \sim \nu$,
\begin{equation}
    \utilde_{\xc}(\nu) \doteq u(\df(\xc + G)) - u(\delta_{\xc}) \label{eq:utilde}
\end{equation}
and we will show that \cref{item:vnm-ccrl-bellman,item:vnm-ccrl-equivalent-value-function,item:vnm-ccrl-linearity} hold.
We also define the shorthand $f(\xc) \doteq u(\delta_{\xc})$, and note that, for all $\xc, g \in \cs$ we have $\utilde_{\xc}(\delta_{g}) = f(\xc + g) - f(\xc)$.

For \cref{item:vnm-ccrl-bellman}, we define the reward function:
\begin{equation}
    \rtilde(\xc, r') \doteq \alpha f\left(\gamma^{-1}(\xc + r')\right) - f(\xc) + (1 - \alpha) f(0). \label{eq:rtilde-proof}
\end{equation}
From \cref{prop:positively-homogeneous}, we get that for all $\xc, r', g \in \cs$
\[
    \alpha \left(f\left(\gamma^{-1}(\xc + r' + \gamma g)\right) - f(0)\right) = f(\xc + r + \gamma g) - f(0),
\]
which we can rearrange as
\begin{equation}
    f(\xc + r + \gamma g) = 
    \alpha f\left(\gamma^{-1}(\xc + r' + \gamma g)\right) + (1 - \alpha) f(0).
    \label{eq:homogeneity-for-item-vnm-ccrl-bellman}
\end{equation}
Thus, for all $\xc, r', g \in \cs$,
\begin{align*}
    \utilde_\xc(r' + \gamma g) 
    &= f(\xc + r' + \gamma g) - f(\xc) \\ 
    &= \alpha f\left(\gamma^{-1}(\xc + r' + \gamma g)\right) + (1 - \alpha) f(0) - f(\xc) 
    \tag{\cref{eq:homogeneity-for-item-vnm-ccrl-bellman}} \\ 
    &= \alpha \cdot \utilde_{\gamma^{-1}(\xc + r')}(g) + \alpha f\left(\gamma^{-1}(\xc + r')\right) + (1 - \alpha) f(0) - f(\xc) \\
    &= \alpha f\left(\gamma^{-1}(\xc + r')\right) - f(\xc) + (1 - \alpha) f(0) + \alpha \cdot \utilde_{\gamma^{-1}(\xc + r')}(g) 
    \tag{Rearranging}\\
    &= \rtilde(\xc, r') + \alpha \cdot \utilde_{\gamma^{-1}(\xc + r')}(g),
    \tag{\cref{eq:rtilde-proof}}
\end{align*}
which proves \cref{item:vnm-ccrl-bellman}.

\Cref{item:vnm-ccrl-equivalent-value-function} follows from the fact that the preference induced by $u$ is equivalent to $\succeq$, and $\utilde_\xc(\nu) = u(\df(\xc + G)) - u(\delta_\xc)$, so for all $\xc \in \cs$ and $\nu, \nu' \in (\crvs, \wass)$, we have 
\[
    \utilde_\xc(\nu) \geq \utilde_\xc(\nu') \iff u(\df(\xc + G)) \geq u(\df(\xc + G')) \iff \df(\xc + G) \succeq \df(\xc + G').
\]

For \cref{item:vnm-ccrl-linearity}, we proceed as follows.
For all $\xc \in \cs$ and $\nu \in (\crvs, \wass)$ (with $G \sim \nu$)
\begin{align*}
    \utilde_{\xc}(\nu) 
    &= u(\df(\xc + G)) - u(\delta_{\xc}) 
    \tag{\cref{eq:utilde}} \\
    &= \E\left(u(\delta_{\xc + G})\right) - u(\delta_{\xc})
    \tag{\cref{thm:vnm} \cref{item:vnm-linearity}} \\
    &= \E\left(\utilde_{\xc}(\delta_{G}) \right),
    \tag{\cref{eq:utilde}}
\end{align*}
which proves \cref{item:vnm-ccrl-linearity}.

\emph{\Cref{ax:completeness,ax:transitivity,ax:independence,ax:continuity,ax:discount-indifference} follow from \cref{item:vnm-ccrl-bellman,item:vnm-ccrl-equivalent-value-function,item:vnm-ccrl-linearity}.}
\Cref{item:vnm-equivalence,item:vnm-ccrl-linearity} with $\xc = 0$ imply \cref{item:vnm-equivalence,item:vnm-linearity} of \cref{thm:vnm} with $u = \utilde_0$, which means $\succeq$ satisfies \cref{ax:completeness,ax:transitivity,ax:independence,ax:continuity}.

It remains only to show that $\succeq$ satisfies \cref{ax:discount-indifference}.
By rearranging \cref{item:vnm-ccrl-bellman}, we get that, for all $g \in \cs$,
\begin{align*}
    \rtilde(0, 0) 
    &= \utilde_0(\delta_{\gamma g}) - \alpha \cdot \utilde_0(\delta_g).
\end{align*}
In particular, by taking $g = 0$, we get that $\rtilde(0, 0) = (1 - \alpha) \utilde_0(\delta_0)$.
Thus, for any $g \in \cs$, we have
\[
    \utilde_0(\delta_{\gamma \xc}) - \alpha \cdot \utilde_0(\delta_\xc) = (1 - \alpha) \cdot \utilde_0(\delta_0),
\]
and, by rearranging,
\[
    \alpha \cdot (\utilde_0(\delta_\xc) - \utilde_0(\delta_0)) = \utilde_0(\delta_{\gamma \xc}) - \utilde_0(\delta_0),
\]
so we can satisfy \Cref{eq:u-positively-homogeneous} with $u = \utilde_0$, and, by \cref{prop:positively-homogeneous}, $\succeq$ satisfies \cref{ax:discount-indifference}.
\end{proof}

This is how we will use \cref{thm:vnm-ccrl} to prove the second statement in \cref{thm:reward-design}:
We will show that value in \stock{}-augmented RL is, in effect, a \stock{}-indexed expected utility, so the \stock{}-indexed corresponding relations satisfy \cref{ax:completeness,ax:transitivity,ax:independence,ax:continuity,ax:discount-indifference}.
If this \stock{}-augmented RL objective is equivalent to a \stock{}-augmented \problem{} objective $F_K$, then (we show) $K$ must be equivalent to the \stock{}-indexed utility corresponding to value.
Then we combine \cref{thm:vnm-ccrl,prop:positively-homogeneous,lem:expected-utility-conditions} to show that $K$ must be both an expected utility and indifferent to $\gamma$.

We are now ready to present the proof of \cref{thm:reward-design}.

\rewarddesigntheorem*

\begin{proof}
\emph{Reduction from a \stock{}-augmented \problem{} objective to a \stock{}-augmented RL objective.}
The \stock{}-augmented RL objective we want to reduce to is an expected return where the (designed) rewards have bounded first moment ($\Rtilde_{t+1}$ satisfying \cref{eq:rtilde-bounded-first-moment}), the discount is $\alpha \in (0, 1]$ (where $\gamma < 1 \Rightarrow \alpha < 1$), and policies $\pi \in \historybased$ have value function
\[
    \Vtilde^\pi(s, \xc) \doteq \E\left( \sum_{t=0}^\infty \alpha^t \Rtilde_{t+1} \right).
\]
We will show that, under the given conditions, for all $\pi \in \historybased$ and $(s, \xc) \in \states \times \cs$,
\begin{equation}
    \Vtilde^\pi(s, \xc) = (\Uf\eta^\pi)(s, \xc) - f(\xc) = \E f(\xc + G^\pi(s, \xc)) - f(\xc),
    \label{eq:v-pi-to-uf}
\end{equation}
with $G(s, \xc) \sim \eta^\pi(s, \xc)$.
If this is the case, then both \stock{}-augmented objectives induce the same preference over policies.

Let us first establish that, under the given conditions, the designed rewards have bounded first moment.
In the finite-horizon case we have imposed \Cref{eq:rtilde-bounded-first-moment} as a condition directly.
In the discounted case, $f$ is assumed to be $L$-Lipschitz for some $L$, so:
\begin{align*}
    |\Rtilde_{t+1}| 
    &= |\alpha f(\xC_{t+1}) - f(\xC_t) + (1 - \alpha)f(0)| \\
    &= |f(\gamma \xC_{t+1}) - f(\xC_t)|
    \tag{\cref{eq:reward-design-condition}} \\
    &= |f(\xC_t + R_{t+1}) - f(\xC_t)|
    \tag{$\xC_{t+1} = \gamma^{-1}(\xC_t + R_{t+1})$} \\
    &= L \cdot \| R_{t+1} \|_1,
    \tag{$f$ $L$-Lipschitz}
\end{align*}
and, by \cref{ass:bounded-first-moment-rewards},
\[
    \sup_{s, \xc, a \in \states \times \cs \times \actions}\E \left( \left. | \Rtilde_{t+1} | \, \right| S_t = s, \xC_t = \xc, A_t = a \right) \leq \sup_{s, a \in \states \times \actions}\E \left( \left. \| R_{t+1} \|_1 \, \right| S_t = s, A_t = a \right)  < \infty.
\]

Next, we establish that $\gamma < 1 \Rightarrow \alpha < 1$ (that is, the $\alpha$-discounting is valid for the infinite-horizon discounted case).
By induction on \Cref{eq:reward-design-condition}, we get that, for all $n \in \naturals_0$ and $\xc \in \cs$, that $f(\gamma^n \xc) - f(0) = \alpha^n (f(\xc) - f(0))$, which we can rearrange as
\begin{equation}
    f(\gamma^n \xc) = \alpha^n f(\xc) + (1 - \alpha^n) f(0).
    \label{eq:reward-design-condition-induction}
\end{equation}
In particular, for all $\xc \in \cs$,
\[
    \liminf_{n \rightarrow \infty} f(\gamma^n \xc) = \liminf_{n \rightarrow \infty}  \alpha^n f(\xc) + (1 - \alpha^n) f(0).
\]
If $\gamma < 1$, the left-hand side is zero, so the right-hand side must be zero, thus $\alpha < 1$.

Finally, we prove \cref{eq:v-pi-to-uf}.
For all $\pi \in \historybased$ and $(s, \xc) \in \states \times \cs$, with $(S_0, \xC_0) = (s, \xc)$ with probability one, we have
\begin{align*}
    \Vtilde^\pi(s, \xc) 
    &= \E\left( \left. \sum_{t=0}^\infty \alpha^t \Rtilde_{t+1} \right| \xC_0 = \xc \right) \\
    &= \lim_{n \rightarrow \infty} \E\left( \left. \sum_{t=0}^{n-1} \alpha^t \Rtilde_{t+1} \right| \xC_0 = \xc \right) \\
    &= \lim_{n \rightarrow \infty} \E\left( \left. \sum_{t=0}^{n-1} \alpha^{t+1} f(\xC_{t+1}) - \alpha^t f(\xC_t) + \alpha^t (1 - \alpha)f(0) \right| \xC_0 = \xc \right) 
    \tag{\cref{eq:rtilde-reward-design}}\\
    &= \lim_{n \rightarrow \infty} \E\left( \left. \sum_{t=0}^{n-1} \alpha^{t+1} f(\xC_{t+1}) - \alpha^t f(\xC_t) \right| \xC_0 = \xc \right) + (1 - \alpha^n)f(0) \\
    &= \lim_{n \rightarrow \infty} \E\left( \left. \alpha^n f(\xC_n) \right| \xC_0 = \xc \right) - f(\xc) + (1 - \alpha^n)f(0) 
    \tag{Telescoping, $\xC_0 = \xc$} \\
    &= \lim_{n \rightarrow \infty} \E\left( \left. \alpha^n f(\xC_n) + (1 - \alpha^n)f(0) \right| \xC_0 = \xc \right) - f(\xc) \\
    &= \lim_{n \rightarrow \infty} \E\left( \left. f\left( \gamma^n \xC_n \right) \right| \xC_0 = \xc \right) - f(\xc)
    \tag{\cref{eq:reward-design-condition-induction}} \\
    &= \lim_{n \rightarrow \infty} \E\left( \left. f\left( \xC_0 + \sum_{t=0}^{n-1} \gamma^t R_{t+1} \right) \right| \xC_0 = \xc \right) - f(\xc) \\
    &= \E\left( \left. f\left( \xC_0 + \sum_{t=0}^{\infty} \gamma^t R_{t+1} \right) \right| \xC_0 = \xc \right) - f(\xc) 
    \tag{$f$ Lipschitz or finite horizon} \\
    &= (\Uf \eta^\pi)(s, \xc) - f(\xc),
\end{align*}
which proves \cref{eq:v-pi-to-uf} and concludes the proof of the first statement.

\emph{Impossible reduction via reward design when the objective $F_K$ is not an expected utility or not indifferent to $\gamma$.}
We will show the contrapositive: If the reduction is possible, then $F_K$ is an expected utility and indifferent to $\gamma$.

Assume we can reduce it to an equivalent \stock{}-augmented RL objective with a suitably designed reward function.
It is important to stress that the reduction must be valid regardless of the underlying MDP transition or reward kernels, as long as \cref{ass:bounded-first-moment-rewards} is satisfied.

Let us define the ``\stock{}-indexed value functional'' $\vtilde_\xc : (\crvs, \wass) \rightarrow \reals$ as follows:
For a Markov chain $\xC_0 \rightarrow R_1 \rightarrow R_2 \rightarrow \ldots$ all taking values in $\cs$ (and satisfying \cref{ass:bounded-first-moment-rewards}) with $G_0 \doteq \sum_{t=0}^\infty \gamma^t R_{t+1}$ and $\xC_0 = \xc$,
we let
\[
    \vtilde_{\xc}(\df(G_0)) \doteq \E\left( \sum_{t=0}^\infty \alpha^t \cdot \rtilde(\xC_t, R_{t+1}) \right),
\]
where $\rtilde: \cs \times \cs \rightarrow \reals$ is the designed (Markov) reward.

The value functional and the designed reward function do not directly depend on states and actions.
This is natural, as trajectories with the same $\xC_0 \rightarrow R_1 \rightarrow R_2 \rightarrow \ldots$, regardless of the underlying $S_0, A_0, S_1, \ldots$, must be equivalent in terms of the objective (either \problem{} or RL).

The reduction requires the same augmented state space $\states \times \cs$ to be used for both \problem{} and RL objectives, so, for all $(\xc, \nu), (\xc', \nu') \in \cs \times (\crvs, \wass)$ we have $K\df(\xc + G) \geq K\df(\xc' + G') \iff \vtilde_{\xc}(\nu) \geq \vtilde_{\xc'}(\nu')$, with $G \sim \nu$ and $G' \sim \nu'$.

We can now apply \cref{thm:vnm-ccrl} to conclude that the relation induced by $K$ on $(\crvs, \wass)$ must satisfy \cref{ax:completeness,ax:transitivity,ax:independence,ax:continuity,ax:discount-indifference}.
To do so, we must prove that \cref{item:vnm-ccrl-bellman,item:vnm-ccrl-equivalent-value-function,item:vnm-ccrl-linearity} hold for $\vtilde_\xc$ and $\rtilde$.

For \cref{item:vnm-ccrl-bellman}, consider $\xC_0 = \xc$, $R_1 = r_1$ and so forth, with probability one, such that $g_1 \doteq \sum_{t=1}^\infty \gamma^t r_{t+2} < \infty$.
Then
\begin{align*}
    \vtilde_\xc(\delta_{r_1 + \gamma g_1})
    &= \vtilde_\xc(\df(G_0)) \\
    &= \E\left( \sum_{t=0}^\infty \alpha^t \cdot \rtilde(\xC_t, R_{t+1}) \right) \\
    &= \E\left( \rtilde(\xC_0, R_1) + \alpha \cdot \E\left( \left. \sum_{t=0}^\infty \alpha^t \cdot \rtilde(\xC_{t+1}, R_{t+2}) \right| \xC_1 \right) \right) \\
    &= \E\left( \rtilde(\xC_0, R_1) + \alpha \cdot \vtilde_{\xC_1}(\df(G_1)) \right) \\
    &= \rtilde(\xc, r_1) + \alpha \cdot \vtilde_{\gamma^{-1}(\xc + r_1)}(\delta_{g_1}),
\end{align*}
which gives us \cref{item:vnm-ccrl-bellman}.

\Cref{item:vnm-ccrl-equivalent-value-function} follows by assumption that the reduction is possible.

\Cref{item:vnm-ccrl-linearity} can be proved as follows: For all $\xc \in \cs$, with $\xC_0 = \xc$ and $\xC_0 \rightarrow R_1 \rightarrow R_2 \rightarrow \ldots$ satisfying \cref{ass:bounded-first-moment-rewards}:
\begin{align*}
    \vtilde_{\xc}(\df(G_0)) 
    &= \E\left( \sum_{t=0}^\infty \alpha^t \cdot \rtilde(\xC_t, R_{t+1}) \right) \\
    &= \E\left( \E\left( \left. \sum_{t=0}^\infty \alpha^t \cdot \rtilde(\xC_t, R_{t+1}) \right| \xC_0, R_1, R_2, \ldots \right) \right) \\
    &= \E \left( \vtilde_\xc(\delta_{G_0}) \right).
\end{align*}

Hence, by \cref{thm:vnm-ccrl}, the relation induced by $K$ on $(\crvs, \wass)$ satisfies \cref{ax:completeness,ax:transitivity,ax:independence,ax:continuity,ax:discount-indifference}, which implies that $K$ is an expected utility.

We know from \cref{thm:vnm} that there exist $a > 0$ and $b \in \reals$ such that, for all $\xc, g \in \cs$, we have $a K\delta_{\xc + g} + b = \vtilde_\xc(\delta_g)$.
So define $f(\xc) \doteq a K\delta_\xc + b$.
Then, for all $\xc \in \cs$,
\begin{align*}
    a \cdot K\delta_{\gamma \xc} + b
    &= f(\gamma \xc) \\
    &= \vtilde_0(\delta_{\gamma \xc}) \\
    &= \rtilde(0, 0) + \alpha \cdot \vtilde_0(\delta_\xc) \\
    &= \rtilde(0, 0) + \alpha f(\xc). 
\end{align*}
In particular, for $g = 0$, the above implies that $\rtilde(0, 0) = (1 - \alpha) f(0)$,
so, for all $\xc \in \cs$,
\[
    f(\gamma \xc) = \alpha f(\xc) + (1 - \alpha) f(0).
\]
The assumption that the reduction is possible ensures that $\alpha \in (0, 1]$ and $\gamma < 1 \Rightarrow \alpha < 1$, so, by \cref{lem:expected-utility-conditions} \cref{lem:expected-utility-conditions:indifferent-to-gamma}, $K$ is indifferent to $\gamma$.
\end{proof}

\section{Proofs for \texorpdfstring{\cref{sec:beyond-expected-utilities}}{Section~\ref{sec:beyond-expected-utilities}}}
\label{app:beyond-expected-utilities}

Our characterization builds on and extends the results by \citet{marthe2024beyond}, which characterized objective functionals that distributional DP can optimize in the finite-horizon undiscounted setting, \emph{without} \stock{} augmentation.
Our proof strategy is to connect indifference to mixtures, indifference to $\gamma$ and Lipschitz continuity to the von Neumann-Morgenstern axioms (from \cref{app:reward-design}), so that we can apply the powerful von Neumann-Morgenstern theorem (or show that it cannot apply, in the case of the non-expected-utility objective functional that distributional DP can optimize).

The following results connect Lipschitz continuity and indifference to mixtures to the von Neumann-Morgenstern independence axiom (\cref{ax:independence}).

\begin{proposition}[If $K$ Lipschitz then \cref{ax:independence}'s $\Leftarrow$ is satisfied.]
\label{prop:lipschitz-implies-independence-if}
If $K$ is Lipschitz, the following holds:
For every $\nu, \nu', \overline{\nu} \in (\crvs, \wass)$ if for all
$p \in (0, 1)$ we have
\[
    K((1 - p)\nu + p\overline{\nu}) \geq K((1 - p)\nu' + p\overline{\nu}),
\]
then
\[
    K\nu \geq K\nu'.
\]
\end{proposition}

\begin{proof}
Fix $\nu, \nu', \overline{\nu} \in (\crvs, \wass)$ and assume that for all $p \in (0, 1)$ we have
\[
    K((1 - p)\nu + p\overline{\nu}) \geq K((1 - p)\nu' + p\overline{\nu}).
\]
Define the sequences of distributions 
\begin{align*}
    \nu_n &\doteq \frac{1}{n}\overline{\nu} + \left(1 - \frac{1}{n}\right)\nu \\
    \nu'_n &\doteq \frac{1}{n}\overline{\nu} + \left(1 - \frac{1}{n}\right)\nu'.
\end{align*}
We have that $\nu_n$ converges to $\nu$ in $\wass$ as $n\rightarrow \infty$ (and $\nu'_n$ to $\nu'$).
Because $K$ is Lipschitz, and by assumption $K\nu_n - K\nu'_n \geq 0$ for all $n \in \naturals$, we get
\[
    K\nu - K\nu' = 
    \lim_{n \rightarrow \infty} K\nu_n - K\nu'_n \geq 0.
    \tag*{\jmlrqedhere}
\]
\end{proof}

\begin{proposition}[If $K$ is indifferent to mixtures, then \cref{ax:independence}'s $\Rightarrow$ is satisfied.]
\label{prop:indifferent-to-mixtures-implies-independence-only-if}
If $K$ is indifferent to mixtures, then the following holds:
For every $\nu, \nu', \overline{\nu} \in (\crvs, \wass)$
if
\[
     K\nu \geq K\nu',
\]
then for all $p \in (0, 1)$ we have
\[
    K((1 - p)\nu + p\overline{\nu}) \geq K((1 - p)\nu' + p\overline{\nu}).
\]
\end{proposition}

\begin{proof}
\Cref{def:indifference-to-mixtures} with $\nu_1, \nu_2, \nu'_1, \nu'_2$ such that $K\nu_1 \geq K\nu'_1$ and $\nu'_2 = \nu_2$, gives us that for all $p \in (0, 1)$
\[
    K\nu \geq K\nu' \Rightarrow K((1 - p)\nu + p\overline{\nu}) \geq K((1 - p)\nu' + p\overline{\nu}).
    \tag*{\jmlrqedhere}
\]
\end{proof}

Next, we apply the von Neumann-Morgenstern theorem to characterize objective functionals that distributional DP can optimize in the infinite-horizon discounted case.

\thmlipschitzcharacterization*

\begin{proof}
Consider the relation $\succeq$ over $(\crvs, \wass)$ defined by
$\nu \succeq \nu' \iff K\nu \geq K\nu'$.
It is easy to see that $\succeq$ satisfies completeness and transitivity (\cref{ax:completeness,ax:transitivity} in \cref{app:reward-design}).
$K$ Lipschitz implies that $\succeq$ also satisfies continuity (\cref{ax:continuity}).
$K$ Lipschitz and $K$ indifferent to mixtures implies that $K$ satisfies \cref{ax:independence}   (\cref{prop:lipschitz-implies-independence-if,prop:indifferent-to-mixtures-implies-independence-only-if}).

Then by the von Neumann-Morgenstern theorem (\cref{thm:vnm}) there exists an expected utility function $u : (\crvs, \wass) \rightarrow \reals$ satisfying \cref{item:vnm-equivalence,item:vnm-linearity}, and it is unique up to affine transformations.
By \cref{item:vnm-equivalence}, for all $\nu, \nu' \in (\crvs, \wass)$,
with $G \sim \nu$ and $G' \sim \nu'$, we have $\nu \succeq \nu' \iff u(\nu) \geq u(\nu)$, and thus $K\nu \geq K\nu' \iff u(\nu) \geq u(\nu)$.
Moreover, by \cref{thm:vnm}, we know $u$ is unique up to positive affine transformations, so there exist $a > 0$ and $b \in \reals$ such that $K\nu = a \cdot u(\nu) + b$ for all $\nu \in (\cs, \wass)$.
Without loss of generality we can consider $u$ in the rest of this proof such that $a = 1$ and $b = 0$.
Since $u$ is linear, we know there exists $f : \cs \rightarrow \reals$ such that $u(\nu) = \E f(G)$ ($G \sim \nu$) for all $\nu \in (\cs, \wass)$.
The statement that $f$ is Lipschitz follows from \cref{lem:expected-utility-conditions}.
\end{proof}

\propnonexpectedutility*

\begin{proof}
\emph{$K$ is indifferent to mixtures.}
Consider $\eta, \eta' \in (\crvs^{\states \times \cs}, \overline{\wass})$ such that,
for all $(s, \xc) \in \states \times \cs$,
\begin{equation}
    K\eta(s, \xc) \geq K\eta'(s, \xc),
    \label{eq:non-expected-utility-assumption}
\end{equation}
and let $(S, \xC)$ be a random variable taking values in $\states \times \cs$, $\nu \doteq \df(G(S, \xC))$ and $\nu' \doteq \df(G'(S, \xC))$.

\Cref{eq:non-expected-utility-assumption} implies that $\{ \nu'(S, \xC)([0, \infty)) = 1 \} \subseteq \{ \nu(S, \xC)([0, \infty)) = 1 \}$, which in turn implies that
\[
    \Ind( \nu'(S, \xC)([0, \infty)) = 1 ) \leq \Ind( \nu(S, \xC)([0, \infty)) = 1 ),
\]
which proves the result.

\emph{$K$ is indifferent to $\gamma$.}
Given $\nu, \nu' \in (\crvs, \wass)$ and letting $G \sim \nu$ and $G' \sim \nu'$, note that $\nu([0, \infty)) = 1 \iff \df(\gamma G)([0, \infty)) = 1$ (and similarly for $\nu'$ and $G'$), so $K(\gamma G) = K\nu$ and $K(\gamma G') = K\nu'$, which means $K\nu \geq K\nu'$ implies $K(\gamma G) \geq K(\gamma G')$.

\emph{$F_K$ is not an expected utility.}
It suffices to show that $K$ violates at least one of the von Neumann-Morgenstern axioms, otherwise \cref{thm:vnm} applies and $F_K$ is an expected utility.
$K$ invariably satisfies completeness and transitivity (\cref{ax:completeness,ax:transitivity}), however it violates independence and continuity (\cref{ax:continuity,ax:independence}; cf.~\citealp{juan2020neumann}, p.~15).
\end{proof}

\section{Implementation details}
\label{app:implementation-details}

\subsection{\texorpdfstring{D$\eta$N}{din}}
\label{app:din:implementation-details}

The architecture diagram for D$\eta$N's \stock{}-augmented return distribution estimator is given in \cref{fig:din-architecture}.
The training and network parameters were set per domain (see \cref{app:gridworld:implementation-details,app:atari:implementation-details}).
The target parameters $\overline{\theta}$ were updated via exponential moving average updates, as done by \citet{schwarzer2023bigger}, and differently from the periodic updates used by \citet{mnih2015human}.
Our intent was to have smoother quantile regression targets, rather than sudden changes introduced by the periodic update.
The target network is updated as an exponential moving average with step size $\alpha$ as $\overline{\theta} \leftarrow (1 - \alpha)\overline{\theta} + \alpha \theta$.
D$\eta$N uses the target network parameters $\overline{\theta}$ for both training and evaluation~\citep[similar to][]{abdolmaleki2018maximum}.
Our intent was to slower-changing behavior and quantile regression targets.

As in DQN~\citep{mnih2015human} and QR-DQN~\citep{dabney2018distributional}, the action selection used by D$\eta$N during data collection is $\varepsilon$-greedy.
For greedy policy selection during both data generation (\cref{eq:din-action}) and learning (\cref{eq:din-quantile-regression-target}),
given a return distribution function $\xi : \states \times \cs \times \actions \rightarrow \crvs$,
D$\eta$N selects the greedy policy $\pibar \in \stationary$ satisfying
\[
    \Uf(M_f\xi)(s, \xc) = \E f(\xc + G(s, \xc, A))
    \tag{$A \sim \pibar(s, \xc)$, $G(s, \xc, a) \sim \xi(s, \xc, a)$}
\]
and, for all $(s, \xc) \in \states \times \cs$ and $a \in \actions$,
\[
    \pibar(a | s, \xc) > 0 \Rightarrow \pibar(a | s, \xc) = \max_{a'} \pibar(a' | s, \xc).
\]
We chose this because ties may happen often in \problem{}.
This is not the case in standard deep RL with DQN, and we rarely need to resort to tie-breaking, because action-value estimates are often noisy.
However, the choice of $\Uf$ may introduce ties in practice. 
For example, when maximizing the risk-averse $\tau$-CVaR, we have $f(x) = x_-$, which can introduce ties among maximizing actions.

With vector-valued returns, D$\eta$N maintains estimates of the quantiles each individual return coordinate, rather than an estimate of the joint distribution of the vector-valued return.
This means we cannot optimize all expected utilities over vector-valued returns, but only the ones with the form:
\[
    f(x) = \sum_i f_i(x_i).
\]
We believe this is acceptable for a proof-of-concept algorithm, and that future work will address this limitation based on results for multivariate distributional RL \citep{zhang2021distributional,wiltzer2024foundations}.

For the quantile regression loss, the greedy policy $\pibar$ breaks ties via uniform random action selection, but to avoid having to sample multiple actions from $\pibar$ we use the policy directly.
For a transition $(s, \xc), a, r', (s', \xc')$, the loss estimate is:
\[
    \frac{1}{n^2}\sum_{i, j \in \{1, \ldots, n\}}\sum_{a' \in \actions} \pibar(a' | s, \xc) \ell(r' + \gamma \xi_{\overline{\theta}}(s', a', \xc')_j - \xi_{\theta}(s, a, \xc)_i, \tau_i),
\]
where $\ell$ is the quantile regression loss~\citep{dabney2018distributional}
\[
    \ell(x, \tau) \doteq \left|\Ind(x > 0) - \tau\right| \cdot |x|,
\]
and the quantiles are the bin centers of an $n$-bin discretization of $[0, 1]$, that is, for $i \in \{1, \ldots, n\}$ we have $\tau_i \doteq \frac{2i-1}{2n}$.
As in DQN~\citep{mnih2015human} and QR-DQN~\citep{dabney2018distributional}, 
we explicitly use $\delta_0$ as the return distribution of the terminal state.

\subsection{Gridworld}
\label{app:gridworld:implementation-details}

In these experiments we trained D$\eta$N on an Nvidia V100 GPU.
For simplicity, D$\eta$N did not use a replay in these experiments.
Instead, it alternated generating a minibatch of transitions by having the agent interact with the environment, and then updating the network with the generated minibatch (the ``learner update'').
The transitions were generated in episodic fashion, with the agent starting at $\sinit$ and acting in the environment until the end of the episode.
The episode ended when the agent reached a terminating cell, or when it was interrupted on the $16$-th step.
Upon interruption, $s'$ was not treated as terminal.
Each minibatch consisted of $64$ trajectories of length $16$, and each transition had the form $(s_k, \xc_k), a_k, r'_k, (s'_k, \xc'_k)$.
If a termination or interruption happened at the $k$-th step in a trajectory, the next transition would start from the initial state, in which case $s'_k \neq s_{k+1}$ ($s'_k = s_{k+1}$ held otherwise).

\Cref{tab:gridworld:training,tab:gridworld:model} contain additional implementation details.
For training, we have used the Adam optimizer~\citep{kingma2014adam} with defaults from the Optax library~\citep{deepmind2020jax} unless otherwise stated.
\begin{table}[tb]
    \centering
    \begin{tabularx}{\textwidth}{Xc}
    \toprule
    Parameter & Value \\
    \midrule
    Batch size & $64$ \\
    Trajectory length & $16$ \\
    Training duration (environment steps) & $\approx 2M$ \\
    Training duration (learner updates) & $2K$ \\
    Adam optimizer learning rate & $10^{-4}$ \\
    Target network exponential moving average step size ($\alpha$) & $10^{-2}$ \\
    Discount ($\gamma$) & $0.997$ \\
    $\varepsilon$-greedy parameter & $0.1$ \\
    Interval for sampling $\xc_0$ & $[-10, 10)$ \\
    \bottomrule
    \end{tabularx}
    \caption{Training parameters for D$\eta$N in the gridworld experiments.
    \label{tab:gridworld:training}}
\end{table}
\begin{table}[tb]
    \centering
    \begin{tabularx}{\textwidth}{lXc}
    \toprule
    Component & Parameter & Value \\
    \midrule
    Vision (ConvNet) & & \\
    \midrule
        & Output channels (per layer) & $(32, 64, 64)$ \\
        & Kernel sizes (per layer) & $((8, 8), (4, 4), (3, 3))$ \\
        & Strides (all layers) & $(1, 1)$ \\
        & Padding & \texttt{SAME} \\
    \midrule
    Linear & & \\
    \midrule
    & Output size & $512$ \\
    \midrule
    MLP & & \\
    \midrule
        & Number of quantiles (per action) & $128$ \\
        & Hidden layer size & $512$ \\
    \bottomrule
    \end{tabularx}
    \caption{Neural network parameters for D$\eta$N's return distribution estimator $\xi_{\theta}$ in the gridworld experiments.
    See \cref{fig:din-architecture} for reference.
    \label{tab:gridworld:model}}
\end{table}

During evaluation, D$\eta$N followed greedy policies ($\varepsilon = 0$ for the $\varepsilon$-greedy exploration).
For the $\tau$-CVaR experiments (\cref{sec:gridworld:risk-averse,sec:gridworld:risk-seeking}), we selected $\xc^*_0$ based on \cref{thm:cvar-approximate,thm:ocvar-approximate}, with a grid search of $256$ equally spaced points on the interval $[-10, 10]$ (with points on the interval limits).

The vision network in the gridworld experiments is a ConvNet~\citep{lecun2015deep} following the implementation used by \citet{mnih2015human}.
Convolutional layers used ReLU activations~\citep{nair2010rectified}, as did the MLP hidden layer.
The ``Linear'' components in \cref{fig:din-architecture} did not use an activation function on the outputs (with the exception of the explicit ReLU activation shown in the diagrams).
The outputs of the ConvNet were flattened before being input to the ``Linear'' component.

\subsection{Atari}
\label{app:atari:implementation-details}

In these experiments we trained D$\eta$N in a distributed actor-learner setup~\citep{horgan2018distributed} using TPUv3 actors and learners.
The data was generated in episodic fashion (with multiple asynchronous actors).
The episode duration was set to $25\mathrm{s}$, at $15\mathrm{Hz}$ and $4$ frames per environment step due to action repeats~\citep{mnih2015human}.
The Atari benchmark typically has sticky actions~\citep{machado2018revisiting}, but we disabled them for these experiments, to have deterministic returns.
D$\eta$N, similar to DQN~\citep{mnih2015human} and QR-DQN~\citep{dabney2018distributional}, observes $84 \times 84$ grayscale Atari frames with frame stacking of $4$.

D$\eta$N was trained with a $3:7$ mixture of online and replay data in each learner update.
Each minibatch consisted of $144$ sampled trajectories (sequences of subsequent transitions) of length $19$ (the minibatch was distributed across multiple learners, and updates were combined before being applied).
The data generated in the actors was added simultaneously to a queue (for the online data stream) and to the replay (for the replay data stream).
The replay was not prioritized, and we edited the \stock{}s in each minibatch as explained in \cref{sec:atari}.

\Cref{tab:atari:training,tab:atari:model} contain additional implementation details.
For training, we have used the Adam optimizer~\citep{kingma2014adam} with defaults from the Optax library~\citep{deepmind2020jax} unless otherwise stated, as well as gradient norm clipping and weight decay.
\begin{table}[tb]
    \centering
    \begin{tabularx}{\textwidth}{Xc}
    \toprule
    Parameter & Value \\
    \midrule
    Batch size (global, across $6$ learners) & $144$ \\
    Trajectory length & $19$ \\
    Training duration (environment steps) & $75M$ \\
    Training duration (learner updates) & $\approx 3.44K$ \\
    Adam optimizer learning rate & $10^{-4}$ \\
    Weight decay & $10^{-2}$ \\
    Gradient norm clipping & $10$ \\
    Target network exponential moving average step size $\alpha$ & $10^{-2}$ \\
    Discount ($\gamma$) & $0.997$ \\
    Interval for sampling $\xc_0$ & $[-9, 9)$ \\
    \bottomrule
    \end{tabularx}
    \caption{Training parameters for D$\eta$N in the Atari experiments.
    \label{tab:atari:training}}
\end{table}
\begin{table}[tb]
    \centering
    \begin{tabularx}{\textwidth}{lXc}
    \toprule
    Component & Parameter & Value \\
    \midrule
    Vision (ResNet) & & \\
    \midrule
        & Output channels (per for Conv2D and residual layers per section) & $(64, 128, 128)$ \\
        & Kernel sizes (all Conv2D and residual layers) & $(3, 3)$ \\
        & Strides (all Conv2D and residual layers) & $(1, 1)$ \\
        & Padding & \texttt{SAME} \\
        & Pool sizes (all sections) & $(3, 3)$ \\
        & Pool strides (all sections) & $(3, 3)$ \\
        & Residual blocks (per section) & $(2, 2, 2)$ \\
    \midrule
    Linear & & \\
    \midrule
    & Output size & $512$ \\
    \midrule
    Quantile MLP & & \\
    \midrule
        & Number of quantiles (per action) & $100$ \\
        & Hidden layer size & $512$ \\
    \bottomrule
    \end{tabularx}
    \caption{Neural network parameters for D$\eta$N's return distribution estimator $\xi_{\theta}$ in the Atari experiments.
    See \cref{fig:din-architecture} for reference.
    \label{tab:atari:model}}
\end{table}

Similar to DQN, we annealed the $\varepsilon$-greedy parameter linearly from $1.0$ at the start to $0.1$ at the end of training, and used $10^{-2}$-greedy policies for evaluation.

The convolutional network in the Atari experiments is a ResNet \citep{he2016deep} as used by \citet{espeholt2018impala}.
Convolutional layers and residual blocks used ReLU activations~\citep{nair2010rectified}, as did the MLP hidden layer.
The ``Linear'' components in \cref{fig:din-architecture} did not use an activation function on the outputs (note that the explicit ReLU activation in the diagrams is used).
The outputs of the ResNet were flattened before being input to the ``Linear'' component.

\section{Summary of Guarantees}
\label{app:summary-of-guarantees}

\Cref{tab:conditions-overview} provides a summary of the necessary and sufficient conditions for the objective $F_K$ to be optimizable by DP in the different scenarios considered in this work.
\begin{table}[tb]
    \centering
    \begin{tabularx}{\textwidth}{
            >{\raggedright\arraybackslash}p{0.1\textwidth}
            >{\raggedright\arraybackslash}p{0.2\textwidth}
            >{\raggedright\arraybackslash}p{0.15\textwidth}
            X
        }
        \toprule
        Setting & DP & Case & Conditions on the Objective (and references) \\
        \midrule
        Standard & Classic or distributional & Finite horizon ($\gamma = 1$) & \emph{Necessary and sufficient}: Expected utility $\Uf$ with (up to affine transformations) $f(\xc) = e^{\lambda \xc}$ for $\lambda \in \reals$ or $f$ identity \citep{marthe2024beyond}. \\
        \cmidrule(l){3-4}
        & & Infinite horizon ($\gamma < 1$) & \emph{Necessary and sufficient}: Expected utility $\Uf$ with $f$ (up to affine transformations) positively homogeneous~\citep[see \cref{prop:positively-homogeneous,thm:value-iteration,thm:policy-iteration} and ][]{bowling2023settling}. \\
        \midrule
        \Stock{}-augmented & Classic & Finite horizon ($\gamma = 1$) & \emph{Necessary and sufficient}: Expected utility, RL rewards with bounded first moment (\cref{thm:reward-design}). \\
        \cmidrule(l){3-4}
        & & Infinite horizon ($\gamma < 1$) & \emph{Necessary}: Expected utility $\Uf$ with $f$ (up to affine transformations) positively homogeneous (\cref{thm:reward-design,lem:expected-utility-conditions}). \emph{Sufficient}: Expected utility $\Uf$ with $f$ Lipschitz and $f$ (up to affine transformations) positively homogeneous (\cref{thm:reward-design}). \\
        \midrule
        \Stock{}-augmented & Distributional & Finite horizon ($\gamma = 1$) & \emph{Necessary and sufficient}: Indifferent to mixtures (\cref{thm:value-iteration,thm:policy-iteration,prop:indifferences-necessary}). \\
        \cmidrule(l){3-4}
        & & Infinite horizon ($\gamma < 1$) & \emph{Necessary}: Indifferent to mixtures and indifferent to $\gamma$ (\cref{thm:value-iteration,thm:policy-iteration,prop:indifferences-necessary}). \emph{Sufficient}: Lipschitz, indifferent to mixtures and indifferent to $\gamma$ (\cref{thm:value-iteration,thm:policy-iteration,prop:indifferences-necessary}). \\
        \bottomrule
    \end{tabularx}
    \caption{Summary of necessary and sufficient conditions on $F_K$ for classic and distributional DP in various scenarios, including references. Previous work only considered the scalar case ($\cs = \reals$); our results also apply to the vector-valued case ($\cs = \reals^m$). All instances of positive homogeneity mentioned on this table have the following condition: $(1 - \alpha)(f(\xc) - f(0)) = f(\gamma \xc) - f(0)$ with $\alpha \in (0, 1]$ and $\gamma < 1 \Rightarrow \alpha < 1$ (see \cref{eq:indifference-to-gamma-condition}).
    \label{tab:conditions-overview}}
\end{table}
The table includes references to specific results in this work and in previous work, as applicable.
For a more detailed discussion on DP guarantees from previous work, see \cref{sec:ddp:existing-results}.
For a comparison between classic and distributional DP bounds (value iteration and policy iteration) refer to \cref{sec:ddp:value-iteration,sec:ddp:policy-iteration}. 

\emph{Counter-examples.}
In the standard setting, without \stock{} augmentation, classic and distributional DP can solve the same set of problems (see \cref{tab:conditions-overview}).
With \stock{} augmentation in the finite-horizon undiscounted case, see \cref{prop:non-expected-utility} for a functional that distributional DP can optimize, but classic DP cannot.
In the \stock{}-augmented infinite-horizon setting, we are not aware of any functionals that can only be optimized by either classic or distributional DP (cf.~\cref{thm:value-iteration,thm:policy-iteration,thm:reward-design}).
In the finite-horizon undiscounted setting with \stock{} augmentation, there exist functionals that distributional DP can optimize but classic DP cannot (see \cref{prop:non-expected-utility}). In the \stock{}-augmented infinite-horizon setting, we are not aware of any functionals that can only be optimized by either classic or distributional DP (cf.~\cref{thm:value-iteration,thm:policy-iteration,thm:reward-design}).
If a counter-example exists, it must fall in one of the following two cases:
i) an expected utility with $f$ non-Lipschitz but $\xc \mapsto f(\xc) - f(0)$ positively homogeneous; ii) a non-Lipschitz non-expected-utility that is indifferent to mixtures and indifferent to $\gamma$ (classic DP cannot optimize this; see \cref{prop:indifferences-necessary,thm:lipschitz-characterization}).

\vskip 0.2in
\bibliography{main}

\end{document}